\documentclass[twoside]{article}
\usepackage[accepted]{aistats2024}

\usepackage{amsmath}
\usepackage{amsfonts}
\usepackage{amsthm}
\usepackage{commath}
\usepackage{graphicx}
\usepackage{svg}

\usepackage{caption}
\usepackage{subcaption}

\usepackage{url}
\usepackage{hyperref}       

\usepackage{algorithm}
\usepackage{algorithmicx}
\usepackage{algpseudocode}
\DeclareCaptionFormat{myformat}{#3}
\newtheorem{theorem}{Theorem}

\algrenewcommand\algorithmicrequire{\textbf{Input:}}
\algrenewcommand\algorithmicensure{\textbf{Output:}}

\usepackage[style=apa]{biblatex}
\begin{document}

%

%

\twocolumn[

\aistatstitle{GmGM: a Fast Multi-Axis Gaussian Graphical Model}

\aistatsauthor{ Bailey Andrew \And David R. Westhead \And  Luisa Cutillo }

\aistatsaddress{ University of Leeds \And  University of Leeds \And University of Leeds } ]

\begin{abstract}
  This paper introduces the Gaussian multi-Graphical Model, a model to construct sparse graph representations of matrix- and tensor-variate data.
  We generalize prior work in this area by simultaneously learning this representation across several tensors that share axes, which is necessary to allow the analysis of multimodal datasets such as those encountered in multi-omics. Our algorithm uses only a single eigendecomposition per axis, achieving an order of magnitude speedup over prior work in the ungeneralized case.
  This allows the use of our methodology on large multi-modal datasets such as single-cell multi-omics data, which was challenging with previous approaches.
  We validate our model on synthetic data and five real-world datasets.
\end{abstract}

\section{INTRODUCTION}

A number of modern applications require the estimation of networks (graphs) exploring the dependency structures underlying the data.  In this paper, we propose a new approach for estimating conditional dependency graphs.  Two datapoints $x, y$ are \textit{conditionally independent} (with respect to a dataset $\mathcal{D}$) if knowing one provides no information about the other that is not already contained in the rest of the dataset: $\mathbb{P}(x | y, \mathcal{D}_{\backslash xy}) = \mathbb{P}(x | \mathcal{D}_{\backslash xy})$.  For normally distributed data, conditional dependencies are encoded in the inverse of the covariance matrix (the `precision' matrix).  Two datapoints are conditionally dependent on each other if and only if their corresponding element in the precision matrix is not zero.  If our dataset was in the form of a vector $\mathbf{d}$, we could then model it as $\mathbf{d} \sim \mathcal{N}(\mathbf{0}, \mathbf{\Psi}^{-1})$ for precision matrix $\mathbf{\Psi}$.  This is a Gaussian Graphical Model (GGM); $\mathbf{\Psi}$ encodes the adjacency matrix of the graph.

However, datasets are often more structured than vectors.  For example, single-cell RNA sequencing datasets (scRNA-seq) come in the form of a matrix of gene expression counts whose rows are cells and columns are genes.  Video data naturally requires a third-order tensor of pixels to represent it - rows, columns, and frames.  Furthermore, multi-omics datasets, such as those including both scRNA-seq and scATAC-seq, may require two or more matrices to be properly represented; one for each modality.

We could assume that each row of our matrix is an i.i.d. sample drawn from our model.  However, independence is a strong and often incorrect assumption.  If we wanted to make no independence assumptions, we could vectorize the dataset $\mathcal{D}$ and estimate $\mathbf{\Psi}$ in $\mathrm{vec}[\mathcal{D}] \sim \mathcal{N}(\mathbf{0}, \mathbf{\Psi}^{-1})$.  Unfortunately, this produces intractably large $\mathbf{\Psi}$, whose number of elements is quadratic in the product of the lengths of our dataset's axes: the size of $\mathbf{\Psi}$ is $O\left(\prod_\ell d_\ell^2\right)$ (throughout this paper, $d_\ell$ will denote the size of the $\ell$th axis of a tensor).

Thankfully, tensors are highly structured, and we are often interested in the dependency structure of each axis individually - i.e. the dependencies between cells and the dependencies between genes, in the case of scRNA-seq - rather than the dependencies between the elements of the tensor themselves.  To model this, we can represent $\mathbf{\Psi}$ as some deterministic combination of the axis-wise dependencies: $\mathrm{vec}[\mathbf{D}] \sim \mathcal{N}(\mathbf{0}, \zeta(\mathbf{\Psi}_\mathrm{row}, \mathbf{\Psi}_\mathrm{col})^{-1})$, for some function $\zeta$.  The strategy is to estimate $\mathbf{\Psi}_\mathrm{row}$ and $\mathbf{\Psi}_\mathrm{col}$ directly, without computing the intractable $\zeta(\mathbf{\Psi}_\mathrm{row}, \mathbf{\Psi}_\mathrm{col})^{-1}$.  While there are multiple choices for $\zeta$, this paper considers only the Kronecker sum.  Choices of $\zeta$, and their definitions, will be given in Section \ref{sec:prior-work}.

\subsection{Notation}

When possible, this paper follows the notation of prior work in multi-axis methods, such as \cite{kalaitzis_bigraphical_2013} and \cite{greenewald_tensor_2019}.  For tensor-specific notation, we typically adhere to \cite{kolda_tensor_2009}.  Some changes have been made due to inconsistencies or conflicts.

The subscript $\ell$ will always be used to denote an arbitrary axis.  $d_\ell$ and $\mathbf{\Psi}_\ell$ represent the size and precision matrix for the axis $\ell$, respectively - as in prior work.  It will be helpful to define the symbols $d_{<\ell}$, $d_{>\ell}$, $d_{\backslash\ell}$, and $d_{\forall}$ to represent the product of the lengths of all axes before $\ell$, after $\ell$, excluding $\ell$, and overall, respectively.

Lowercase bold represents vectors ($\mathbf{d}$), uppercase bold represents matrices ($\mathbf{D}$), and uppercase calligraphic represents tensors ($\mathcal{D}$).  We will also be dealing with sets of tensors, $\left\{\mathcal{D}^\gamma\right\}_\gamma$; each individual tensor is called a modality, and the superscript $\gamma$ will always represent indexing a modality.  The number of axes of a tensor $\mathcal{D}^\gamma$ is denoted $K^\gamma$, and to describe that an axis $\ell$ is one of the axes of a tensor $\mathcal{D}^\gamma$, we will write $\ell\in\gamma$.  Axes can be part of several modalities.  An important tensor-variate operation is matricization, which takes a $d_1 \times \ldots \times d_K$ tensor $\mathcal{D}$ and outputs a $d_\ell \times d_{\backslash\ell}$ matrix.  It is denoted $\mathrm{mat}_\ell\left[\mathcal{D}\right]$, and can be viewed as vectorizing each slice of $\mathcal{D}$ along the $\ell$th axis.

We will make frequent use of the Kronecker sum, defined in terms of the Kronecker product.  The Kronecker product of two matrices $\mathbf{\Psi}_{\ell_1}, \mathbf{\Psi}_{\ell_2}$ is denoted $\mathbf{\Psi}_{\ell_1} \otimes \mathbf{\Psi}_{\ell_2}$, and is best defined by example;

\begin{align*}
    \begin{bmatrix}
        1 & 2 \\
        3 & 4
    \end{bmatrix} \otimes \begin{bmatrix}
        5 & 6
    \end{bmatrix} &= \begin{bmatrix}
        1 \times \begin{bmatrix}5 & 6\end{bmatrix} & 2 \times \begin{bmatrix}5 & 6\end{bmatrix} \\
        3 \times \begin{bmatrix}5 & 6\end{bmatrix} & 4 \times \begin{bmatrix}5 & 6\end{bmatrix}
    \end{bmatrix}\\
    &=\begin{bmatrix}
        5 & 6 & 10 & 12 \\
        15 & 18 & 20 & 24
    \end{bmatrix}
\end{align*}

The Kronecker sum is $\mathbf{\Psi}_{\ell_1} \oplus \mathbf{\Psi}_{\ell_2} = \mathbf{\Psi}_{\ell_1} \otimes \mathbf{I}_{d_{\ell_2}\times d_{\ell_2}} + \mathbf{I}_{d_{\ell_1}\times d_{\ell_1}} \otimes \mathbf{\Psi}_{\ell_2}$.  When the matrices $\mathbf{A}, \mathbf{B}$ are adjacency matrices of graphs, the Kronecker sum has the interpretation as the Cartesian product of those graphs.

The ``blockwise trace" operation defined by \textcite{kalaitzis_bigraphical_2013} appears when calculating the gradient of the log-likelihood of KS-structured normal distributions (and is strongly related to the $\mathrm{proj}_{KS}$ operator in \cite{greenewald_tensor_2019}).  Let $\mathbf{I}_a$ be the $a\times a$ identity matrix, and let $\mathbf{J}^{ij}$ be a matrix of zeros except at $(i, j)$, where it equals 1.  Then, we denote the blockwise trace of a matrix $\mathrm{tr}^{a}_{b}$ and define it according to Line \ref{eq:blockwise-trace}.  

\begin{align}
    \mathrm{tr}^{a}_{b}\left[\mathbf{M}\right] &= \left[\mathrm{tr}\left[\mathbf{M}\left(\mathbf{I}_{a} \otimes \mathbf{J}^{ij}\otimes \mathbf{I}_{b}\right)\right]\right]_{ij} \label{eq:blockwise-trace}
\end{align}

The final needed concept is that of the Gram matrix $\mathbf{S}_{\ell}^\gamma$, defined as $\mathbf{S}^{\gamma}_{\ell} = \mathrm{mat}_\ell\left[\mathcal{D}^\gamma\right]\mathrm{mat}_\ell\left[\mathcal{D}^\gamma\right]^T$.  In the single-tensor (multi-axis) case, this is a sufficient statistic; all prior work first computes these matrices as the initial step in their algorithm, as do we.  When there are multiple modalities, we consider the `effective Gram matrices' $\mathbf{S}_\ell = \sum_{\gamma | \ell\in\gamma} \mathbf{S}_{\ell}^\gamma$, as these fulfill the role of sufficient statistic in the multi-tensor case.

\section{PRIOR WORK}
\label{sec:prior-work}

The Graphical Lasso \parencite{friedman_sparse_2008} is the standard single-axis GGM.  The model assumes one has $n$ independent samples of $d_1$-length vectors, and seeks to estimate the $d_1 \times d_1$ precision matrix $\mathbf{\Psi}$.  $\mathbf{\Psi}$ is estimated via Equation \ref{eq:1}.

\begin{align}
    \mathbf{\Psi} &= \mathrm{argmax}_{\mathbf{\Psi} \succ 0} \log \det \mathbf{\Psi} - \mathrm{tr}\left[\mathbf{S}\mathbf{\Psi}\right] + \rho \norm{\mathbf{\Psi}}_1 \label{eq:1}
\end{align}

Without the regularizing $\rho \norm{\mathbf{\Psi}}_1$ term, this represents the maximum likelihood estimate for $\mathbf{\Psi}$; $\mathbf{\Psi}\succ0$ ensures positive definiteness.  The regularization penalty is an L1 penalty over the elements of $\mathbf{\Psi}$ - typically, the diagonal elements are not included in this penalty, as they do not represent edges on the conditional dependency graph.

Our work considers the Kronecker-sum (multi-axis) Graphical Lasso, which was introduced by \textcite{kalaitzis_bigraphical_2013}.  This sum is one choice of $\zeta$ to combine the per-axis precision matrices into the precision matrix of the vectorized dataset, $\mathrm{vec}[\mathbf{D}] \sim \mathcal{N}(\mathbf{0}, (\mathbf{\Psi}_\mathrm{row} \oplus \mathbf{\Psi}_\mathrm{col})^{-1})$.  Other choices for $\zeta$ have been considered, such as using the Kronecker product \parencite{tsiligkaridis_convergence_2013, dahl_network_2013}, or the square of the Kronecker sum \parencite{wang_sg-palm_2021, wang_sylvester_2020}.  Each method has its strengths; the benefits of a Kronecker sum structure are its interpretability as a graph product, stronger sparsity, convexity of the maximum likelihood, and its allowance of inter-task transfer.

\textcite{kalaitzis_bigraphical_2013}'s BiGLasso model only worked for matrix-variate data, and was very slow to converge to a solution, in large part due to its non-optimal space complexity of $O(d_1^2d_2^2)$.  This prohibited its use on moderately sized datasets.  Numerous modifications have been made to the algorithm to improve its speed and to achieve an optimal space complexity of $O(d_1^2+d_2^2)$, such as scBiGLasso \parencite{li_scalable_2022}, TeraLasso \parencite{greenewald_tensor_2019}, and EiGLasso \parencite{yoon_eiglasso_2020}.  Of these TeraLasso is notable in that it allows an arbitrary number of axes, i.e. $\zeta(\mathbf{\Psi}_1, ..., \mathbf{\Psi}_k) = \mathbf{\Psi}_1 \oplus ... \oplus \mathbf{\Psi}_k$, like our proposed method.  TeraLasso and EiGLasso, the fastest prior algorithms, both rely on computing an eigendecomposition every iteration.

All of these models, including our own, rely on a normality assumption.  We are most interested in the case of omics data, which is typically non-Gaussian; an overview of the use of GGMs in omics data is given by \textcite{altenbuchinger_gaussian_2020}.  For single-axis methods, the nonparanormal skeptic \parencite{liu_nonparanormal_2012} can replace the normality assumption with a weaker nonparanormality assumption; \textcite{li_scalable_2022} proposed a generalization of this approach to multi-axis methods, which is compatible with our approach.  The nonparanormal skeptic often leads to better results on real datasets, for example see Figure \ref{fig:mouse-assortativity}.  For a thorough comparison of the nonparanormal skeptic, see the supplementary material.

In the single-axis case, several techniques exist for combining datasets with shared axes  \parencite{cai_joint_2016, danaher_joint_2014, mccarter_sparse_2014}.  However, none of these have been applied to the multi-axis case.

\subsection{Unmet need}

Many datasets, especially those in multi-omics, are representable as a collection of matrices or tensors.  As a case study, we consider (a subset of) the Lifelines-DEEP dataset from \textcite{tigchelaar_cohort_2015}, which is summarized graphically in Figure \ref{fig:lifelines-deep}.

\begin{figure}[t]
  \centering
  \includegraphics[width=0.48\textwidth]{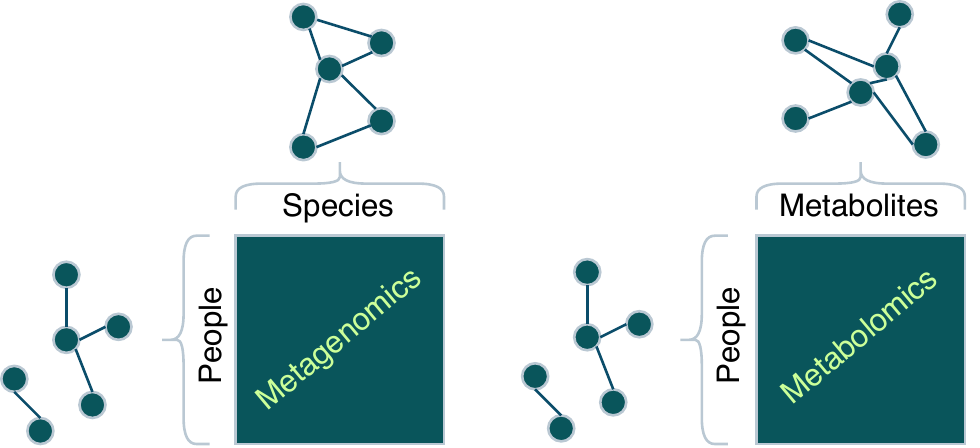}
  \caption{The two matrices of the LifeLines-DEEP dataset.  As both matrices include data for the same people, the learned graph between people should be the same.}
  \label{fig:lifelines-deep}
\end{figure}

In this dataset, two different modalities of data were gathered from the same people: counts of microbial species found in their stools (metagenomics) and counts of metabolites found in their blood plasma (metabolomics).  While the metagenomics and metabolomics are different matrices, each modality shares an axis (`people').  If we were to estimate a graph of people on each modality independently, they would likely yield different graphs.  This is not ideal; if our aim is to estimate the true graph of conditional dependencies, there should be only one resultant graph.  To estimate it, we should be considering both modalities simultaneously to take advantage of the shared axis.

One way to do this would be to concatenate the modalities, producing a matrix of people by `species+metabolites'.  This could yield interesting results, if one is interested in connections between individual species and a metabolite.  However, it would vastly increase the size of the output graph, which grows quadratically in the length of the axis.  Furthermore, it is not always possible; some datasets may not be concatenatable.


\section{OUR CONTRIBUTIONS}



We make three contributions to the field of multi-axis graphical models.  Firstly, we produce an algorithm to estimate multi-axis graphical models an order of magnitude faster than existing work.  This allows us to apply our algorithm on large datasets, such as single-cell RNA-sequencing datasets which include roughly 20,000 genes and tens of thousands of cells.  Prior work could not handle these datasets in reasonable time - ours is the only method able to produce multi-axis estimates for data of this size, and runs comparably quickly to single-axis methods.

Secondly, we generalize the multi-axis model from prior work, which was only able to consider unimodal data, into a model that can handle multiple tensors with shared axes.  This generalization is essential to the use of multi-axis methods on multi-modal data such as multi-omics; no multi-axis method had been made for this type of data before our work.

Finally, we produce a few helpful theorems to improve accuracy and further improve speed.  A covariance-thresholding technique by \textcite{mazumder_exact_2012} allowed the partitioning of data into guaranteed-independent chunks of data in the single-axis case; we show how to lift this result to the multi-axis case, and demonstrate its efficacy (Figure \ref{fig:cov-thresh}).  We also show how to efficiently include some natural prior distributions into our model, which allows us to demonstrate an extraordinary increase in performance on real-world data (Figure \ref{fig:lifelines-prior}).

\subsection{The model}
\label{sec:model}

To properly handle sets of tensors, we propose modelling each tensor as being drawn independently from a Kronecker-sum normal distribution.  If the tensors share an axis $\ell$, then they will still be drawn independently - but their distributions will be parameterized by the same $\mathbf{\Psi}_\ell$.  For an arbitrary set of tensors, the model is:

\begin{align*}
    \mathcal{D}^\gamma \sim& \mathcal{N}_{KS}\left(\left\{\mathbf{\Psi}_\ell\right\}_{\ell\in\gamma}\right) \\
    &\text{for }\mathcal{D}^\gamma\in\left\{\mathcal{D}^\gamma\right\}_\gamma
\end{align*}

We call this model the ``Gaussian multi-Graphical Model'' (GmGM) as it extends Gaussian Graphical Models to estimate multiple graphs from a set of tensors.  In this paper, we will make the assumption that no tensor in our set contains the same axis twice.  Any tensor with a repeated axis would naturally be interpretable as a graph - such datasets are rare, and if one already has a graph then the need for an algorithm such as this is diminished.

As an example, we model the LifeLines-DEEP dataset $\mathbf{D}^\mathrm{metagenomics}$ and $ \mathbf{D}^\mathrm{metabolomics}$ independently as:

\begin{align*}
    \mathbf{D}^\mathrm{metagenomics} \sim& \mathcal{N}_{KS}\left(\mathbf{\Psi}^\mathrm{people}, \mathbf{\Psi}^\mathrm{species}\right) \\
    \mathbf{D}^\mathrm{metabolomics} \sim& \mathcal{N}_{KS}\left(\mathbf{\Psi}^\mathrm{people}, \mathbf{\Psi}^\mathrm{metabolites}\right)
\end{align*}

\subsection{The algorithm}

\begin{table*}[!htbp]
\caption{A comparison of the asymptotic complexity of TeraLasso and GmGM.  EiGLasso is omitted because it has the same complexity as TeraLasso.  $N$ represents the number of iterations before convergence.  ``Simplified Time Complexity'' represents the case when all axes have the same size $d$, there is only one modality, and the dataset only has two axes.}
\label{tab:asymptotic-performance}
\vskip 0.15in
\begin{center}
\begin{small}
\begin{tabular}{c|ccc}
    Algorithm & Simplified Time Complexity & Full Time Complexity & Memory Complexity \\
    \hline
    TeraLasso & $O(Nd^3)$ & $O(\sum_\ell d_\ell^2d_{\backslash\ell} + Nd_\ell^3 + N\prod_\ell d_\ell)$ & $O(\sum_\ell d_\ell^2)$ \\
    GmGM & $O(d^3 + Nd^2)$ & $O(\sum_\ell d_\ell^2d_{\backslash\ell} + d_\ell^3 + \sum_\gamma N\prod_{\ell\in\gamma}d_\ell)$ & $O(\sum_\ell d_\ell^2)$ 
\end{tabular}
\end{small}
\end{center}
\end{table*}

Here, we present an algorithm to compute the maximum likelihood estimate (MLE) jointly for all parameters $\mathbf{\Psi}_\ell$ of the GmGM.  The general idea is to produce an analytic estimate for the eigenvectors of $\mathbf{\Psi}_\ell$, and then iterate to solve for the eigenvalues; this is summed up graphically in Figure \ref{fig:GmGM-flowchart}.

\begin{figure}[t]
  \centering
  \includegraphics[width=0.98\linewidth]{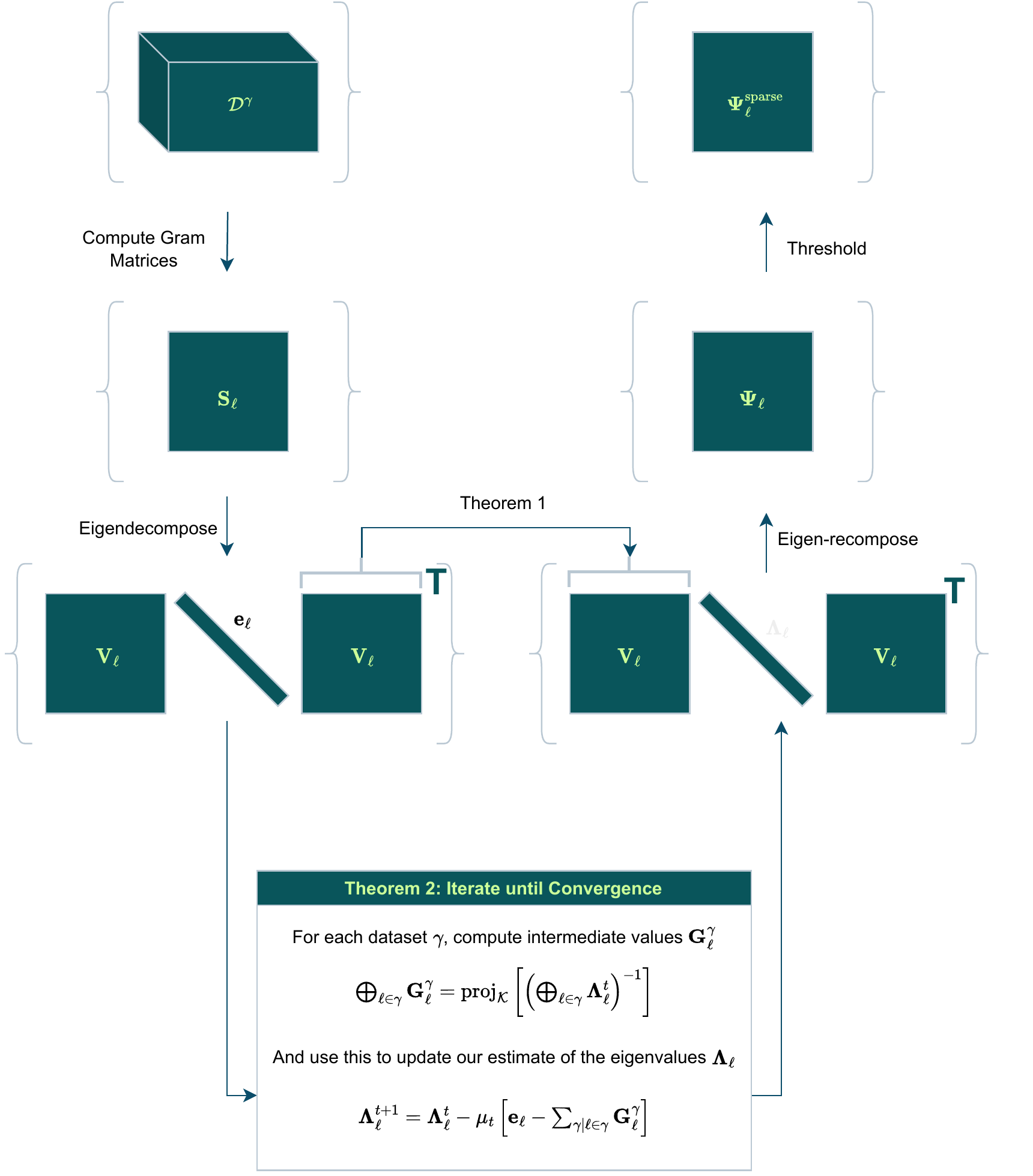}
  \caption{A graphical overview of how the GmGM algorithm works.  We use $\gamma$ to represent an arbitrary modality, and $\ell$ to represent an arbitrary axis.  Proofs are given in the supplementary material.}
  \label{fig:GmGM-flowchart}
\end{figure}




\begin{theorem}
    \label{thm:evec}
    Let $\mathbf{V}_\ell\mathrm{diag}\left[\mathbf{e}_\ell\right]\mathbf{V}_\ell^T$ be the eigendecomposition of $\mathbf{S}_\ell$ (where $\mathbf{V}_\ell \in \mathbb{R}^{d_\ell \times d_\ell}$ and $\mathrm{diag}\left[\mathbf{e}_\ell\right] \in \mathbb{R}^{d_\ell \times d_\ell}$ is a diagonal matrix with diagonal $\mathbf{e}_\ell$).  Then $\mathbf{V}_\ell$ are the eigenvectors of the maximum likelihood estimate of $\mathbf{\Psi}_\ell$.
\label{thm:evec}
\end{theorem}

Theorem \ref{thm:evec} is critical to allowing efficient estimation of $\mathbf{\Psi}_\ell$, as it not only allows us to extract the computationally intensive eigendecomposition operation from the iterative portion of the algorithm, but also reduces the remaining number of parameters to be linear in the length of an axis.

It is now necessary to find estimate the eigenvalues $\mathbf{\Lambda}_\ell = \mathrm{diag}\left[\boldsymbol{\lambda}_\ell\right]$.  We will do this iteratively, which requires the gradient of the negative log-likelihood with respect to the eigenvalues.  A sketch of the derivation of this is given.

\begin{align}
    \frac{\partial\mathrm{NLL}}{\partial \lambda_{\ell}^{(i)}} &= \frac{1}{2} \frac{\partial}{\partial\lambda_{\ell}^{(i)}}\left(\mathrm{tr}\left[\mathbf{S}_{\ell}\mathbf{V}_{\ell}\mathbf{\Lambda}_{\ell}\mathbf{V}_{\ell}^T\right] - \sum_\gamma \log\left|\bigoplus_{\ell'\in\gamma}\mathbf{\Lambda}_{\ell'}\right|\right) \\
    &= \frac{1}{2}\frac{\partial}{\partial\lambda_{\ell}^{(i)}} \mathbf{e}_\ell \boldsymbol{\lambda}_\ell - \frac{1}{2} \frac{\partial}{\partial \lambda_{\ell, (i)}} \sum_{\gamma, j} \log\left[\bigoplus_{\ell'\in\gamma}\mathbf{\Lambda}_{\ell'}\right]_{jj} \label{eq:convexity}\\
    &= \frac{1}{2}\mathbf{e}_{\ell, (i)} - \frac{1}{2}\sum_{\gamma, j} \frac{\left[\mathbf{I}_{d_{<\ell}} \otimes \mathbf{J}_{ii} \otimes \mathbf{I}_{d_{>\ell}}\right]_{jj}}{\left[\bigoplus_{\ell'\in\gamma}\mathbf{\Lambda}_{\ell'}\right]_{jj}} \\
    \frac{\partial \mathrm{NLL}}{\partial \boldsymbol{\lambda}_\ell}
    &= \frac{1}{2}\mathbf{e}_\ell - \frac{1}{2}\sum_\gamma \mathrm{tr}^{d_{<\ell}}_{d_{>\ell}}\left[\left(\bigoplus_{\ell'\in\gamma}\mathbf{\Lambda}_{\ell'}\right)^{-1}\right] \label{eq:vectorize}
\end{align}

Line \ref{eq:vectorize} follows from the definition of the blockwise trace; the rest use Theorem \ref{thm:evec} and standard algebra.

An L1 penalty is often included in graphical models to enforce sparsity.  To enforce sparsity in our model, there are a few options; one can choose to either keep the top $p$\% of edges, or keep the top $k$ edges per vertex.  In Section \ref{sec:results}, we see that this technique performs well.  Including regularization in our algorithm is nontrivial, as we rely on analytic estimates of the eigenvectors.  However, by utilizing a restricted form of the L1 penalty, we can bypass this limitation.  For details on how to do so, see Section 5 in the supplementary material - we will refer to GmGM equipped with this penalty as `GmGM L1'.  In short, the restricted L1 penalty is an L1 penalty on the off-diagonals of a matrix (i.e. the Graphical Lasso) that has given (fixed) eigenvectors (since the iterative portion of our algorithm only iterates over the eigenvalues).  We incorporate it with subgradient descent.

\subsection{Additional Theorems}

We can incorporate priors into our algorithm as well.  This is nontrivial, as we want it to be in a manner compatible with Theorem \ref{thm:evec}.

\begin{theorem}[GmGM Estimator with Priors]
    \label{thm:evec-prior}
    Suppose we have the same setup as in Theorem \ref{thm:evec}, and that we have a prior of the form:
    
    \begin{align}
        \prod_\ell
        g_\ell(\mathbf{\Theta}_\ell)
        e^{\mathrm{tr}\left[\eta_\ell(\mathbf{\Theta}_\ell)^T
        \mathbf{\Psi}_\ell\right]
        + h_\ell(\mathbf{\Psi_\ell})}
    \end{align}

    In other words, our prior is an exponential distribution for $\mathbf{\Psi}_\ell$, in which $\mathbf{\Psi}_\ell$ is the sufficient statistic.

    Then, if $h_\ell$ depends only on the eigenvalues (i.e. it is `unitarily invariant'), the eigenvectors of $\frac{1}{2}\mathbf{S}_\ell - \eta_\ell(\mathbf{\Theta}_\ell)$ are the eigenvectors $\mathbf{V}_\ell$ of the MAP estimate for $\mathbf{\Psi}_\ell$.  Furthermore, if $h_\ell$ is convex, then a result by \textcite{lewis_derivatives_1996} allows us to use an analog of Theorem \ref{thm:eval} (see supplementary material).
\end{theorem}

$\frac{1}{2}\mathbf{S}_\ell - \eta_\ell(\mathbf{\Theta}_\ell)$ can be thought of as our ``priorized'' effective Gram matrix; its eigenvalues and eigenvectors play the same role as those of $\mathbf{S}_\ell$ in the original algorithm.  The required property of unitary invariance is very common - both the Wishart (arguably the most natural prior to use in this context) and matrix gamma distributions satisfy Theorem \ref{thm:evec-prior}.  The benefits of including priors are demonstrated in Figure \ref{fig:lifelines-prior}.

We also show that the covariance thresholding trick \parencite{mazumder_exact_2012} works in the multi-axis case.  This allows us to partition the dataset into guaranteed-to-be-disconnected subsets before using our algorithm (although we do not make use of this theorem when comparing performance to other work, as it equally benefits all methods).  The benefits of using this trick are shown in Figure \ref{fig:cov-thresh}.

\begin{theorem}
    \label{thm:cov-thresh}
    Set all elements of $\mathbf{S}_\ell$ whose absolute value is less than $\rho$ to 0.  This encodes a potentially disconnected graph.  Likewise, consider $\mathbf{\Psi}_\ell$ to be the estimated precision matrix for our model equipped with an L1 penalty of strength $\rho$.  This also encodes a potentially disconnected graph.  If we label the vertices by which disconnected component they are part of, then this labeling is the same in both procedures (the procedure with $\mathbf{S}_\ell$ and the procedure with $\mathbf{\Psi}_\ell$).
\end{theorem}

\subsection{Convexity and Uniqueness}
\label{sec:identifiability}

As can be observed in Line \ref{eq:convexity}, the negative log-likelihood, as a function of the eigenvalues, can be expressed as the sum of a dot product (an affine, and hence convex, function) and negative logarithms, which are also convex.  Thus, the objective is convex.

The Kronecker-sum model, without additional restrictions, has non-identifiable diagonals.  \textcite{greenewald_tensor_2019} show that this can be solved by projecting the gradients at each iteration.  Letting $\mathbf{A}_\ell = \mathrm{tr}\left[\bigoplus_{\ell\in\gamma} \mathbf{\Lambda}_\ell\right]$, their projection is defined according to Line \ref{eq:proj}.  See Sections D and I.3 in their paper for thorough information.  We apply this projection to our algorithm as well.

\begin{align}
    \mathrm{proj}_\mathcal{K} &= \mathbf{A}_\ell - \frac{K_\gamma - 1}{K_\gamma} \frac{\mathrm{tr}\left[\mathbf{A}_\ell\right]}{d_\ell}\mathbf{I}_{d_\ell} \label{eq:proj}
\end{align}

\section{RESULTS}
\label{sec:results}

We tested our algorithm on synthetic data and five real-world datasets.  For synthetic data, we generated Erdos-Renyi graphs as our ground truth precision matrices for each axis (except for Figure \ref{fig:pr-curves-matrix-ar-1} which uses an AR(1) process).  Datasets were then generated from the Kronecker-sum normal distribution using these precision matrices.  In all cases, we only generated 1 sample from this distribution.  For comparison, we set EiGLasso's K parameter to 1; this parameter corresponds to the fastest version of EiGLasso.  We used the same convergence tolerance and max iterations for all algorithms.  Time and memory complexity for GmGM, TeraLasso, and EiGLasso are given in Table \ref{tab:asymptotic-performance}.  When measuring runtimes on synthetic data, we let the regularization parameters be zero for prior work.


\subsection{Synthetic Data}
\label{sec:synth}

We verified that our algorithm was indeed faster on matrix-variate data compared to prior work (Figure \ref{fig:axis-runtimes}) on our computer (2020 MacBook Pro with an M1 chip and 8GB RAM).  Our results on matrix data are encouraging - datasets of size 4,000 by 4,000 could have their graphs estimated in a minute, and, extrapolating the runtimes, datasets up to size 16,000 by 16,000 could have their graphs estimated in about an hour.  This is significantly faster than our baselines; neither EiGLasso nor TeraLasso could compute on 1,500 by 1,500 datasets in a minute.  Larger datasets would require more than 6GB of memory for our algorithm to run, pushing the limits of RAM.  On higher-order tensor data, our algorithm continues to outperform prior work.  For 4-axis and higher data, our algorithm could handle every dataset that could fit in RAM in less than a minute (Fig \ref{fig:4-axis-runtimes} and supplementary material).

In addition to these speed improvements, we show that we perform similarly to state-of-the-art (Figure \ref{fig:pr-curves}).  We demonstrate that taking into account shared axes does indeed improve performance (Figure \ref{fig:pr-curves-matrix-erdos-shared}), and furthermore that our restricted L1 penalty leads to near-perfect performance on tensor data (Figure \ref{fig:pr-curves-tensor-erdos}).

\begin{figure*}
    \centering
    \begin{subfigure}[b]{0.32\textwidth}
         \centering
         \includegraphics[width=\textwidth]{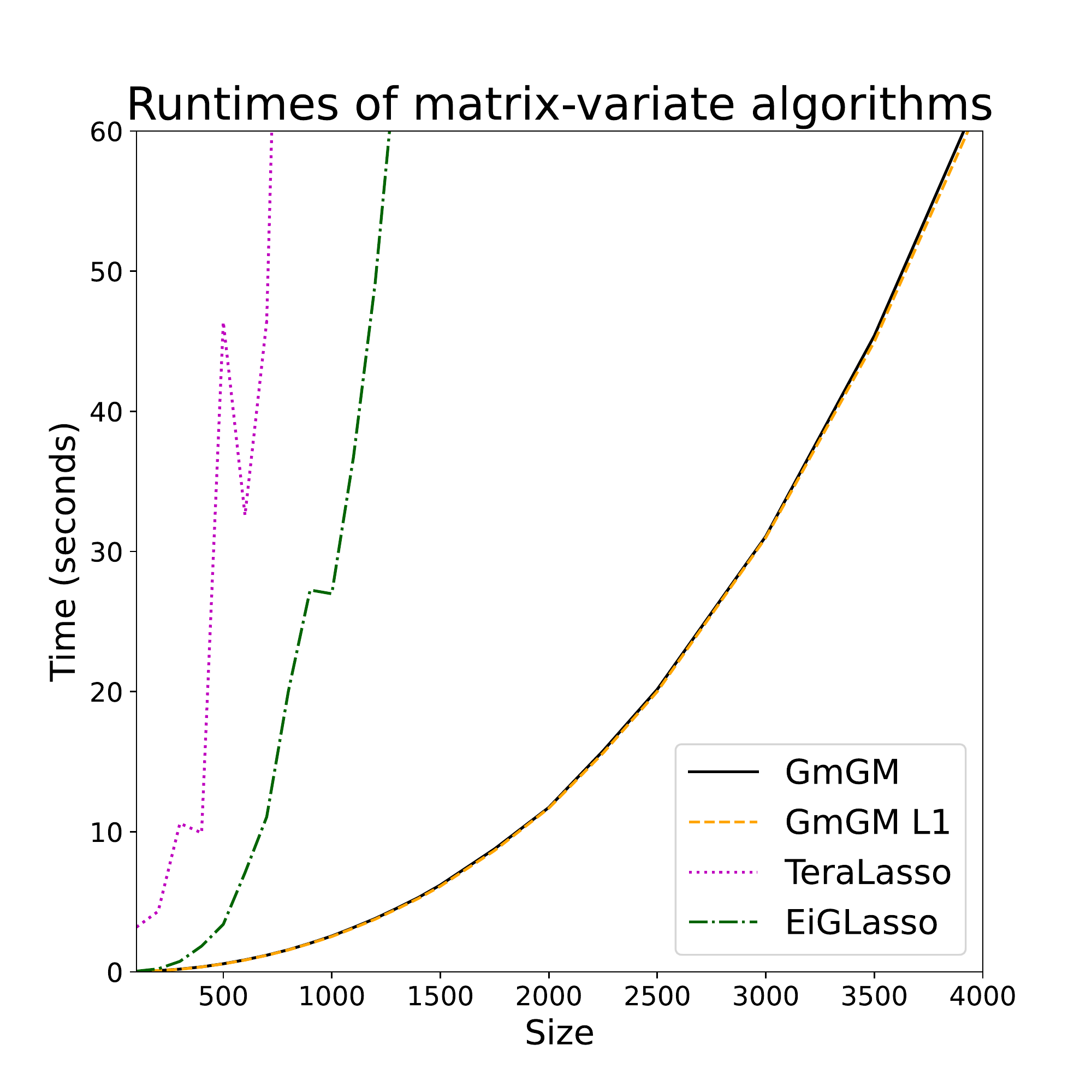}
         \caption{}
         \label{fig:2-axis-runtimes}
     \end{subfigure}
     \hfill
     \begin{subfigure}[b]{0.32\textwidth}
         \centering
         \includegraphics[width=\textwidth]{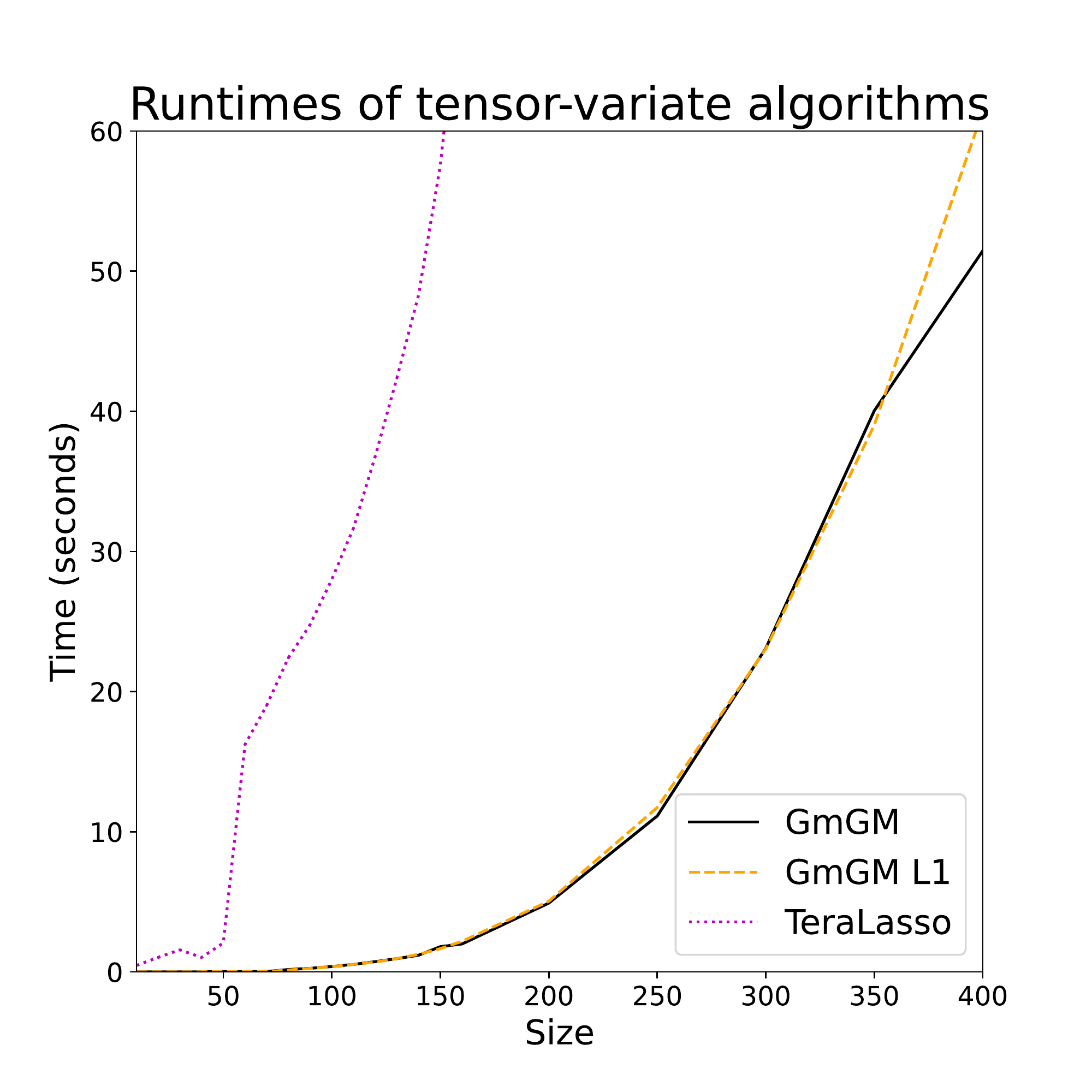}
         \caption{}
         \label{fig:3-axis-runtimes}
     \end{subfigure}
     \begin{subfigure}[b]{0.32\textwidth}
         \centering
         \includegraphics[width=\textwidth]{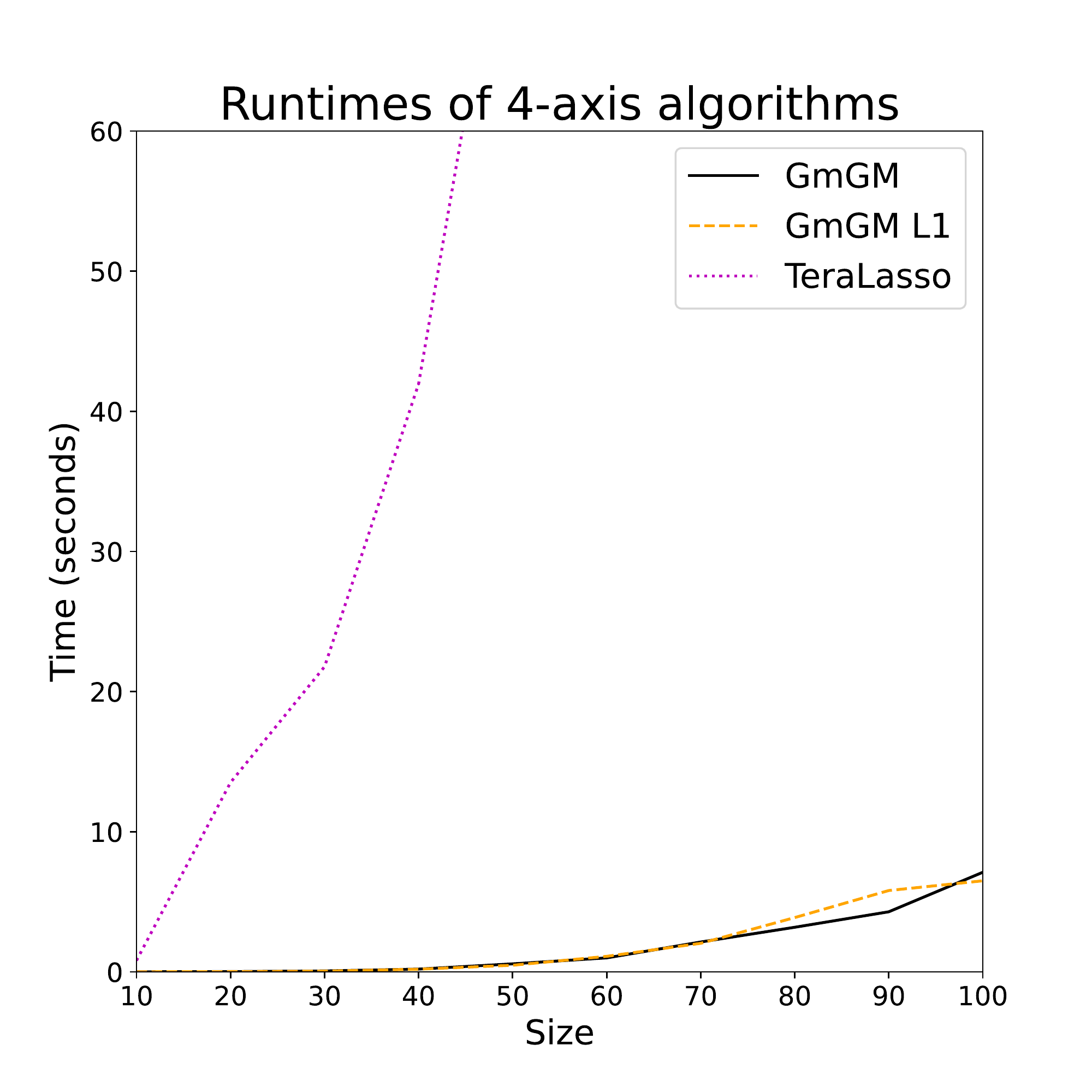}
         \caption{}
         \label{fig:4-axis-runtimes}
     \end{subfigure}
     \caption{A comparison of the runtimes of our algorithm against (a) bi-graphical, (b) 3-axis, and (c) 4-axis prior work.}
     \label{fig:axis-runtimes}
\end{figure*}

\begin{figure*}
    \centering
    \begin{subfigure}[b]{0.33\textwidth}
        \centering
        \includegraphics[width=\textwidth]{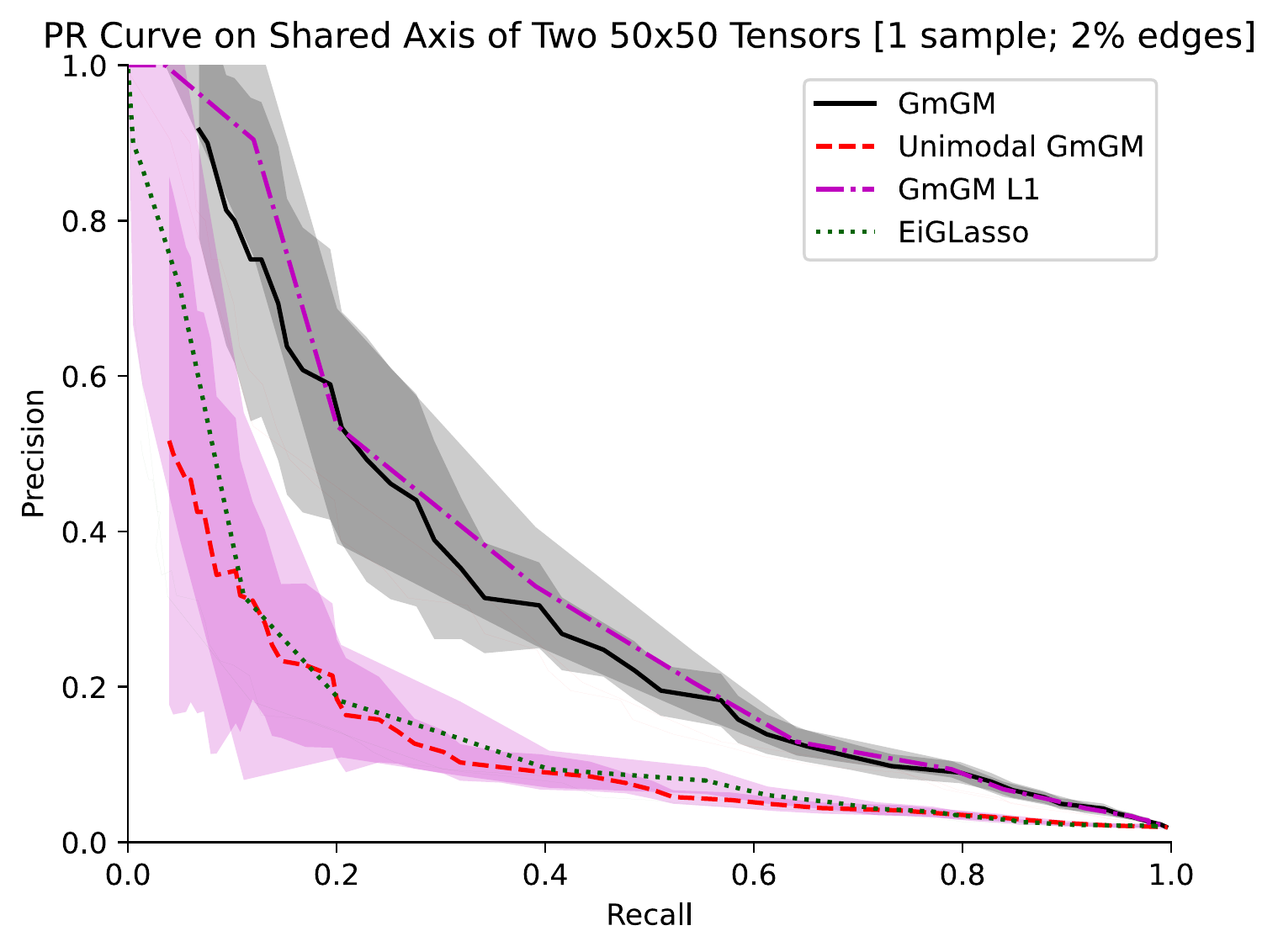}
        \caption{}
        \label{fig:pr-curves-matrix-erdos-shared}
    \end{subfigure}
    \begin{subfigure}[b]{0.32\textwidth}
        \includegraphics[width=\textwidth]{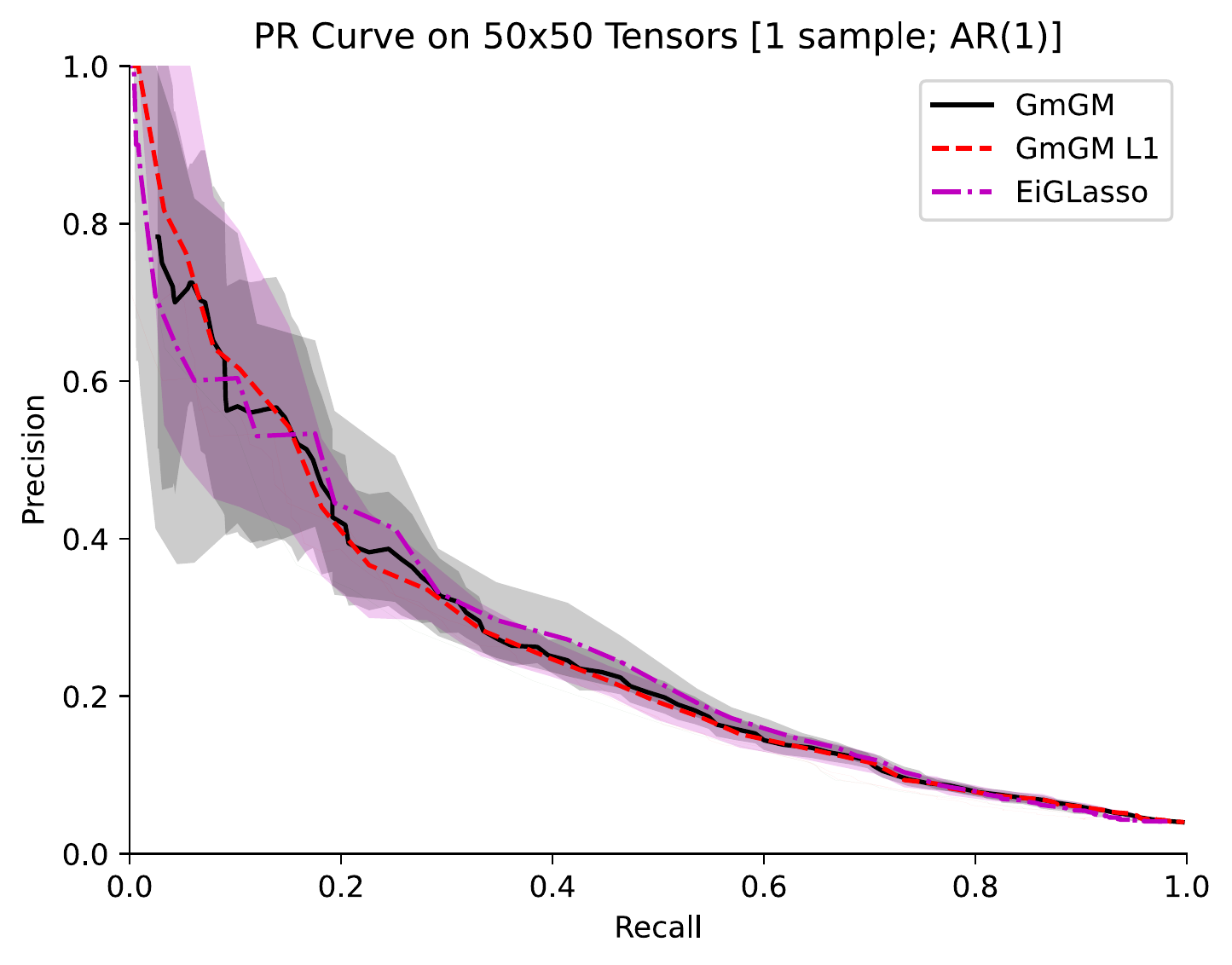}
        \caption{}
        \label{fig:pr-curves-matrix-ar-1}
    \end{subfigure}
    \begin{subfigure}[b]{0.32\textwidth}
        \includegraphics[width=\textwidth]{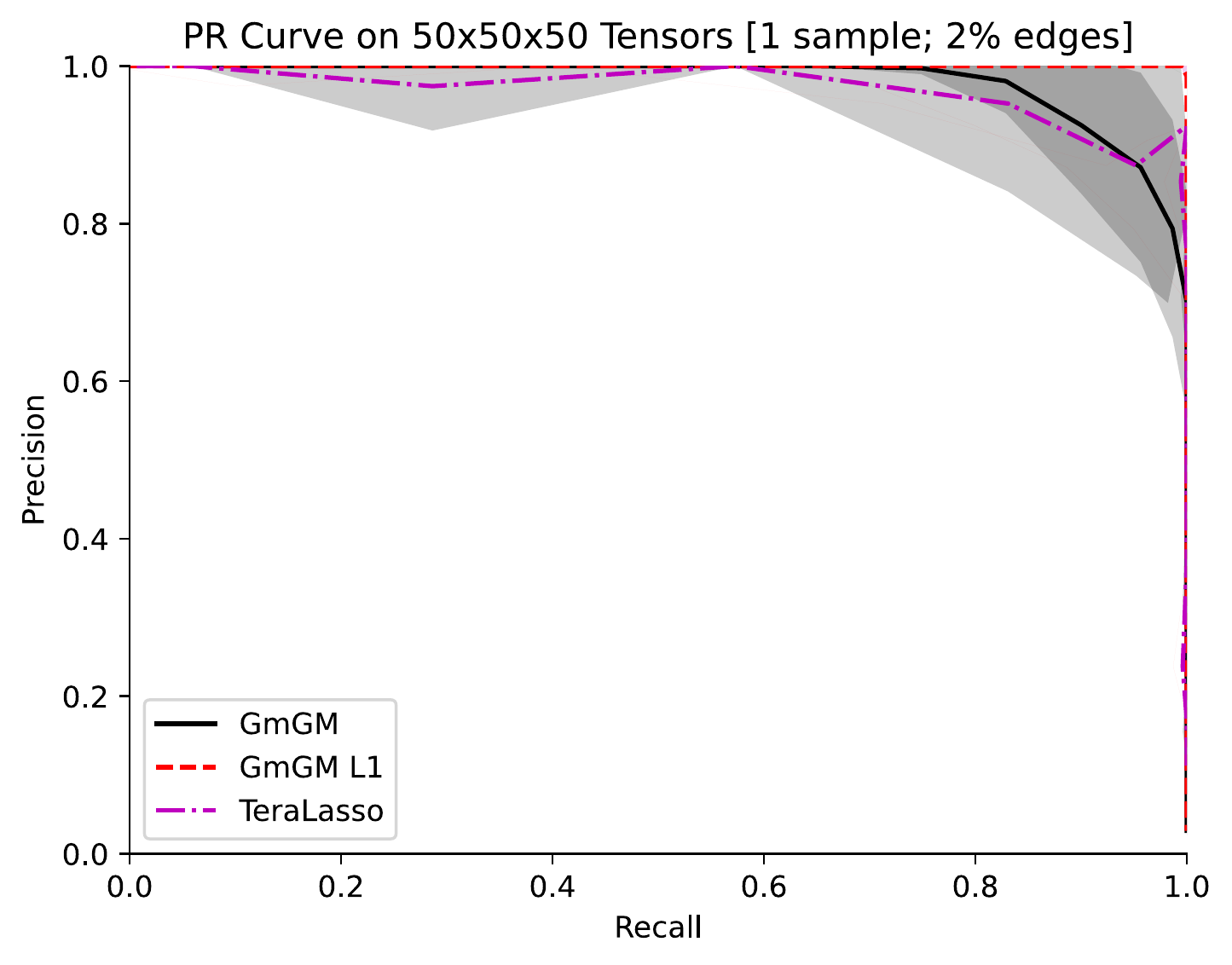}
        \caption{}
        \label{fig:pr-curves-tensor-erdos}
    \end{subfigure}
    \caption{Precision-recall curves comparing various algorithms on synthetic 50x50 data, averaged over multiple runs.  For (a) and (c), true graphs have each edge independently having a 2\% chance of existing.  For (b), true graphs are generated from an AR(1) process.  Shaded background represents standard deviation.  (a) Two 50x50 matrices with a shared axis.  EiGLasso and `Unimodal GmGM' only consider one of the matrices.  (b) A single 50x50 matrix.  (c) A single 50x50x50 tensor.  With the restricted L1 penalty, our algorithm performs nearly perfectly in the tensor-variate case.}
    \label{fig:pr-curves}
    
    
\end{figure*}

\begin{figure*}
    \centering
    \begin{subfigure}[b]{0.29\textwidth}
        \centering
        \includegraphics[width=\textwidth]{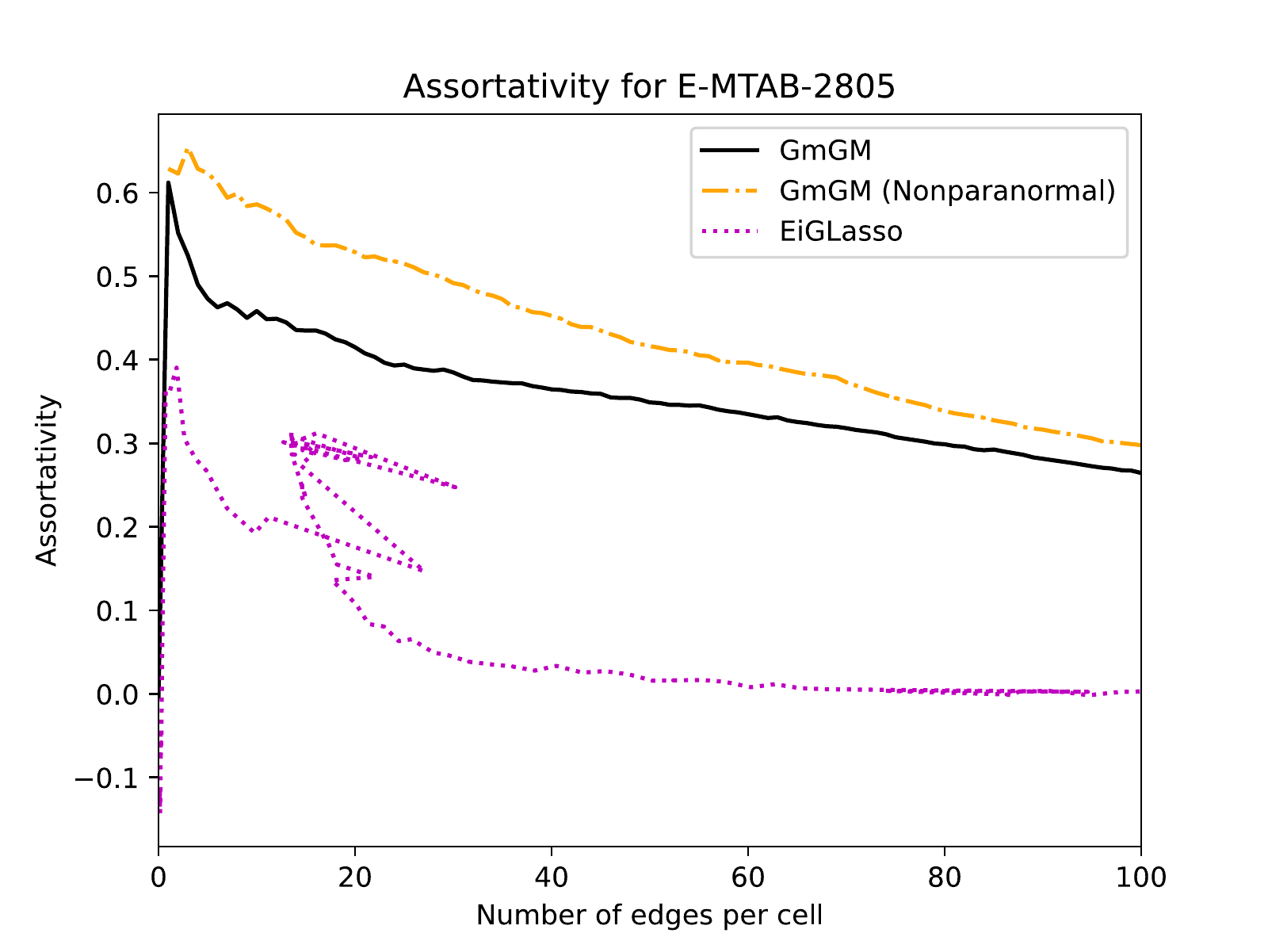}
        \caption{}
        \label{fig:mouse-assortativity}
    \end{subfigure}
    \begin{subfigure}[b]{0.20\textwidth}
        \centering
        \includegraphics[width=\textwidth]{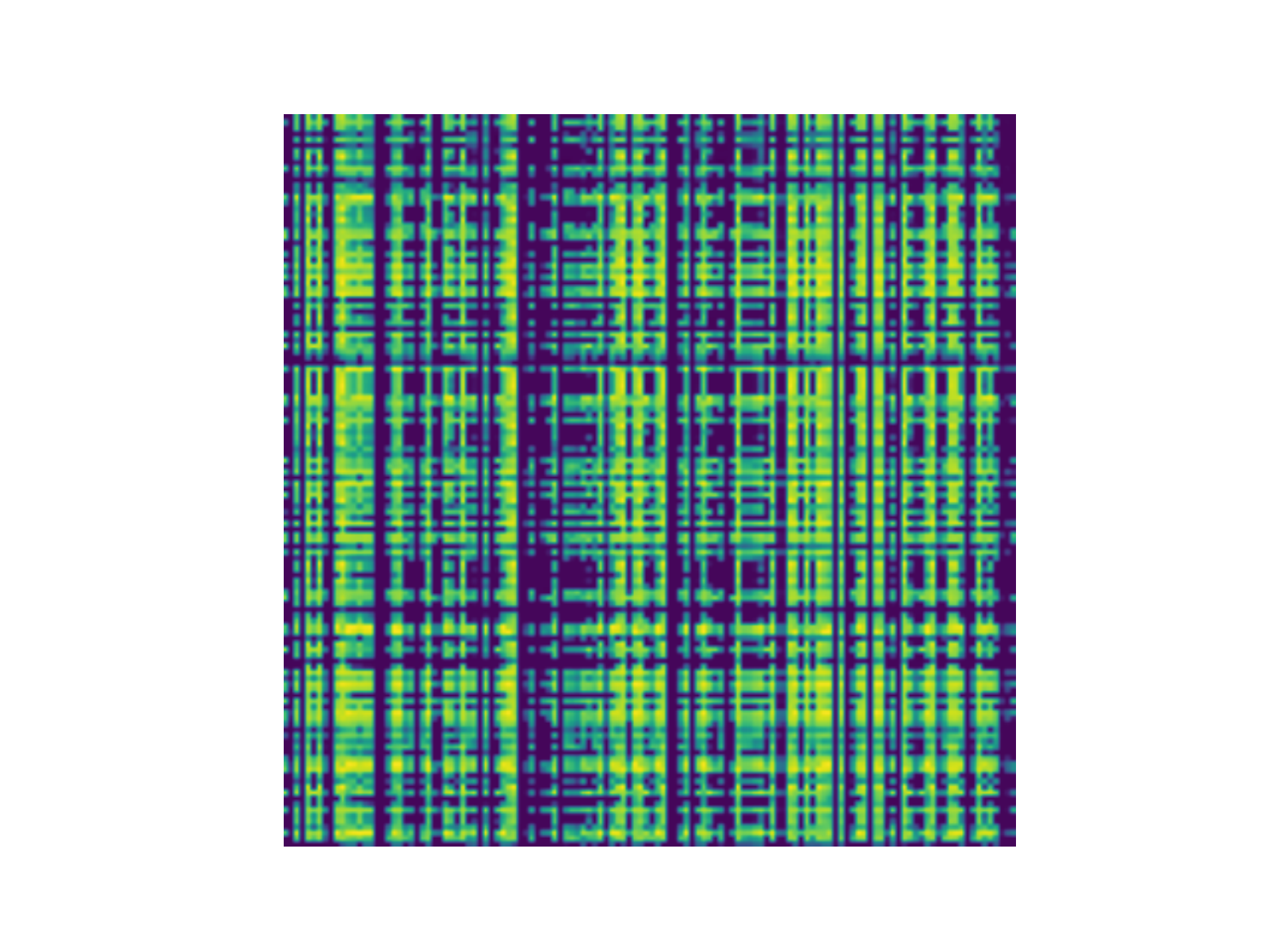}
        \caption{}
        \label{fig:shuffled-duck}
    \end{subfigure}
    \begin{subfigure}[b]{0.20\textwidth}
        \centering
        \includegraphics[width=\textwidth]{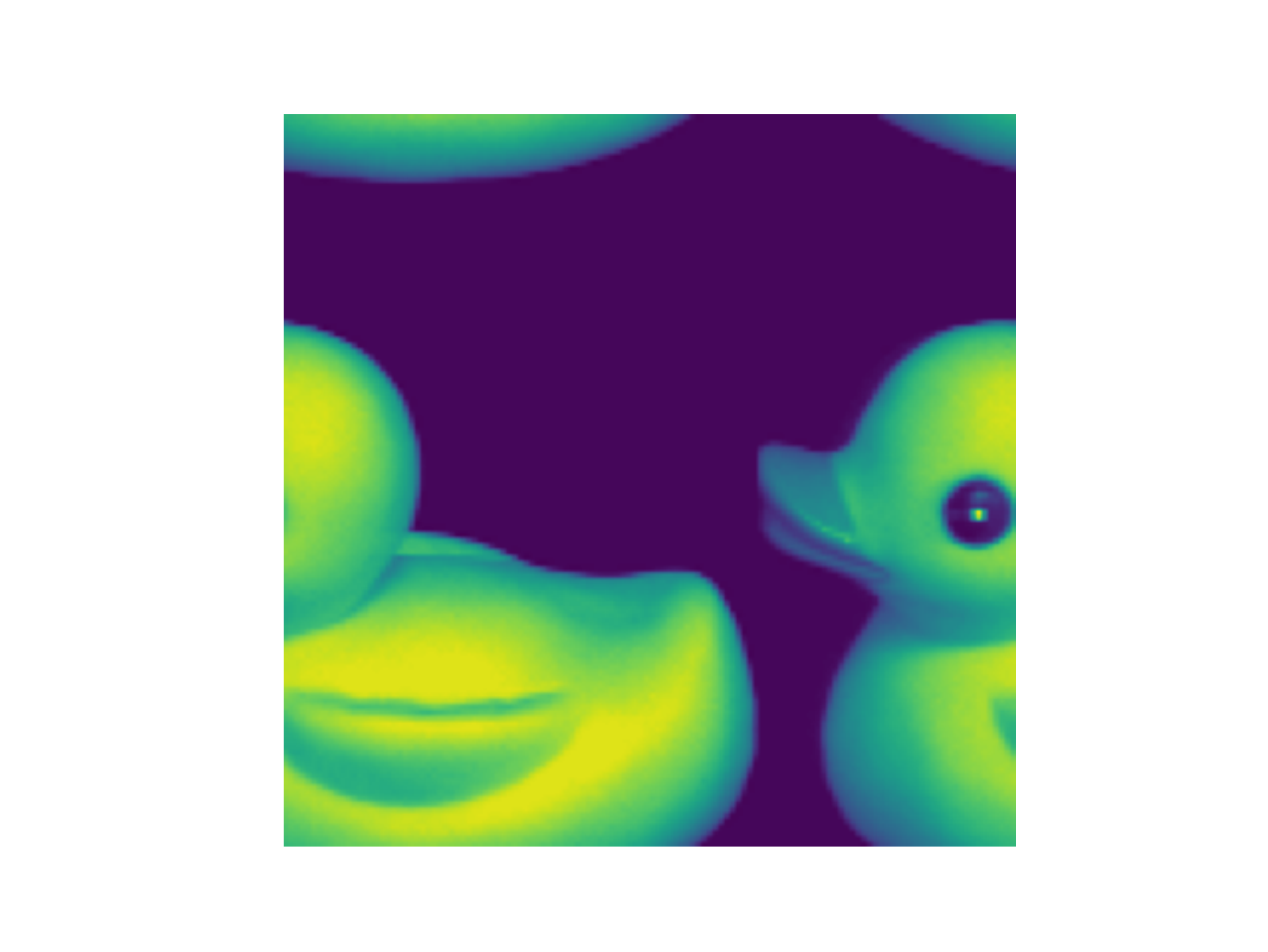}
        \caption{}
        \label{fig:duck}
    \end{subfigure}
    \begin{subfigure}[b]{0.29\textwidth}
        \centering
        \includegraphics[width=0.75\textwidth]{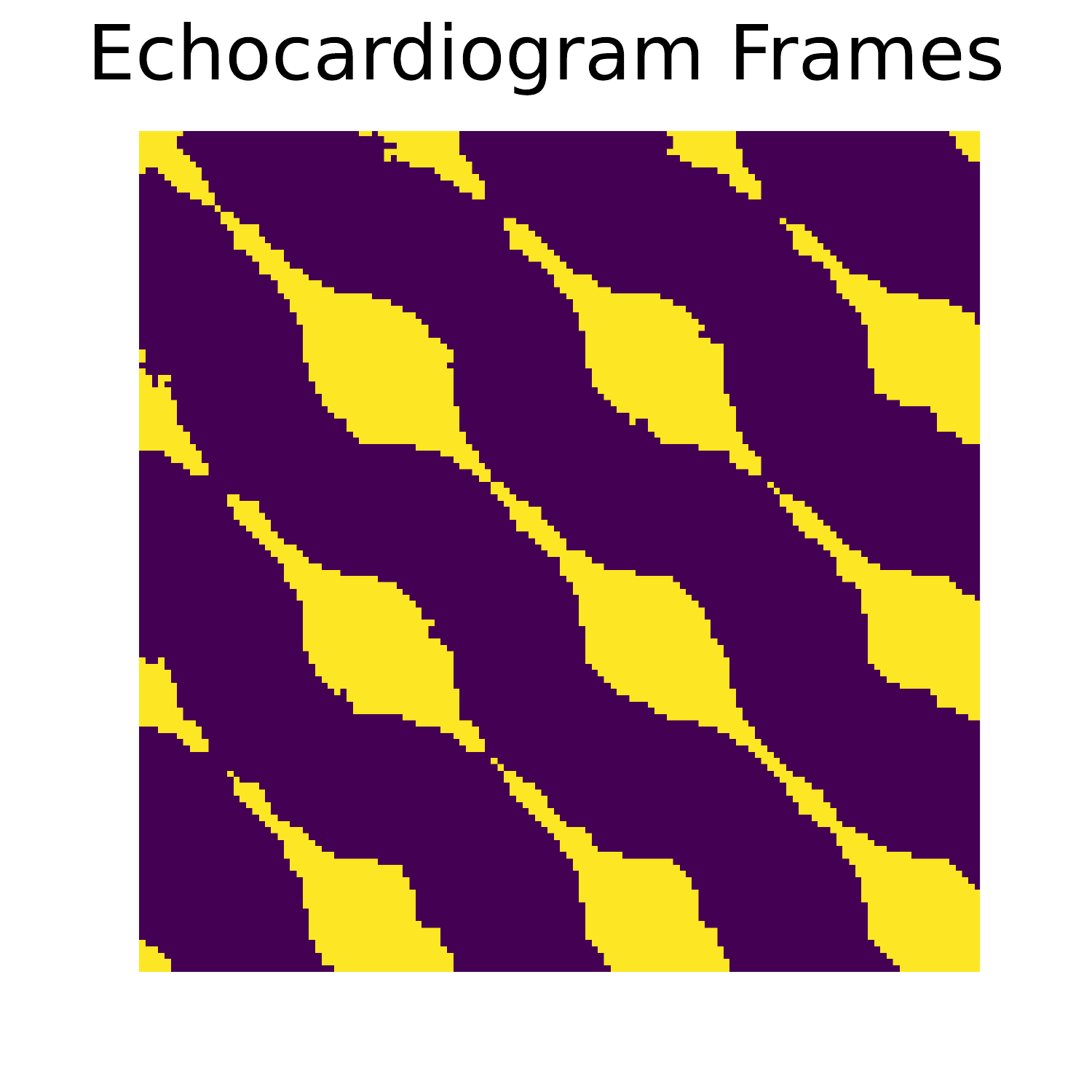}
        \caption{}
        \label{fig:heart}
    \end{subfigure}
    \caption{(a) Assortativity as we increase the number of edges in the final graph.  EiGLasso data was generated by calculating the assortativity-per-regularization penalty, and calculating the edges-per-cell from this data afterwards; this is the cause for the curve's strange behavior around 20 edges per cell. (b) The COIL-20 duck, in which the rows, columns and frames are shuffled.  (c) GmGM does an almost perfect job at recovering the duck; unfortunately, it gets cut in half.  (d) Predicted precision matrix for echocardiogram frames (yellow=connected, blue=disconnected).  The periodic structure of the heartbeat is evident in the diagonals.}
\end{figure*}

\begin{figure*}
    \centering
    \begin{subfigure}[b]{0.32\textwidth}
        \centering
        \includegraphics[width=\textwidth]{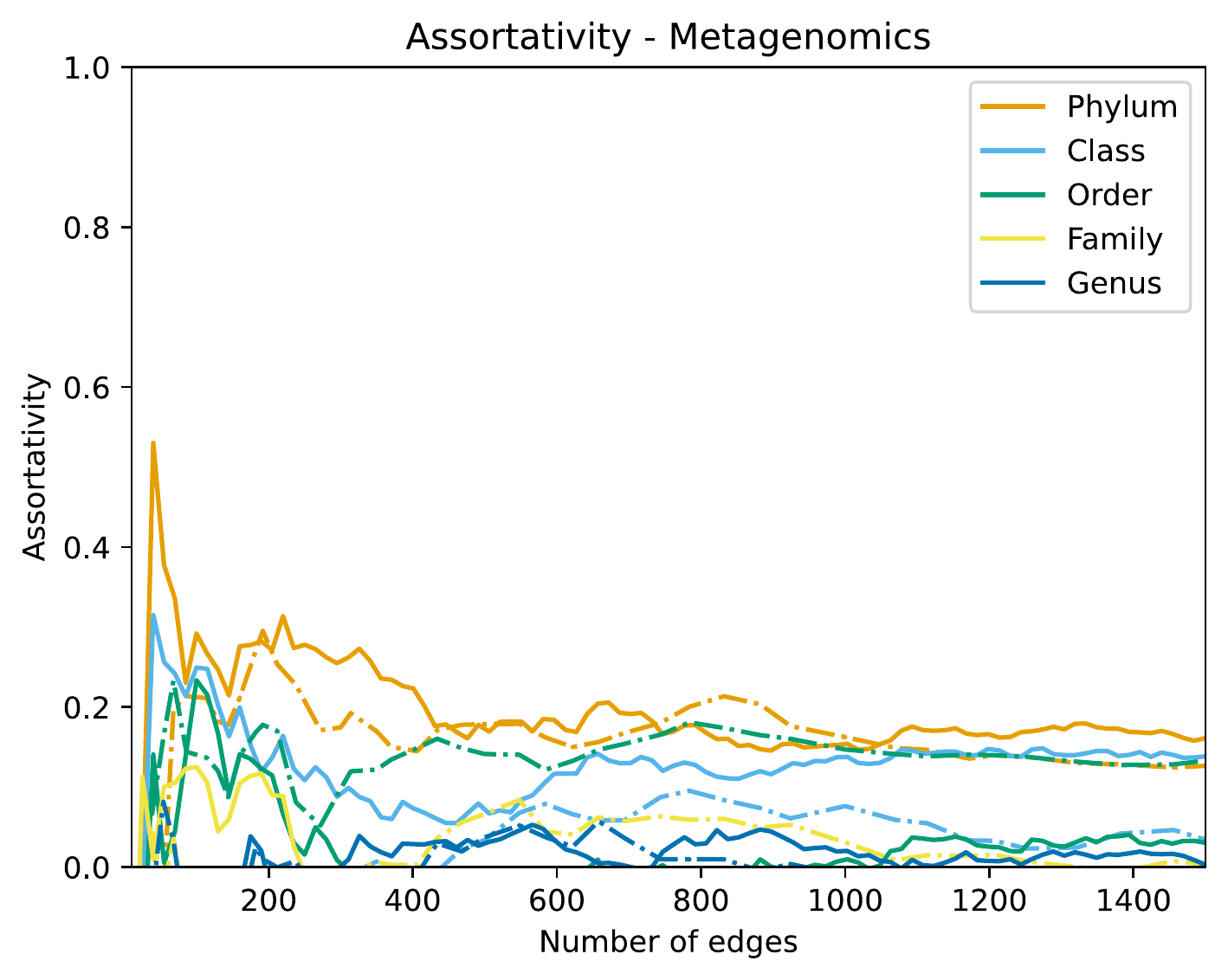}
        \caption{}
        \label{fig:assort-without}
    \end{subfigure}
    \hfill
    \begin{subfigure}[b]{0.32\textwidth}
        \centering
        \includegraphics[width=\textwidth]{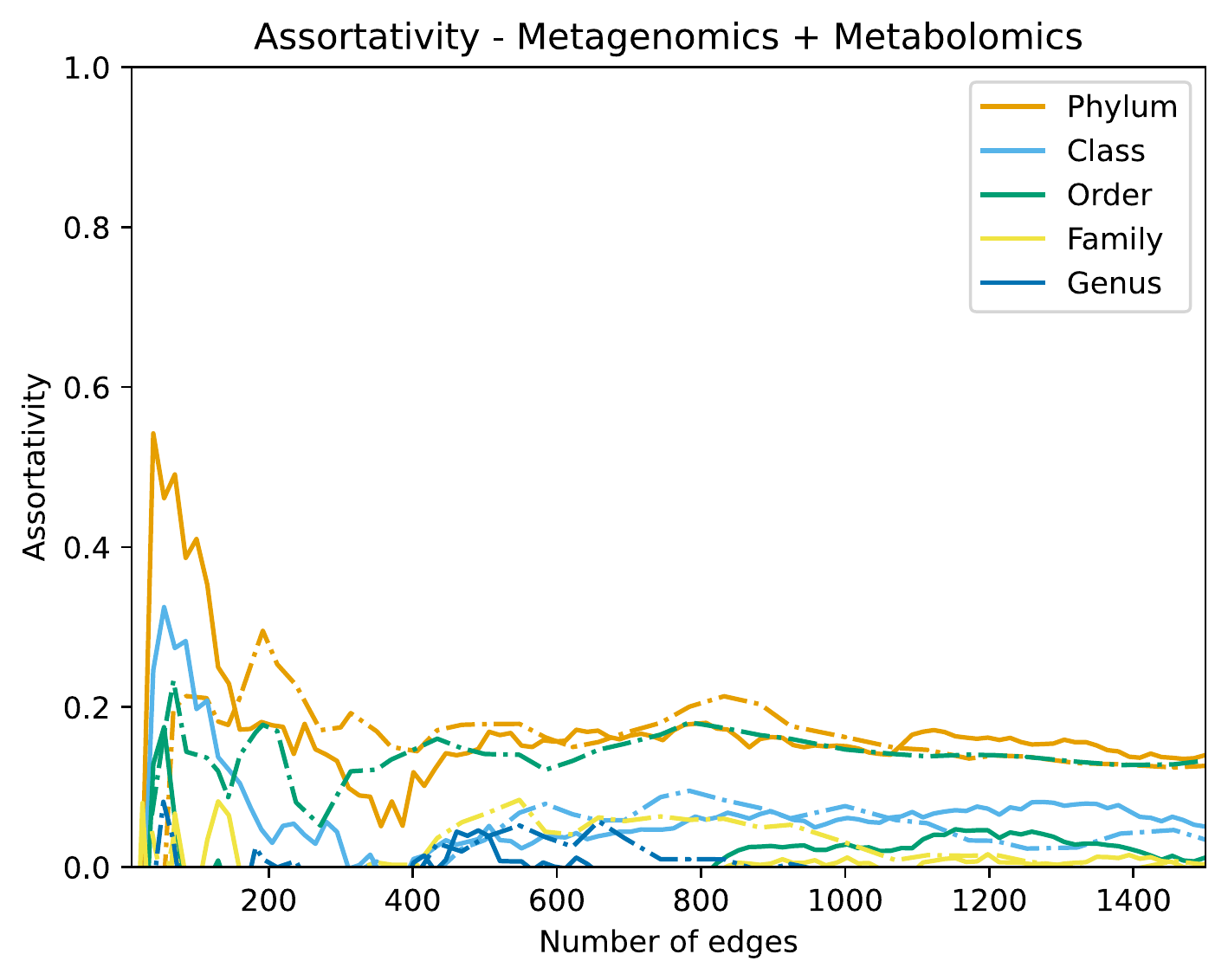}
        \caption{}
        \label{fig:assort-with}
    \end{subfigure}
    \begin{subfigure}[b]{0.32\textwidth}
        \centering
        \includegraphics[width=\textwidth]{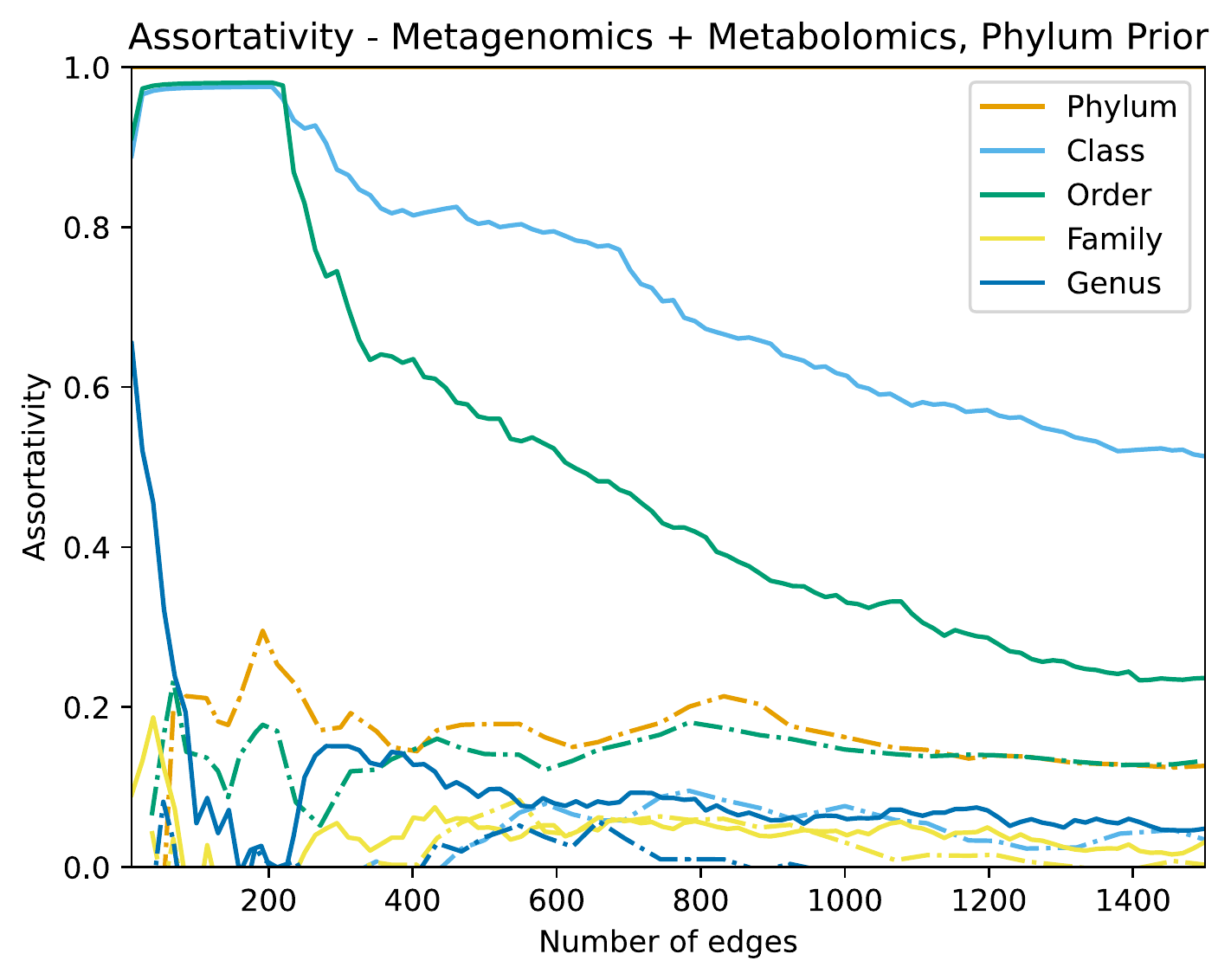}
        \caption{}
        \label{fig:lifelines-prior}
    \end{subfigure}
    \caption{Assortativity in the LifeLines-DEEP dataset, comparing our method with the Zero-inflated Log-Normal (ZiLN) model.  Solid lines represent our algorithm, dashed represent ZiLN.  In one case we show the performance of our algorithm restricted to the metagenomics dataset (a) and when augmented with the metabolomics dataset (b).  In both cases, ZiLN only has access to the metagenomics dataset, as it is a single-axis model. (c) The same test as (b), with the incorporation of phylum information as a prior (Theorem \ref{thm:evec-prior}).  Specifically, we used a Wishart prior whose parameter encoded a graph connecting two species if and only if they are in the same phylum.}
    \label{fig:assort}
\end{figure*}

\begin{figure*}
    \centering
    \begin{subfigure}[b]{0.35\textwidth}
        \centering
        \includegraphics[width=\textwidth]{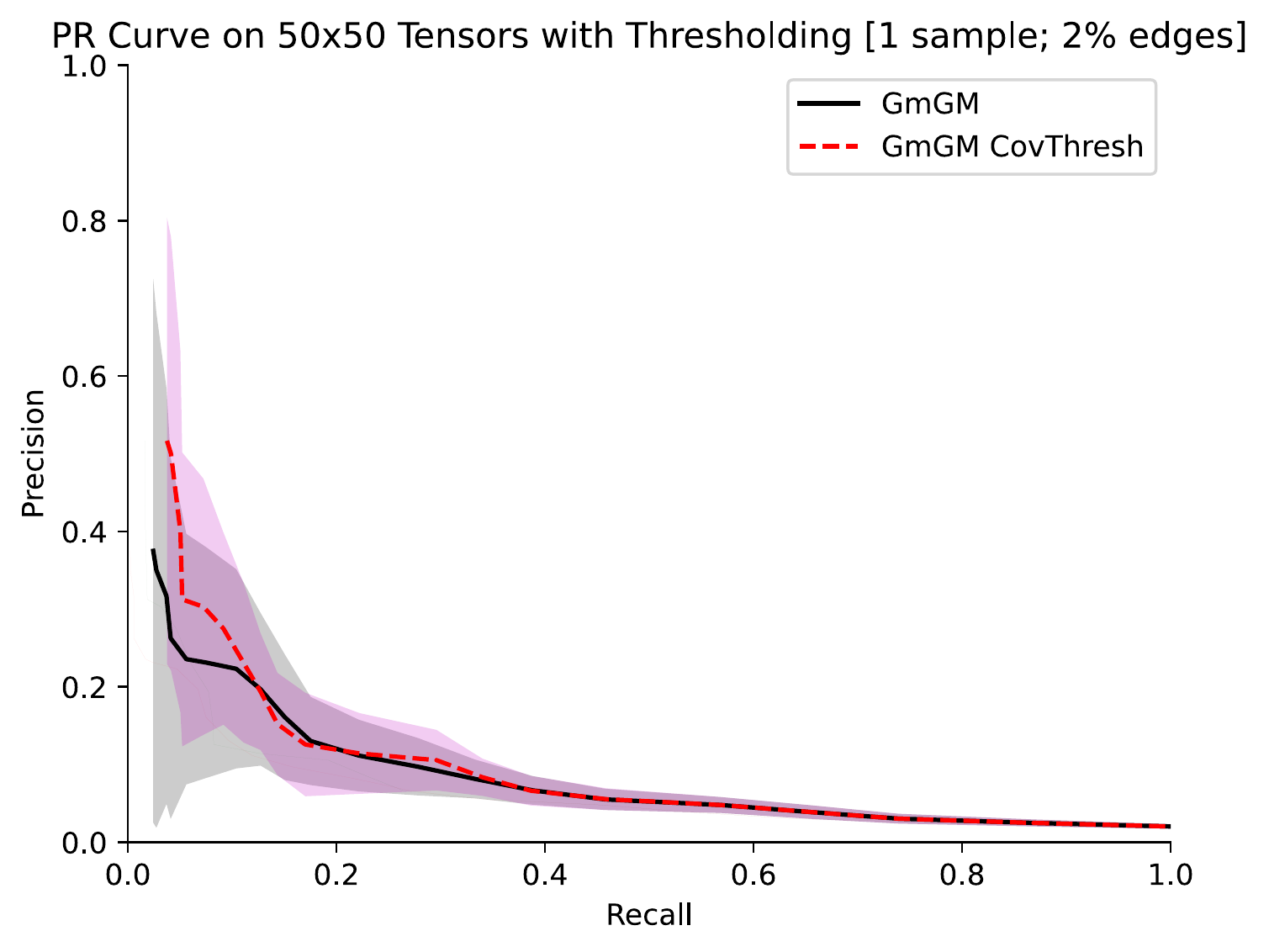}
        \caption{}
        \label{fig:cov-thresh-pr}
    \end{subfigure}
    \hfill
    \begin{subfigure}[b]{0.35\textwidth}
        \centering
        \includegraphics[width=0.75\textwidth]{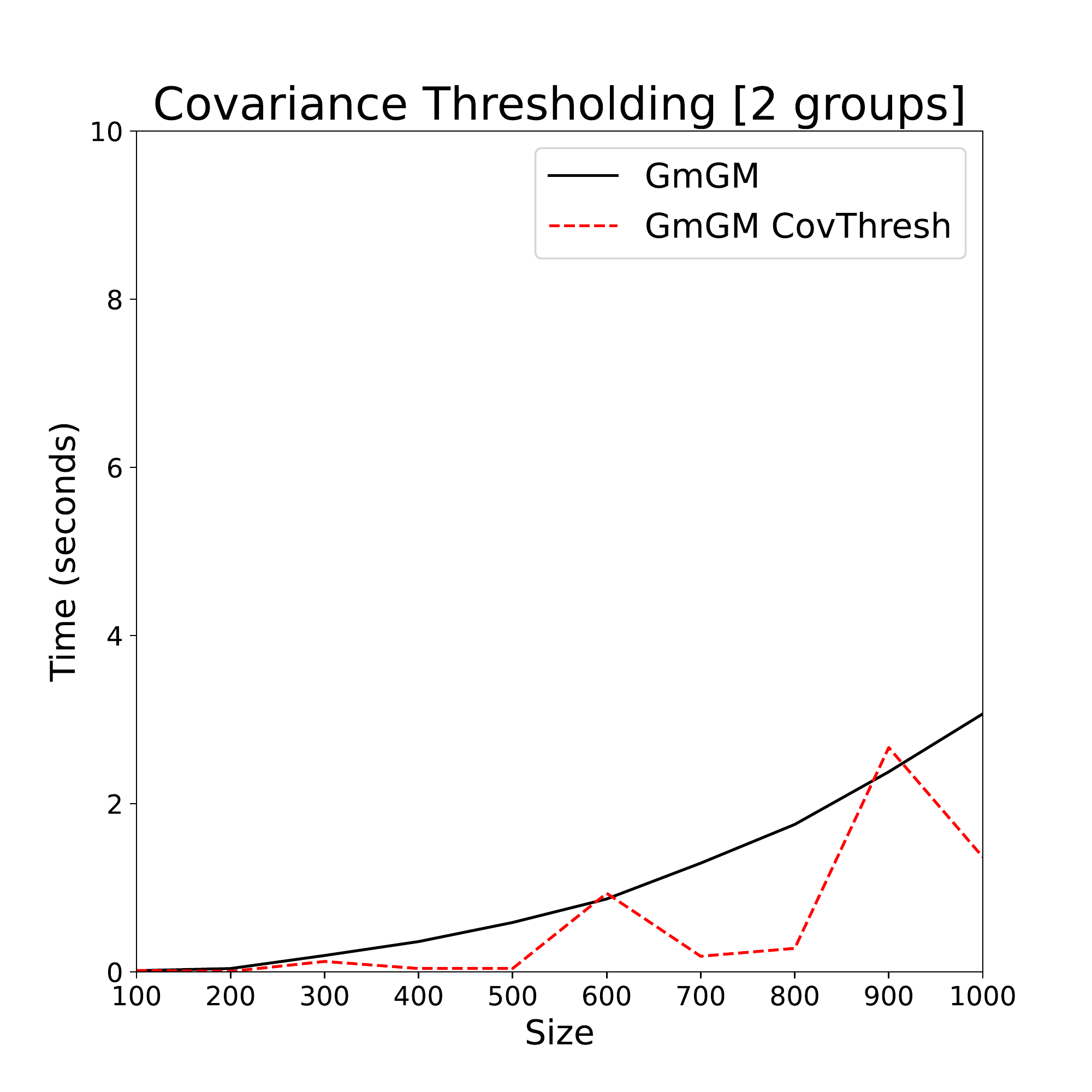}
        \caption{}
        \label{fig:cov-thresh-runtime}
    \end{subfigure}
    \caption{(a) The precision-recall curves comparing vanilla GmGM with its performance using Theorem \ref{thm:cov-thresh}'s covariance thresholding trick. (b) Runtime comparison of vanilla GmGM and the covariance thresholding trick.  Times reported were averaged over 50 runs.  Data was drawn from two disjoint versions of the same distribution as from Figure \ref{fig:pr-curves-matrix-erdos-shared}, to simulate there being multiple groups.}
    \label{fig:cov-thresh}
\end{figure*}

\begin{figure*}
    \centering
    \begin{subfigure}[b]{0.4\textwidth}
        \centering
        \includegraphics[width=\textwidth]{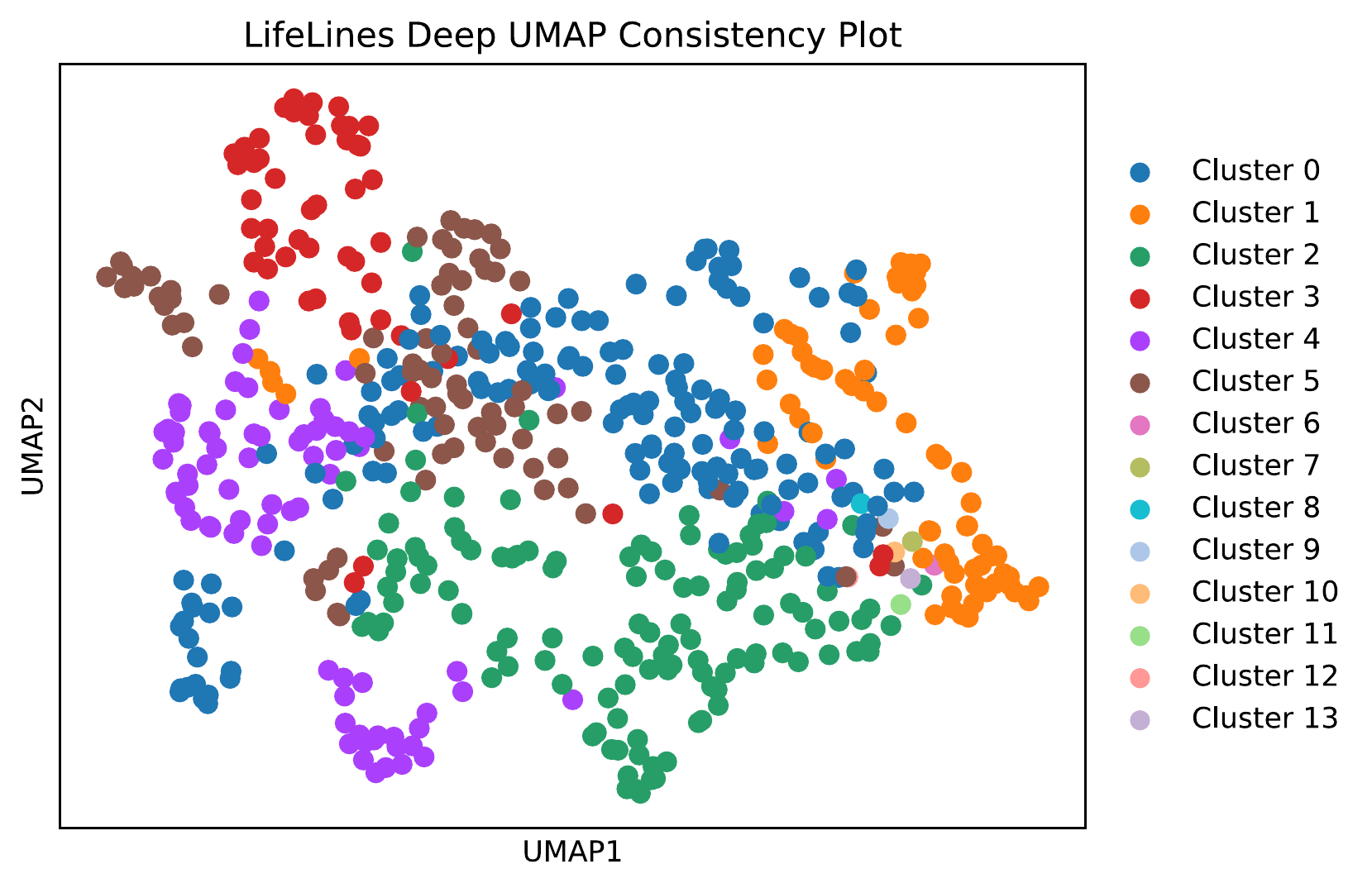}
        \caption{}
        \label{fig:lifelines-umap}
    \end{subfigure}
    \hfill
    \begin{subfigure}[b]{0.32\textwidth}
        \centering
        \includegraphics[width=\textwidth]{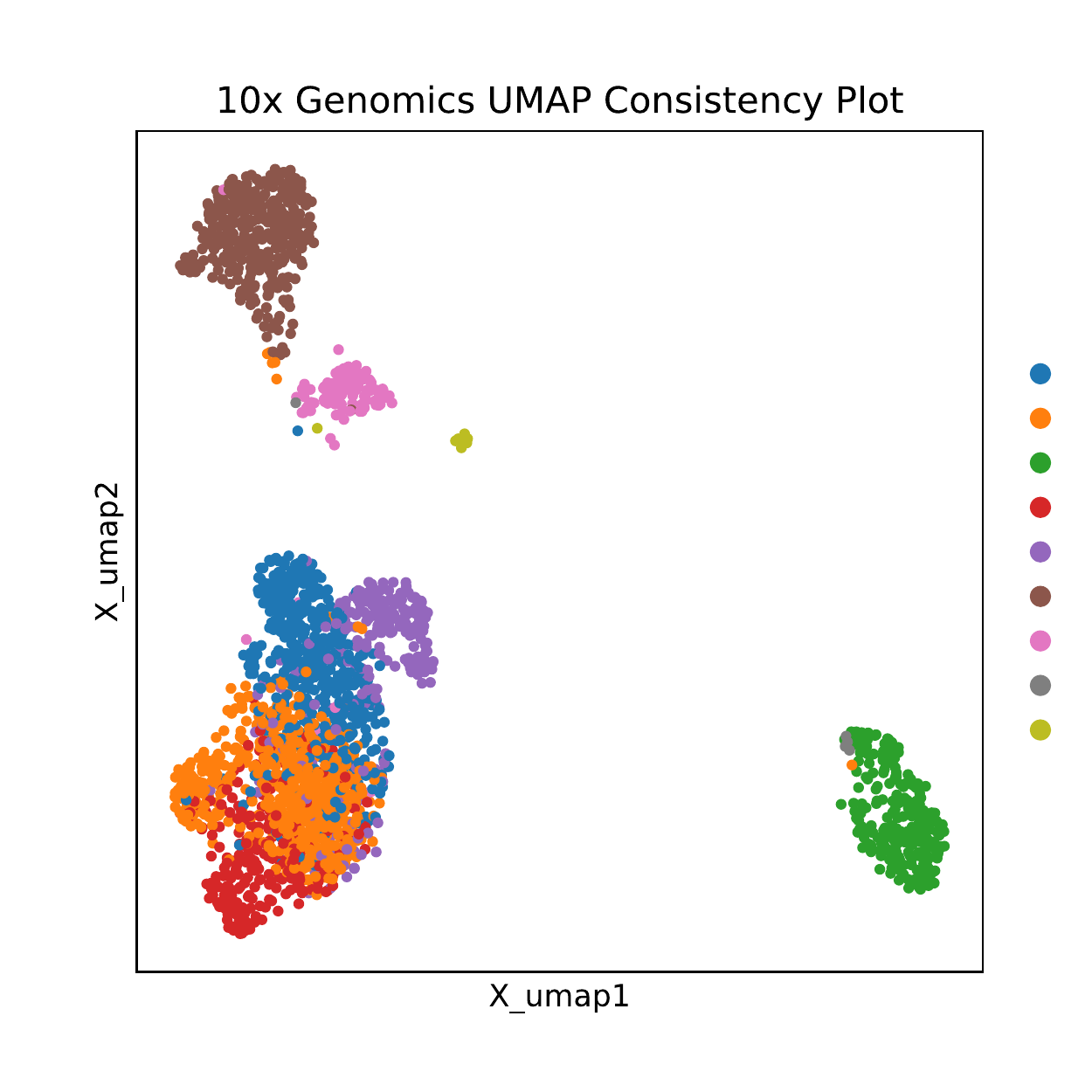}
        \caption{}
        \label{fig:10x-umap}
    \end{subfigure}
    \caption{Both (a) and (b) are UMAP plots, colored by Louvain clustering of GmGM's resultant graphs.  (a) Species from the LifeLines-DEEP dataset.  (b) Cells from the 10x Genomics Dataset.  In both cases, we can see that clusters on our graph correspond to contiguous regions of UMAP-space.}
    \label{fig:10x}
\end{figure*}

\subsection{Real Data}

We tested our method on various real datasets.  These include two video datasets (COIL-20 \parencite{nene_columbia_nodate} and EchoNet-Dynamic \parencite{ouyang_video-based_2020}), a transcriptomics dataset (E-MTAB-2805 \parencite{buettner_computational_2015}), and two multi-omics datasets (LifeLines-DEEP \parencite{tigchelaar_cohort_2015} and a 10x Genomics dataset \parencite{10x_genomics_flash-frozen_2021}).  Due to space concerns, we only briefly describe each experiment and its results here.  Full details are available in the supplementary material.

For the COIL-20 dataset, we tested whether our algorithm could be used to reconstruct the frames of a shuffled video.  We chose this test as it was the test used in the first Kronecker sum multi-axis model, BiGLasso.  BiGLasso had to heavily downsample the image (to a 9x9 image with half the frames), and flatten the rows and frames into a single axis.  Due to the speed improvements of our algorithm, and its ability to handle tensor-variate data, we were able to run our algorithm on the full-sized dataset and achieve a similar result in negligible time.  Specifically, the reconstruction of the rows, columns, and frames all had an accuracy of 99\%.  In Figure \ref{fig:shuffled-duck} we show the shuffled data, with Figure \ref{fig:duck} being the reconstruction using our algorithm.

Due to the simplicity of COIL-20, we chose to do a similar test on the more complicated EchoNet-Dynamic dataset.  We expected there to be structure in our predicted graphs, which would allow us to detect the pattern of heartbeats.  This pattern can be seen in Figure \ref{fig:heart}.  We were largely successful in this endeavor, with only one of the five videos we tested having poor results.  For full details on the experiment, and how the heartbeat was extracted, see the supplementary material.

To test GmGM on our intended use case, `omics data, we started with the E-MTAB-2805 transcriptomics dataset, as this dataset was used in the original scBiGLasso paper for exploratory analysis.  In this dataset, each cell belongs to one of three classes (G, S, or G2M) depending on its stage in the cell cycle.  To evaluate our algorithm's performance, we use assortativity.  Assortativity measures the tendency of cells in the same class to cluster together in the estimated graph; an assortativity of 0 represents no tendency for connection, with 1 being the maximum.  Negative assortativites are possible, if there is a tendency to connect to different classes.  We report the assortativities in Figure \ref{fig:mouse-assortativity}, comparing our results to EiGLasso and with the nonparanormal skeptic applied to the input.

Our fourth dataset, the LifeLines-DEEP dataset, was chosen because it was multi-modal, and because a single-axis graphical model, ZiLN \parencite{prost_zero_2021}, had already been tested on one of its modalities.  We use assortativity as a metric, as that is what was used in prior work.  Our results are comparable to their work (Figure \ref{fig:assort}).  Our algorithm's runtime of 4.7 seconds was similar to ZiLN's 3.6 seconds, showing that our multi-axis method was capable of similar runtimes as single-axis methods for the first time.  ZiLN was specifically built for metagenomics data - the compositionality assumption it makes would not be correct for most other datasets considered here, and hence we do not evaluate its performance on other datasets.  Likewise, TeraLasso and EiGLasso were too slow to evaluate assortativity curves as in Figure \ref{fig:assort}.

Finally, we tested our approach on a 10x Genomics single-cell (RNA+ATAC) dataset taken from a B cell lymphoma tumour.  We chose this dataset because it is multi-modal and quite large - 14,000 cells, 20,000 genes, and 100,000 peaks (although after preprocessing it became 2359 cells, 5350 genes, and 12485 peaks).  This dataset is unlabelled, prohibiting more quantitative analyses; our main aim in this test was to prove that the algorithm could produce results in reasonable time.  Our algorithm took 590 seconds to consider both modalities jointly and 52 seconds on just the scRNA-seq data.  In contrast, EiGLasso took more than 60,000 seconds on the scRNA-seq data - due to its excessive runtime, we did not let EiGLasso finish.

To ensure our algorithm was producing sensible results on the 10x Genomics dataset, we performed a `UMAP consistency analysis'.  Specifically, we ran Louvain clustering \parencite{blondel_fast_2008} on our graph, and viewed the clusters in UMAP-space \parencite{mcinnes_umap_2020}.  We believe this is a good test in this case, as nonlinear dimensionality reduction is used in almost every scRNA-seq study \parencite{svensson_curated_2020}.  As we can see in Figure \ref{fig:10x-umap}, the clusters on our graph correspond to sensible regions in UMAP-space as well.  In the supplementary material, we provide interpretations for some of these clusters.  Figure \ref{fig:lifelines-umap} shows another UMAP consistency plot, for the LifeLines-DEEP dataset.

In the supplementary material, we give variants of the UMAP consistency plot for both the LifeLines-DEEP and 10x Genomics datasets. These variants use different clustering methods (Leiden \parencite{traag_louvain_2019} and Ensemble Clustering \parencite{poulin_ensemble_2019}), or tSNE instead of UMAP \parencite{maaten_visualizing_2008}.  In all cases the results are broadly similar, except that Leiden on the LifeLines-Deep dataset tends to put everything in one cluster.

All of the code to run the algorithm and recreate the experiments has been made publicly available on GitHub: \href{https://github.com/AIStats-GmGM/GmGM/}{https://github.com/AIStats-GmGM/GmGM}, and we have a package ``GmGM'' on PyPI.

\section{LIMITATIONS AND FUTURE WORK}



One limitation of our algorithm is its lack of compatibility with the standard L1 penalty.  We do show that our restricted L1 penalty leads to better performance in synthetic data, even decisively outperforming L1-penalized prior work in some cases (Figure \ref{fig:pr-curves-tensor-erdos}).  In the supplementary material, we also show that our restricted penalty leads to much better performance on the E-MTAB-2805 transcriptomics dataset.  However, it would be nice to incorporate the standard L1 penalty, as it is a well-studied and well-accepted technique.  This would be a rewarding avenue for future work.

Another limitation of our algorithm is that it does not handle the case in which two axes partially overlap.  For example, in the LifeLines-DEEP dataset, not every person is in both the metabolomics and metagenomics datasets.  For now, our solution to this is to restrict ourselves to the subset of indices that are available in all modalities sharing that axis (this was about 90\% of people in the LifeLines-DEEP dataset).  Prior multi-axis methods could only handle a single modality, whereas this partial overlap problem only arises when considering multiple modalities - a case which we are the first to consider.  We intend to relax this restriction in future work.

One notable effect of our algorithm's speed improvements is that we are now limited by RAM rather than runtime; this affected our experiments, wherein we were more agressive with our quality control filtering of the 100,000 ATAC peaks in the 10x dataset.  Before our work, algorithms were limited by how long it would take to run rather than how much space was needed - our algorithm is the first to be fast enough to hit memory limits.  Our algorithm has a theoretically optimal space usage of $O(\sum_\ell d_\ell^2)$, since the output is a set of $d_\ell \times d_\ell$ precision matrices.  However, given that we are often interested in \textit{sparse} matrices, an avenue for future work would be to create variants of the algorithm that are capable of leveraging this sparsity.  As our algorithm (along with the fastest prior work, EiGLasso and TeraLasso) depends on eigendecompositions, this would be a nontrivial task.

\section{CONCLUSION}

We have created a novel model, GmGM, which successfully generalizes Gaussian graphical models to the common scenario of multi-modal datasets.  Furthermore, we demonstrated that our algorithm is significantly faster than prior work focusing on Gaussian multi-graphical models while still preserving state-of-the-art performance.  These improvements allow multi-graphical models to be applied to datasets with axes of length in the tens of thousands.  Additionally, we showed how to integrate a couple natural priors into our model, as well as provided a technique to partition datasets into smaller portions before analysis.  The latter is particularly useful as runtime scales cubically after the computation of the empirical Gram matrices; partitioning a dataset in half provides an eightfold speedup.  Finally, we demonstrated the application of our algorithm on five real-world datasets to prove its efficacy.

\subsubsection*{Acknowledgements}

We thank the participants and the staff of LifeLines-DEEP for their collaboration. Funding for the project was provided by the Top Institute Food	and Nutrition Wageningen grant GH001. The sequencing was carried out in collaboration with the Broad Institute.

We also thank the reviewers for their comments - this paper has been much improved thanks to their feedback.


\printbibliography

\section*{Checklist}



 \begin{enumerate}

 \item For all models and algorithms presented, check if you include:
 \begin{enumerate}
   \item A clear description of the mathematical setting, assumptions, algorithm, and/or model. \textbf{Yes}; all proofs, and the pseudocode, are available in the supplementary material.
   \item An analysis of the properties and complexity (time, space, sample size) of any algorithm. \textbf{Yes}; the time and space complexity are given in the supplementary material.  All experiments, synthetic and real, were done with sample size=1 as that is the case in the real world when using this type of algorithm.
   \item (Optional) Anonymized source code, with specification of all dependencies, including external libraries. \textbf{Yes}; the source code has been put onto GitHub under a custom-made anonymized account for this submission.  All libraries, including their version numbers, are given in the README of that repository.
 \end{enumerate}

 \item For any theoretical claim, check if you include:
 \begin{enumerate}
   \item Statements of the full set of assumptions of all theoretical results. \textbf{Yes}
   \item Complete proofs of all theoretical results. \textbf{Yes}; all proofs, and the psuedocode, are available in the supplementary material.
   \item Clear explanations of any assumptions. \textbf{Yes}; all theorems give their specific assumptions within them, and our model is stated clearly in Section \ref{sec:model}.  
 \end{enumerate}

 \item For all figures and tables that present empirical results, check if you include:
 \begin{enumerate}
   \item The code, data, and instructions needed to reproduce the main experimental results (either in the supplemental material or as a URL). \textbf{Yes}; in the GitHub repo linked in supplementary material
   \item All the training details (e.g., data splits, hyperparameters, how they were chosen). \textbf{Yes}; in the supplementary material, we give details on thresholding techniques.
         \item A clear definition of the specific measure or statistics and error bars (e.g., with respect to the random seed after running experiments multiple times). \textbf{Yes}; for each real dataset and the synthetic data, details of the experiment not contained in the main paper are contained in the supplementary material.
         \item A description of the computing infrastructure used. (e.g., type of GPUs, internal cluster, or cloud provider). \textbf{Yes}; we stated our computer's specs in Section \ref{sec:synth}.
 \end{enumerate}

 \item If you are using existing assets (e.g., code, data, models) or curating/releasing new assets, check if you include:
 \begin{enumerate}
   \item Citations of the creator If your work uses existing assets. \textbf{Yes}; we cite the datasets used
   \item The license information of the assets, if applicable. \textbf{Yes}; All details about the datasets are available in the supplementary material 
   \item New assets either in the supplemental material or as a URL, if applicable. \textbf{Not applicable}
   \item Information about consent from data providers/curators. \textbf{Yes}; All details about datasets are available in the supplementary material
   \item Discussion of sensible content if applicable, e.g., personally identifiable information or offensive content. \textbf{Not applicable}
 \end{enumerate}

 \item If you used crowdsourcing or conducted research with human subjects, check if you include:
 \begin{enumerate}
   \item The full text of instructions given to participants and screenshots. \textbf{Not applicable}
   \item Descriptions of potential participant risks, with links to Institutional Review Board (IRB) approvals if applicable. \textbf{Not applicable}
   \item The estimated hourly wage paid to participants and the total amount spent on participant compensation. \textbf{Not applicable}
 \end{enumerate}

 \end{enumerate}

\end{document}


%

%


\onecolumn
\aistatstitle{Supplementary Materials for: \\ GmGM: a Fast Multi-Axis Gaussian Graphical Model}

\section*{OUTLINE OF THE SUPPLEMENTARY MATERIAL}

\begin{enumerate}
    \item \hyperref[sec:notation]{Notation}
    \item \hyperref[sec:proofs]{Proofs}
    \begin{enumerate}
        \item \hyperref[sec:permutations]{Permutations}
        \item \hyperref[sec:deriv-prob-dense]{Derivation of the probability density function}
        \item \hyperref[sec:gradient]{Gradient}
        \item \hyperref[sec:mle-eigenvectors]{Maximum likelihood estimate for the eigenvectors}
        \item \hyperref[sec:mle-eigenvalues]{Maximum likelihood estimate for the eigenvalues}
        \item \hyperref[sec:incorp-priors]{Incorporation of priors}
        \item \hyperref[sec:covariance-thresholding]{Covariance thresholding}
    \end{enumerate}
    \item \hyperref[sec:dependencies]{Dependencies}
    \item \hyperref[sec:experiments]{Experiments}
    \begin{enumerate}
        \item \hyperref[sec:synthetic-data]{Synthetic data}
        \item \hyperref[sec:coil]{COIL video}
        \item \hyperref[sec:echonet]{EchoNet-Dynamic ECGs}
        \item \hyperref[sec:mouse]{Mouse embryo stem cell transcriptomics}
        \item \hyperref[sec:lifelines]{LifeLines-DEEP metagenomics + metabolomics}
        \item \hyperref[sec:10x]{10x Genomics flash frozen lymph nodes}
    \end{enumerate}
    \item \hyperref[sec:regularization]{Regularization}
    \item \hyperref[sec:complexity]{Asymptotic complexity}
    \item \hyperref[sec:centering]{Centering the data}
\end{enumerate}

\newpage

In the supplementary material, we give proofs of all theorems (Section \ref{sec:proofs}), provide code, pseudocode, and computer specs (Section \ref{sec:dependencies}), give more informations on our experiments (Section \ref{sec:experiments}), describe how we incorporated regularization (Section \ref{sec:regularization}), and finally give an overview of the asymptotic complexity of ours and other algorithms (Section \ref{sec:complexity}).

Lemmas \ref{lem:rearrangement} and \ref{lem:ks-vec} are used to derive the Kronecker sum normal distributions's pdf, which has been known in prior work; we have included it for completeness.  All other lemmas and theorems are novel.

\section{NOTATION}
\label{sec:notation}

In addition to the notation used in the main paper, we also introduce further notation to aid in the proofs.  For working with tensors, \textcite{kolda_tensor_2009} proved to be an invaluable resource; we have borrowed their notation in most cases.  The only exception is that we have chosen to denote the $\ell$-mode matricization of a tensor $\mathcal{T}$ as $\mathrm{mat}_\ell\left[\mathcal{T}\right]$ rather than $\mathcal{T}_{\left(\ell\right)}$, to highlight its similarity to $\mathrm{vec}\left[\mathcal{T}\right]$ and free up the subscript for other purposes.

To keep track of lengths of axes, we define the following notation:

\begin{itemize}
    \item $d_\ell^\gamma$ is the length of axis $\ell$
    \item $d_{>\ell}^\gamma$ is the product of lengths of all axes after $\ell$
    \item $d_{<\ell}^\gamma$ is the product of lengths of all axes before $\ell$
    \item $d_{\backslash\ell}^\gamma$ is the product of lengths of all axes except for $\ell$
    \item $d_{\forall}^\gamma$ is the product of lengths of all axes (i.e. the number of elements in $\mathcal{D}^\gamma$)
    \item $d_\forall = \sum_\gamma d_\forall^\gamma$ is the total number of elements across all datasets
\end{itemize}

In prior work, $d_\ell$ has been used to represent the lengths of axes but $m_\ell$ was used where we write $d_{\backslash\ell}$ (such as in \textcite{greenewald_tensor_2019}).  As prior work also used $\backslash\ell$ to represent leaving out the $\ell$th axis in other contexts (such as in \textcite{kalaitzis_bigraphical_2013}), and the analogous definitions of $d_{>\ell}$ and $d_{<\ell}$ were convenient for use in proofs, we chose to introduce $d_{\backslash\ell}$ as the variable to represent leave-one-out length products.  By representing all of these related concepts with similar symbols, we hope the maths will be easier to parse.

We will let $\mathbf{I}_a$ be the $a\times a$ identity matrix, which allows a concise definition of the Kronecker sum: $\bigoplus_\ell \mathbf{\Psi}_\ell = \sum_\ell \mathbf{I}_{d_{<\ell}} \otimes \mathbf{\Psi}_\ell \otimes \mathbf{I}_{d_{>\ell}}$.

We make frequent use of the vectorization $\mathrm{vec}\left[\mathbf{M}\right]$ of a matrix $\mathbf{M}$, and more generally of a tensor $\mathrm{vec}\left[\mathcal{T}\right]$.  We adopt the rows-first convention of vectorization, such that:

\begin{align}
    \mathrm{vec}\begin{bmatrix}
    1 & 2 \\
    3 & 4
    \end{bmatrix} &= \begin{bmatrix}
        1 & 2 & 3 & 4
    \end{bmatrix}
\end{align}

While columns-first is more common, rows-first is more natural when we adopt the convention that rows are the first axis of tensor; this is the convention that matricization uses, and matricization is much more important for our work due to its role in defining the Gram matrices.  Note that, for matrices, a rows-first vectorization of $\mathbf{M}$ is equivalent to a columns-first vectorization of $\mathbf{M}^T$, so there is no fundamental difference between the two.  For vectorizing a tensor, we proceed by stacking the rest of the axes in order, such that an element $(i_1, ..., i_K)$ in $\mathcal{T}$ gets mapped to the element $\sum_\ell i_\ell d_{<\ell}$ in $\mathrm{vec}\left[\mathcal{T}\right]$.

We define the Gram matrices as $\mathbf{S}^{\gamma}_{\ell} = \mathrm{mat}_\ell\left[\mathcal{D}^\gamma\right]\mathrm{mat}_\ell\left[\mathcal{D}^\gamma\right]^T$.  Typically we consider only the one-sample case.  If you have multiple samples, indexed by a subscript $i$, then the Gram matrix becomes an average: $\mathbf{S}^{\gamma}_{\ell} = \frac{1}{n}\sum_i^n\mathrm{mat}_\ell\left[\mathcal{D}_i^\gamma\right]\mathrm{mat}_\ell\left[\mathcal{D}_i^\gamma\right]^T$.

An essential concept is that of the `stridewise-blockwise trace', defined as:

\begin{align}
    \mathrm{tr}^{a}_{b}\left[\mathbf{M}\right] &= \left[\mathrm{tr}\left[\mathbf{M}\left(\mathbf{I}_{a} \otimes \mathbf{J}^{ij}\otimes \mathbf{I}_{b}\right)\right]\right]_{ij}
\end{align}

Where $\mathbf{J}^{ij}$ is the matrix of zeros except at $(i, j)$ where it has a 1.  It is a generalization of the blockwise trace considered by \textcite{kalaitzis_bigraphical_2013}, and is related to the $\mathrm{proj}_\mathcal{K}$ operation defined by \textcite{greenewald_tensor_2019}.  Specifically, $\mathrm{proj}_\mathcal{K}\left[\mathbf{M}\right]$ is equivalent to $\bigoplus_\ell \mathrm{tr}^{d_{<\ell}}_{d_{>\ell}}\left[\mathbf{M}\right]$ up to an additive diagonal factor (Lemma 33 from \textcite{greenewald_tensor_2019}).  $\mathrm{proj}_\mathcal{K}\left[\mathbf{M}\right]$ was defined to be the matrix that best approximates $\mathbf{M}$ (in terms of the Frobenius norm) while being Kronecker-sum-decomposable.  This matrix is not unique; the choice by \textcite{greenewald_tensor_2019} to include an additive factor was to enforce $\mathrm{tr}\left[\mathrm{proj}_\mathcal{K}\left[\mathbf{M}\right]\right] = 0$.  We do not wish to enforce this constraint as it would be difficult to preserve in the multi-tensor case.

The parameter $b$ of the stridewise-blockwise trace partitions the $m\times m$ matrix $\mathbf{M}$ into a block matrix with $b\times b$ blocks of size ($\frac{m}{b}\times\frac{m}{b}$).  The parameter $a$ then partitions these blocks into a `strided' matrix with $a\times a$ strides containing $\frac{m}{ab}\times\frac{m}{ab}$ blocks.  We take the trace of each stride, and the final matrix is the matrix of these traces.  As this is conceptually complicated, we provide an example.

\begin{align}
    &\mathrm{tr}^2_2\begin{bmatrix}
        1 & 2 & 3 & 4 & 5 & 6 & 7 & 8 \\
        1 & 2 & 3 & 4 & 5 & 6 & 7 & 8 \\
        1 & 2 & 3 & 4 & 5 & 6 & 7 & 8 \\
        1 & 2 & 3 & 4 & 5 & 6 & 7 & 8 \\
        1 & 2 & 3 & 4 & 5 & 6 & 7 & 8 \\
        1 & 2 & 3 & 4 & 5 & 6 & 7 & 8 \\
        1 & 2 & 3 & 4 & 5 & 6 & 7 & 8 \\
        1 & 2 & 3 & 4 & 5 & 6 & 7 & 8
    \end{bmatrix} \\=& \mathrm{tr}^2\begin{bmatrix}
        \mathrm{tr}\begin{bmatrix} 1 & 2 \\ 1 & 2 \end{bmatrix} &  \mathrm{tr}\begin{bmatrix} 3 & 4 \\ 3 & 4 \end{bmatrix} &  \mathrm{tr}\begin{bmatrix} 5 & 6 \\ 5 & 6 \end{bmatrix} &  \mathrm{tr}\begin{bmatrix} 7 & 8 \\ 7 & 8 \end{bmatrix} \\
        \mathrm{tr}\begin{bmatrix} 1 & 2 \\ 1 & 2 \end{bmatrix} &  \mathrm{tr}\begin{bmatrix} 3 & 4 \\ 3 & 4 \end{bmatrix} &  \mathrm{tr}\begin{bmatrix} 5 & 6 \\ 5 & 6 \end{bmatrix} &  \mathrm{tr}\begin{bmatrix} 7 & 8 \\ 7 & 8 \end{bmatrix} \\
        \mathrm{tr}\begin{bmatrix} 1 & 2 \\ 1 & 2 \end{bmatrix} &  \mathrm{tr}\begin{bmatrix} 3 & 4 \\ 3 & 4 \end{bmatrix} &  \mathrm{tr}\begin{bmatrix} 5 & 6 \\ 5 & 6 \end{bmatrix} &  \mathrm{tr}\begin{bmatrix} 7 & 8 \\ 7 & 8 \end{bmatrix} \\
        \mathrm{tr}\begin{bmatrix} 1 & 2 \\ 1 & 2 \end{bmatrix} &  \mathrm{tr}\begin{bmatrix} 3 & 4 \\ 3 & 4 \end{bmatrix} &  \mathrm{tr}\begin{bmatrix} 5 & 6 \\ 5 & 6 \end{bmatrix} &  \mathrm{tr}\begin{bmatrix} 7 & 8 \\ 7 & 8 \end{bmatrix} \\
    \end{bmatrix}
    \\=& \mathrm{tr}^2\begin{bmatrix}
        3 & 7 & 11 & 15 \\
        3 & 7 & 11 & 15 \\
        3 & 7 & 11 & 15 \\
        3 & 7 & 11 & 15 \\
    \end{bmatrix}
    \\=& \begin{bmatrix}
        \mathrm{tr}\begin{bmatrix}3 & 11 \\ 3 & 11\end{bmatrix} & \mathrm{tr}\begin{bmatrix}7 & 15 \\ 7 & 15\end{bmatrix} \\
        \mathrm{tr}\begin{bmatrix}3 & 11 \\ 3 & 11\end{bmatrix} & \mathrm{tr}\begin{bmatrix}7 & 15 \\ 7 & 15\end{bmatrix}
    \end{bmatrix} \label{eq:strblck}
    \\=& \begin{bmatrix}
        14 & 22 \\ 14 & 22
    \end{bmatrix}
\end{align}

Notice the construction of the `strides' in Line \ref{eq:strblck} - the parameter of $2$ told us to take the trace of the most `spread out' evenly spaced $2\times 2$ strided submatrix.

\section{PROOFS}
\label{sec:proofs}

We will assume that no dataset contains repeated axes (i.e. no single tensor has two axes represented by the same graph), as this greatly affects the derived gradients.  Shared axes - two tensors having one or more axes in common - are allowed.  The case of shared axes is, after all, the whole point of developing this extension to prior work.

\subsection{Permutations}
\label{sec:permutations}

Note that both $\mathrm{vec}\left[\mathrm{mat}_1 \left[\mathcal{D}^\gamma\right]\right]$ and $\mathrm{vec}\left[\mathrm{mat}_\ell \left[\mathcal{D}^\gamma\right]\right]$ are row vectors containing the same elements, just in a different order.  This means that there is a permutation matrix $\mathbf{P}_{\ell\rightarrow 1}$ such that $\mathrm{vec}\left[\mathrm{mat}_1 \left[\mathcal{D}^\gamma\right]\right]^T\mathbf{P}_{\ell\rightarrow 1} = \mathrm{vec}\left[\mathrm{mat}_\ell \left[\mathcal{D}^\gamma\right]\right]^T$.

\begin{lemma}[Rearrangement lemma]
$\mathbf{P}_{\ell\rightarrow 1}\left(\mathbf{I}_{d_{<\ell}}\otimes\mathbf{\Psi}_\ell\otimes\mathbf{I}_{d_{>\ell}}\right)\mathbf{P}_{\ell\rightarrow 1}^T = \mathbf{\Psi}_\ell\otimes\mathbf{I}_{d_{\backslash\ell}}$
\label{lem:rearrangement}
\end{lemma}

\begin{proof}
    While $\mathrm{vec}$, $\mathrm{mat}_\ell$ and $\bigotimes$ are defined as operations on matrices, for the purposes of permutations we can consider them as operations on indices.  We can express them as follows:

    \begin{align}
        \mathrm{vec}: (i_1, ..., i_K) &\rightarrow \left(\sum_\ell i_\ell d_{<\ell}\right) \\
        \mathrm{mat}_\ell: (i_1, ..., i_K) &\rightarrow \left(i_\ell, \sum_{\ell'<\ell} i_{\ell'}d_{<\ell'} + \sum_{\ell'>\ell} i_{\ell'}\frac{d_{<\ell'}}{d_\ell}\right) \\
        \bigotimes: \left( (i_1^1, i_1^2), ..., (i_K^1, i_K^2) \right) &\rightarrow \left(\sum_\ell i^1_\ell d_{<\ell}, \sum_\ell i^2_\ell d_{<\ell}\right)
    \end{align}

We'll consider just the rows of $\bigotimes$, $\bigotimes_{rows}$ - although the same argument applies with columns:

\begin{align}
    \bigotimes_{rows}: \left(i_1^1, ..., i_K^1 \right) &\rightarrow \left(\sum_\ell i^1_\ell d_{<\ell}\right)
\end{align}

Finally, we'll introduce the permutation operation $\sigma_{\ell\rightarrow 1}$ that will change the order of our Kronecker product:

\begin{align}
    \sigma_{\ell\rightarrow 1}: \left( (i_1^1, i_1^2), ..., (i_K^1, i_K^2) \right) &\rightarrow \left(( (i_\ell^1, i_\ell^2), (i_1^1, i_1^2), ..., (i_{\ell-1}^1, i_{\ell-1}^2), (i_{\ell+1}^1, i_{\ell+1}^2), ..., (i_K^1, i_K^2) \right)
\end{align}

And again without loss of generality we restrict ourself to $\sigma_{\ell\rightarrow 1}^{rows}$:

\begin{align}
    \sigma_{\ell\rightarrow 1}^{rows}: \left( i_1^1, ..., i_K^1 \right) &\rightarrow \left( i_\ell^1, i_1^1, ..., i_{\ell-1}^1, i_{\ell+1}^1, ..., i_K^1 \right)
\end{align}

After a Kronecker product our indices are in the form $\sum_\ell i_\ell d_{<\ell}$, and if we were to reorder it with $\sigma_{\ell\rightarrow 1}$ they would be in the form $i_\ell + \sum_{\ell'<\ell} i_{\ell'} d_{<\ell'} d_\ell + \sum_{\ell'>\ell} i_{\ell'} d_{<\ell'}$.  Likewise, if we had matricized it we would have $\left(i_\ell, \sum_{\ell'<\ell} i_{\ell'}d_{<\ell'} + \sum_{\ell'>\ell} i_{\ell'}\frac{d_{<\ell'}}{d_\ell}\right)$, which is vectorized to $i_\ell + \sum_{\ell'<\ell} i_{\ell'} d_{<\ell'} d_\ell + \sum_{\ell'>\ell} i_{\ell'} d_{<\ell'}$.  These reorderings are the same, and hence the matrix that represents it is $\mathbf{P}_{\ell\rightarrow 1}$.

\end{proof}
\newpage
\subsection{Derivation of the probability density function}
\label{sec:deriv-prob-dense}

Recall that the Kronecker-sum-structured normal distribution for a single tensor is defined as follows:

\begin{align}
    \vecop{\mathcal{D}^\gamma} \sim \mathcal{N}\left(\mathbf{0}, \left(\bigoplus_{\ell\in\gamma}\mathbf{\Psi}_\ell\right)^{-1}\right) &\iff \mathcal{D}^\gamma \sim \mathcal{N}_{KS}\left(\left\{\Psi_\ell\right\}_{\ell\in\gamma}\right)
\end{align}

The log-likelihood for this distribution is given in \textcite{kalaitzis_bigraphical_2013} for the matrix case and \textcite{greenewald_tensor_2019} for the general tensor case.  However, neither of these papers provide a derivation.  As the full derivation will motivate the construction of lemmas useful for the proofs of Theorems \ref{thm:evec}, we will give it here. First, observe that the density function is that of a normal distribution.

\begin{align}
    p\left(\mathcal{D}^\gamma\right) &= \frac{\sqrt{\left|\bigoplus_{\ell\in\gamma} \mathbf{\Psi}_\ell\right|}}{\left(2\pi\right)^{\frac{d_\forall^\gamma}{2}}}e^{\frac{-1}{2} \mathrm{vec}\left[\mathcal{D}^\gamma\right]^T \left(\bigoplus_\ell \mathbf{\Psi}_\ell\right)\mathrm{vec}\left[\mathcal{D}^\gamma\right]}
\end{align}

\begin{lemma}[$\oplus$-$\mathrm{vec}$ lemma]
$\mathrm{vec}\left[\mathcal{D}^\gamma\right]^T \left(\bigoplus_\ell \mathbf{\Psi}_\ell\right)\mathrm{vec}\left[\mathcal{D}^\gamma\right] =  \sum_\ell\mathrm{tr}\left[\mathbf{S}^\gamma_\ell\mathbf{\Psi}_\ell\right]$
\label{lem:ks-vec}
\end{lemma}
\begin{proof}

This proof relies on the following two properties of $\mathrm{vec}$: $\left(\mathbf{A}\otimes \mathbf{B}\right)\mathrm{vec}\left[\mathbf{C}\right] = \mathrm{vec}\left[\mathbf{B}\mathbf{C}^T\mathbf{A}^T\right]$ and $\mathrm{tr}\left[\mathbf{A}^T\mathbf{B}\right] = \mathrm{vec}\left[\mathbf{A}\right]^T\mathrm{vec}\left[\mathbf{B}\right]$.  The $\mathbf{C}$ term picks up a transpose due to our use of the rows-first vectorization; when using columns-first notation the right hand side becomes $\mathrm{vec}\left[\mathbf{B}\mathbf{C}\mathbf{A}^T\right]$.

\begin{align}
\mathrm{vec}\left[\mathcal{D}^\gamma\right] \left(\bigoplus_\ell \mathbf{\Psi}_\ell\right)\mathrm{vec}\left[\mathcal{D}^\gamma\right] &=  \sum_\ell \mathrm{vec}\left[\mathcal{D}^\gamma\right]^T \left(\mathbf{I}_{d_{<\ell}}\otimes\mathbf{\Psi}_\ell\otimes \mathbf{I}_{d_{>\ell}}\right)\mathrm{vec}\left[\mathcal{D}^\gamma\right] \tag{Definition of $\bigoplus$} \\
&=  \sum_\ell \mathrm{vec}\left[\mathrm{mat}_1\left[\mathcal{D}^\gamma\right]\right]^T \left(\mathbf{I}_{d_{<\ell}}\otimes\mathbf{\Psi}_\ell\otimes \mathbf{I}_{d_{>\ell}}\right)\mathrm{vec}\left[\mathrm{mat}_1\left[\mathcal{D}^\gamma\right]\right] \\
&=  \sum_\ell \mathrm{vec}\left[\mathrm{mat}_\ell\left[\mathcal{D}^\gamma\right]\right]^T \mathbf{P}_{\ell\rightarrow 1}^T \left(\mathbf{I}_{d_{<\ell}}\otimes\mathbf{\Psi}_\ell\otimes \mathbf{I}_{d_{>\ell}}\right)\mathbf{P}_{\ell\rightarrow 1}\mathrm{vec}\left[\mathrm{mat}_\ell\left[\mathcal{D}^\gamma\right]\right] \\
&= \sum_\ell \mathrm{vec}\left[\mathrm{mat}_\ell\left[\mathcal{D}^\gamma\right]\right]^T  \left(\mathbf{\Psi}_\ell\otimes \mathbf{I}_{d_{\backslash\ell}}\right)\mathrm{vec}\left[\mathrm{mat}_\ell\left[\mathcal{D}^\gamma\right]\right] \tag{Rearrangement Lemma} \\
&= \sum_\ell \mathrm{vec}\left[\mathrm{mat}_\ell\left[\mathcal{D}^\gamma\right]\right]^T\mathrm{vec}\left[\mathrm{mat}_\ell\left[\mathcal{D}^\gamma\right]\mathbf{\Psi}_\ell^T\right] \\
&= \sum_\ell \mathrm{tr}\left[\mathbf{S}^\gamma_\ell\mathbf{\Psi}_\ell\right]
\end{align}

\end{proof}

With this lemma, the probability density function in the single-tensor case can be expressed in the form:

\begin{align}
    p\left(\mathcal{D}^\gamma\right) &= \frac{\sqrt{\left|\bigoplus_{\ell\in\gamma} \mathbf{\Psi}_\ell\right|}}{\left(2\pi\right)^{\frac{d_\forall^\gamma}{2}}}e^{\frac{-1}{2} \sum_\ell \mathrm{tr}\left[\mathbf{S}^\gamma_\ell\mathbf{\Psi}_\ell\right]}
\end{align}

Leading to the probability density function for the multi-tensor case as:

\begin{align}
    p\left(\left\{\mathcal{D}^\gamma\right\}\right) &= \prod_{\gamma} \frac{\sqrt{\left|\bigoplus_{\ell\in\gamma} \mathbf{\Psi}_\ell\right|}}{\left(2\pi\right)^{\frac{d_\forall^\gamma}{2}}}e^{\frac{-1}{2} \sum_\ell \mathrm{tr}\left[\mathbf{S}^\gamma_\ell\mathbf{\Psi}_\ell\right]} \\
    &=  \frac{\prod_{\gamma}\sqrt{\left|\bigoplus_{\ell\in\gamma} \mathbf{\Psi}_\ell\right|}}{\left(2\pi\right)^{\frac{d_\forall}{2}}}e^{\frac{-1}{2} \sum_\gamma\sum_\ell \mathrm{tr}\left[\mathbf{S}^\gamma_\ell\mathbf{\Psi}_\ell\right]} \\
    &= \frac{\prod_{\gamma}\sqrt{\left|\bigoplus_{\ell\in\gamma} \mathbf{\Psi}_\ell\right|}}{\left(2\pi\right)^{\frac{d_\forall}{2}}}e^{\frac{-1}{2} \sum_\ell \mathrm{tr}\left[\mathbf{S}_\ell\mathbf{\Psi}_\ell\right]}
\end{align}

The negative log-likelihood is thus:

\begin{align}
    \mathrm{NLL}\left(\left\{\mathcal{D}^\gamma\right\}\right) &= \frac{d_\forall}{2}\log \left(2\pi\right) + \frac{1}{2}\sum_\ell \mathrm{tr}\left[\mathbf{S}_\ell\mathbf{\Psi}_\ell\right] - \frac{1}{2}\sum_\gamma \log \left|\bigoplus_{\ell\in\gamma} \mathbf{\Psi}_\ell \right| 
\end{align}

\subsection{Gradient}
\label{sec:gradient}

The derivation of the gradient of the negative log-likelihood is essentially the same as the derivation given by \textcite{kalaitzis_bigraphical_2013} for the original Bi-Graphical Lasso.  Our derivation is complicated by the fact that we are considering general tensors rather than matrices.  We'll let $\mathrm{sym}$ be the symmetricizing operator that must be applied as we are taking the derivative with respect to a symmetric matrix: $\mathrm{sym}\left[\mathbf{M}\right] = \mathbf{K} \circ \mathbf{M}$, where $\mathbf{K}$ is a matrix with 1s on the diagonal and 2s everywhere else.  We'll also define $\mathbf{J}^{ij}$ to be the matrix of zeros except for a 1 at position $(i, j)$.

\begin{align}
    \frac{d}{d\mathbf{\Psi}_\ell}\mathrm{NLL}\left(\left\{\mathcal{D}^\gamma\right\}\right) &= \frac{1}{2}\mathrm{sym}\left[\mathbf{S}_\ell\right] - \frac{1}{2}\sum_\gamma \frac{d}{d\mathbf{\Psi}_\ell}\log\left|\bigoplus_{\ell'\in\gamma}\mathbf{\Psi}_{\ell'}\right| \\
    &= \frac{1}{2}\mathrm{sym}\left[\mathbf{S}_\ell\right] - \frac{1}{2}\sum_\gamma \mathrm{tr}\left[\left(\bigoplus_{\ell'\in\gamma}\mathbf{\Psi}_{\ell'}\right)^{-1} \frac{d}{d\psi_\ell^{ij}} \bigoplus_{\ell'\in\gamma}\mathbf{\Psi}_{\ell'}\right]_{ij} \\
    &= \frac{1}{2}\mathrm{sym}\left[\mathbf{S}_\ell\right] - \frac{1}{2}\sum_\gamma \mathrm{tr}\left[\left(\bigoplus_{\ell'\in\gamma}\mathbf{\Psi}_{\ell'}\right)^{-1}  \left(\mathbf{I}_{d_{<\ell}} \otimes \frac{d}{d\psi_\ell^{ij}}\mathbf{\Psi}_{\ell} \otimes \mathbf{I}_{d_{>\ell}}\right)\right]_{ij} \\
    &= \frac{1}{2}\mathrm{sym}\left[\mathbf{S}_\ell\right] - \frac{1}{2}\sum_{\gamma|\ell\in\gamma} \mathrm{tr}\left[\left(\bigoplus_{\ell'\in\gamma}\mathbf{\Psi}_{\ell'}\right)^{-1}  \left(\mathbf{I}_{d_{<\ell}} \otimes \left(\mathbf{J}^{ij} + \mathbf{J}^{ji} - \delta_{ij}\mathbf{J}^{ij}\right)\otimes \mathbf{I}_{d_{>\ell}}\right)\right]_{ij} \\
    &= \frac{1}{2}\mathrm{sym}\left[\mathbf{S}_\ell\right] - \frac{1}{2}\sum_{\gamma|\ell\in\gamma} \left[(2-\delta_{ij})\mathrm{tr}\left[\left(\bigoplus_{\ell'\in\gamma}\mathbf{\Psi}_{\ell'}\right)^{-1}  \left(\mathbf{I}_{d_{<\ell}} \otimes \mathbf{J}^{ij}\otimes \mathbf{I}_{d_{>\ell}}\right)\right]\right]_{ij} \\
    &= \frac{1}{2}\mathrm{sym}\left[\mathbf{S}_\ell\right] - \frac{1}{2}\sum_{\gamma|\ell\in\gamma} \left(2\mathbf{J} - \mathbf{I}\right)\circ\mathrm{tr}\left[\left(\bigoplus_{\ell'\in\gamma}\mathbf{\Psi}_{\ell'}\right)^{-1}  \left(\mathbf{I}_{d_{<\ell}} \otimes \mathbf{J}^{ij}\otimes \mathbf{I}_{d_{>\ell}}\right)\right]_{ij} \\
    &= \frac{1}{2}\mathrm{sym}\left[\mathbf{S}_\ell\right] - \frac{1}{2}\sum_{\gamma|\ell\in\gamma} \mathrm{sym}\left[\mathrm{tr}^{d_{<\ell}^\gamma}_{d_{>\ell}^\gamma}\left[\left(\bigoplus_{\ell'\in\gamma}\mathbf{\Psi}_{\ell'}\right)^{-1}  \right]\right]
\end{align}

The MLE occurs when this gradient is zero, i.e. when the following equation is satisfied:

\begin{align}
    \mathbf{S}_\ell &= \sum_{\gamma|\ell\in\gamma} \mathrm{tr}^{d_{<\ell}^\gamma}_{d_{>\ell}^\gamma}\left[\left(\bigoplus_{\ell'\in\gamma}\mathbf{\Psi}_{\ell'}\right)^{-1}\right] \label{eq:grad}
\end{align}

In other words, our effective Gram matrices are the best Kronecker-sum decomposition of the covariance matrix of the maximum likelihood estimate.  Unfortunately, the Kronecker-sum decomposition does not interact well with matrix inverses, so this does not directly yield an analytic solution.  It does, however, yield a solution for the eigenvectors.

\subsection{Maximum Likelihood Estimate for the Eigenvectors}
\label{sec:mle-eigenvectors}

We first produce two lemmas to aid in the derivation.

\begin{lemma}[Cyclic property of the stridewise-blockwise trace]
    For any matrices $\mathbf{M}, \mathbf{A}_{a\times a}, \mathbf{B}_{b\times b}$, we have that $\mathbf{tr}^a_b\left[\left(\mathbf{A} \otimes \mathbf{I} \otimes \mathbf{B}\right) \mathbf{M} \right] = \mathbf{tr}^a_b\left[\mathbf{M} \left(\mathbf{A} \otimes \mathbf{I} \otimes \mathbf{B}\right)\right]$
    \label{lem:cyc}
\end{lemma}
\begin{proof}
    This follows directly from the cyclic property of the (normal) trace operator and the definition of the stridewise-blockwise trace.
\end{proof}

\begin{lemma}[Conjugacy extraction of the stridewise-blockwise trace]
    For any matrices $\mathbf{M}$ and $\mathbf{V}$, we have that $\mathrm{tr}^a_b\left[\left(\mathbf{I}_a \otimes \mathbf{V} \otimes \mathbf{I}_b\right)\mathbf{M}\left(\mathbf{I}_a \otimes \mathbf{V} \otimes \mathbf{I}_b \right)^T\right] = \mathbf{V}\mathrm{tr}^a_b\left[\mathbf{M}\right]\mathbf{V}^T$.
    \label{lem:conj}
\end{lemma}
\begin{proof}
    \begin{align}
        \mathrm{tr}^a_b\left[\left(\mathbf{I}_a \otimes \mathbf{V} \otimes \mathbf{I}_b\right)\mathbf{M}\left(\mathbf{I}_a \otimes \mathbf{V} \otimes \mathbf{I}_b \right)^T\right] &= \left[\mathrm{tr}\left[\left(\mathbf{I}_a \otimes \mathbf{V} \otimes \mathbf{I}_b\right)\mathbf{M}\left(\mathbf{I}_a \otimes \mathbf{V} \otimes \mathbf{I}_b \right)^T \left(\mathbf{I}_a \otimes \mathbf{J}^{ij} \otimes \mathbf{I}_b \right)\right]\right]_{ij} \tag{Definition of $\mathrm{tr}^a_b$}
    \end{align}

    Thanks to the Rearrangement Lemma, we can get this just in terms of the standard blockwise trace, for which there exists a convenient lemma from \textcite{dahl_network_2013} that does the heavy lifting for us.  Unfortunately, this requires inserting permutation matrices into every nook and cranny.
    
    \begin{align}
        &= \left[\mathrm{tr}\left[\mathbf{P}\left(\mathbf{I}_a \otimes \mathbf{V} \otimes \mathbf{I}_b\right)\mathbf{P}^T\mathbf{P}\mathbf{M}\mathbf{P}^T\mathbf{P}\left(\mathbf{I}_a \otimes \mathbf{V} \otimes \mathbf{I}_b \right)^T\mathbf{P}^T\mathbf{P}\left(\mathbf{I}_a \otimes \mathbf{J}^{ij} \otimes \mathbf{I}_b \right)\mathbf{P}^T\right]\right]_{ij} \\
        &= \left[\mathrm{tr}\left[ \left(\mathbf{V} \otimes \mathbf{I}_{ab}\right)^T \mathbf{P}\mathbf{M}\mathbf{P}^T \left(\mathbf{V} \otimes \mathbf{I}_{ab}\right)\left(\mathbf{J}^{ij}\otimes\mathbf{I}_{ab}\right)\right]\right]_{ij} \\
        &= \mathrm{tr}_{ab}\left[\left(\mathbf{V} \otimes \mathbf{I}_{ab}\right)\mathbf{P}\mathbf{M}\mathbf{P}^T \left(\mathbf{V} \otimes \mathbf{I}_{ab}\right)^T\right] \tag{Definition of $\mathrm{tr}_{ab}$} \\
        &= \mathbf{V}\mathrm{tr}_{ab}\left[\mathbf{P}\mathbf{M}\mathbf{P}^T\right]\mathbf{V}^T \tag{Lemma 2 of \textcite{dahl_network_2013}}
    \end{align}

    We then can see analogously that $\mathrm{tr}_{ab}\left[\mathbf{P}\mathbf{M}\mathbf{P}^T\right] = \mathrm{tr}_{b}^a\left[\mathbf{M}\right]$, completing the proof.
    
\end{proof}

\begin{theorem}
    Let $\mathbf{V}_\ell\mathbf{e}_\ell\mathbf{V}_\ell^T$ be the eigendecomposition of $\mathbf{S}_\ell$.  Then $\mathbf{V}_\ell$ are the eigenvectors of the maximum likelihood estimate of $\mathbf{\Psi}_\ell$.
\label{thm:evec}
\end{theorem}
\begin{proof}
    \begin{align}
        \mathbf{S}_\ell &= \sum_{\gamma|\ell\in\gamma} \mathrm{tr}^{d_{<\ell}^\gamma}_{d_{>\ell}^\gamma}\left[\left(\bigoplus_{\ell'\in\gamma}\mathbf{\Psi}_{\ell'}\right)^{-1}\right] \\
        &= \sum_{\gamma|\ell\in\gamma} \mathrm{tr}^{d_{<\ell}^\gamma}_{d_{>\ell}^\gamma}\left[\left(\bigoplus_{\ell'\in\gamma}\mathbf{V}_{\ell'}\mathbf{\Lambda}_{\ell'}\mathbf{V}^T_{\ell'}\right)^{-1}\right] \\
        &= \sum_{\gamma|\ell\in\gamma} \mathrm{tr}^{d_{<\ell}^\gamma}_{d_{>\ell}^\gamma}\left[\left(\bigotimes_{\ell'}\mathbf{V}_{\ell'}\right)\left(\bigoplus_{\ell'\in\gamma}\mathbf{\Lambda}_{\ell'}\right)^{-1}\left(\bigotimes_{ell'}\mathbf{V}_{\ell'}\right)^T\right] \\
        &= \sum_{\gamma|\ell\in\gamma} \mathrm{tr}^{d_{<\ell}^\gamma}_{d_{>\ell}^\gamma}\left[\left(\mathbf{I}_{d_{<\ell}}\otimes\mathbf{V}_{\ell}\otimes\mathbf{I}_{d_{>\ell}}\right)\left(\bigoplus_{\ell'\in\gamma}\mathbf{\Lambda}_{\ell'}\right)^{-1}\left(\mathbf{I}_{d_{<\ell}}\otimes\mathbf{V}_{\ell}\otimes\mathbf{I}_{d_{>\ell}}\right)^T\right] \tag{Cyclic Property} \\
        &= \sum_{\gamma|\ell\in\gamma} \mathbf{V}_\ell\mathrm{tr}^{d_{<\ell}^\gamma}_{d_{>\ell}^\gamma}\left[\left(\bigoplus_{\ell'\in\gamma}\mathbf{\Lambda}_{\ell'}\right)^{-1}\right]\mathbf{V}^T_\ell \tag{Conjugacy Extraction} \\
        &= \mathbf{V}_\ell\left(\sum_{\gamma|\ell\in\gamma}\mathrm{tr}^{d_{<\ell}^\gamma}_{d_{>\ell}^\gamma}\left[\left(\bigoplus_{\ell'\in\gamma}\mathbf{\Lambda}_{\ell'}\right)^{-1}\right]\right)\mathbf{V}_\ell^T
    \end{align}

    We conclude the proof by observing that the central matrix is diagonal, and thus the right hand side constitutes an eigendecomposition of $\mathbf{S}_\ell$.  Thus $\mathbf{S}_\ell$ and $\mathbf{\Psi}_\ell$ share eigenvectors.
\end{proof}

\subsection{Maximum Likelihood Estimate for the Eigenvalues}
\label{sec:mle-eigenvalues}

In the previous section, we derived the eigenvectors of the maximum likelihood estimate.  While interesting (they correspond to the principal components of our data), we need the eigenvalues to reconstruct $\mathbf{\Psi}_\ell$.  In the main paper, we derived the gradient with respect to the Negative Log-Likelihood.  Then, for a learning rate $\mu_t$, gradient descent can be performed with the update equation $\mathbf{\Lambda}^{t+1}_\ell = \mathbf{\Lambda}^t_\ell - \mu_t\left[\mathbf{e}_\ell - \sum_{\gamma|\ell\in\gamma}\mathrm{tr}^{d_{<\ell}^\gamma}_{d_{>\ell}^\gamma}\left[\left(\bigoplus_{\ell'\in\gamma}\mathbf{\Lambda}_{\ell'}\right)^{-1}\right]\right]$.  A reasonable restriction is to make $\mathbf{\Psi}_\ell$ is positive definite, in which case $\mu_t$ must be chosen to prevent $\mathbf{\Lambda}_\ell^t$ from becoming negative.  Iterating over the eigenvalues reduces our optimization task from one with $\sum_\ell d_\ell^2$ parameters to one with $\sum_\ell d_\ell$ parameters.

\subsection{Incorporation of Priors}
\label{sec:incorp-priors}

\begin{theorem}[GmGM Eigenvector Estimator with Priors]
    \label{thm:evec-prior}
    Suppose we have the same setup as in Theorem \ref{thm:evec}, and that we have a prior of the form:
    
    \begin{align}
        \prod_\ell
        g_\ell(\mathbf{\Theta}_\ell)
        e^{\mathrm{tr}\left[\eta_\ell(\mathbf{\Theta}_\ell)^T
        \mathbf{\Psi}_\ell\right]
        + h_\ell(\mathbf{\Psi_\ell})}
    \end{align}

    In other words, our prior is an exponential distribution for $\mathbf{\Psi}_\ell$, in which $\mathbf{\Psi}_\ell$ is the sufficient statistic.

    Then, if $h_\ell$ depends only on the eigenvalues (i.e. it is `unitarily invariant'), the eigenvectors of $\frac{1}{2}\mathbf{S}_\ell - \eta_\ell(\mathbf{\Theta}_\ell)$ are the eigenvectors $\mathbf{V}_\ell$ of the MAP estimate for $\mathbf{\Psi}_\ell$.
\end{theorem}


\begin{proof}
    Observe that our log-prior is:

    \begin{align}
        \sum_\ell \log g_\ell(\mathbf{\Theta}_\ell) + \mathrm{tr}\left[\eta_\ell(\mathbf{\Theta}_\ell)^T\mathbf{\Psi}_\ell\right] + h_\ell(\mathbf{\Psi_\ell})
    \end{align}

    With gradient:

    \begin{align}
         \eta_\ell(\mathbf{\Theta}_\ell) + \frac{\partial}{\partial \mathbf{\Psi}_\ell} h_\ell(\mathbf{\Psi_\ell}) \label{eq:prior-grad}
    \end{align}

    Theorem 1.1 from \textcite{lewis_derivatives_1996} states that the eigenvectors of the derivative of any spectral function are the same as the eigenvectors of its input.  In our case, this means that the eigenvectors of $\frac{\partial}{\partial \mathbf{\Psi}_\ell} h_\ell(\mathbf{\Psi}_\ell)$ are $\mathbf{V}_\ell$, the eigenvectors of $\mathbf{\Psi}_\ell$.

    Recall the gradient of the log-likelihood of the Kronecker-sum-structured normal distribution was:

    \begin{align}
         \frac{1}{2}\sum_{\gamma|\ell\in\gamma} \mathrm{tr}^{d_{<\ell}^\gamma}_{d_{>\ell}^\gamma}\left[\left(\bigoplus_{\ell'\in\gamma}\mathbf{\Psi}_{\ell'}\right)^{-1}  \right]
         - \frac{1}{2}\mathbf{S}_\ell
    \end{align}

    The gradient of the log-likelihood after including the prior is the sum of these two gradients:

    \begin{align}
         \left(\frac{1}{2}\sum_{\gamma|\ell\in\gamma} \mathrm{tr}^{d_{<\ell}^\gamma}_{d_{>\ell}^\gamma}\left[\left(\bigoplus_{\ell'\in\gamma}\mathbf{\Psi}_{\ell'}\right)^{-1}  \right] + \frac{\partial}{\partial \mathbf{\Psi}_\ell} h_\ell(\mathbf{\Psi}_\ell) \right)
         - \frac{1}{2}\mathbf{S}_\ell + \eta_\ell(\mathbf{\Theta}_\ell)
    \end{align}

    And the optimum is achieved at:

    \begin{align}
         \left(\frac{1}{2}\sum_{\gamma|\ell\in\gamma} \mathrm{tr}^{d_{<\ell}^\gamma}_{d_{>\ell}^\gamma}\left[\left(\bigoplus_{\ell'\in\gamma}\mathbf{\Psi}_{\ell'}\right)^{-1}  \right] + \frac{\partial}{\partial \mathbf{\Psi}_\ell} h_\ell(\mathbf{\Psi}_\ell) \right)
         &= \frac{1}{2}\mathbf{S}_\ell - \eta_\ell(\mathbf{\Theta}_\ell)
    \end{align}

    The right hand side is known, and hence has known eigenvectors.  The left hand side is unknown, but has eigenvectors $\mathbf{V}_\ell$.

\end{proof}

The change to the iterative portion of the GmGM algorithm is minimal; by virtue of being unitarily invariant and convex, \textcite{lewis_derivatives_1996} guarantees that $\frac{\partial}{\partial \mathbf{\Psi}_\ell} h_\ell(\mathbf{\Psi}_\ell)$ is some function of just the eigenvalues $\mathbf{\Lambda}_\ell$ of $\mathbf{\Psi}_\ell$.  This function is added to our gradient each step (Line \ref{eq:prior-grad}).  This fact is significant, as it means Lemma \ref{lem:eig-decomp-of-distribution} still applies - he only need to consider $\frac{\partial}{\partial \mathbf{\Lambda}_\ell} h_\ell(\mathbf{\Lambda}_\ell)$. $\frac{1}{2}\mathbf{S}_\ell - \eta_\ell(\mathbf{\Theta}_\ell)$ can be thought of as our `priorized' effective Gram matrix, its eigenvalues fill the role exactly of $\mathbf{e}_\ell$ in the original algorithm.  Thus, the only change needed to incorporate priors in the algorithm, other than use of the priorized Gram matrices, is inclusion of the $\frac{\partial}{\partial \mathbf{\Lambda}_\ell} h_\ell(\mathbf{\Lambda}_\ell)$ term.

\begin{corollary}
    The Wishart and matrix gamma distributions all satisfy Theorem \ref{thm:evec-prior}.
\end{corollary}

\subsection{Covariance Thresholding}
\label{sec:covariance-thresholding}

\begin{figure}
    \centering
    \includegraphics{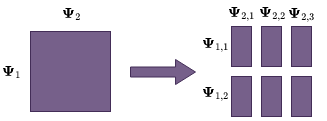}
    \caption{Graphical representation of the benefits of Theorem \ref{thm:cov-thresh}.  Provided the thresholded covariance matrix can be partitioned, so can the precision matrix.  This allows us to split the dataset into distinct parts, and estimate the covariance matrices separately.  Note that, in the diagram, we have split the dataset into 6 chunks but there are only 5 precision matrices to estimate - this is because there are shared axes, necessitating the use of our shared axis framework to take advantage of this partitioning.}
    \label{fig:partition-theorem}
\end{figure}

In this section, we restate and prove Theorem 4 from the main paper in more general terms; see Figure \ref{fig:partition-theorem}.  When we say that two matrices $\mathbf{A}$ and $\mathbf{B}$ have the same block diagonal structure, we mean that both have the following form:

\begin{align}
    \mathbf{B} &= \begin{bmatrix}
        \mathbf{B}_1 & 0 & \ldots \\
        0 & \mathbf{B}_2 &  \\
        \vdots &  & \ddots
    \end{bmatrix} \\
    \mathbf{A} &= \begin{bmatrix}
        \mathbf{A}_1 & 0 & \ldots \\
        0 & \mathbf{A}_2 &  \\
        \vdots &  & \ddots
    \end{bmatrix}
\end{align}

Where there is no restrictions placed on the submatrices except that they are square and that the shape of $\mathbf{A}_i$ is the same as the shape of $\mathbf{B}_i$ for all $i$.  The structure of our proof of Lemma 6 follows the same structure as that of Theorem 1 of \textcite{mazumder_exact_2012}; we give it in a more general form in such that both their result and our result arise as corollaries in conjunction with Lemma \ref{lem:sb-trace-block-diag-preserve}.

\begin{lemma}
    \label{lem:cov-thresh-generalized}
    Let $\mathbf{F}\left[\mathbf{\Psi}_\ell; \left\{\mathbf{\Psi}_{\ell'}\right\}_{\ell'\neq\ell}\right]$ be a function that preserves the block diagonal structure of $\mathbf{\Psi}_\ell$, which as a shorthand we will denote $\mathbf{F}_\ell$.  Suppose we had an objective whose gradient is $\mathbf{S}_\ell - \mathbf{F}_\ell$.  If we incorporate an L1 penalty on $\mathbf{\Psi}_\ell$ with strength $\rho$ into the objective, then the subgradient of the function contains zero for an estimate $\hat{\mathbf{\Psi}}_\ell$ with the following property:

    If we replace every element of $\mathbf{S}_\ell$ whose magnitude is less than $\rho$ with 0, we can partition the encoded graph into disconnected components.  This partition is the same as the partition into disconnected components for $\hat{\mathbf{\Psi}}_\ell$.
\end{lemma}
\begin{proof}
    The incorporation of an L1 penalty leads to a subgradient of:

    \begin{align}
        \mathbf{S}_\ell - \mathbf{F}_\ell + \rho \mathrm{sign}\mathbf{\Psi}_\ell
    \end{align}

    Where $\mathrm{sign}$ is the element-wise sign function:

    \begin{align}
        \mathrm{sign} &\triangleq \left[\mathbf{M}\right]_{ij} \left\{\begin{matrix}
\in \left[-\rho, \rho \right ] & \text{if }\mathbf{M}_{ij}=0\\ 
= -\rho & \text{if } \mathbf{M}_{ij} > 0\\ 
= \rho & \text{if } \mathbf{M}_{ij} < 0
\end{matrix}\right.
    \end{align}

    We have critical points when the subgradient contains zero; this leads to the KKT conditions, which must be satisfied for all $ij$.

    \begin{align}
        \left|\mathbf{S}_\ell^{ij} - \mathbf{F}_\ell^{ij}\right| &\leq \rho &\text{ if } \mathbf{\Psi}_\ell^{ij} = 0 \\
        \mathbf{F}_\ell^{ij} &= \mathbf{S}_\ell^{ij} + \rho &\text{ if } \mathbf{\Psi}_\ell^{ij} > 0 \\
        \mathbf{F}_\ell^{ij} &= \mathbf{S}_\ell^{ij} - \rho &\text{ if } \mathbf{\Psi}_\ell^{ij} < 0
    \end{align}

    Define $\mathrm{thresh}_\rho\left[\mathbf{S}_\ell\right]$ to be a matrix with 0s everywhere that $\left|\mathbf{S}_\ell^{ij}\right| < \rho$ and otherwise is $\mathbf{S}_\ell^{ij}$.  Without loss of generality, we can assume our data has been ordered such that $\mathrm{thresh}_\rho\left[\mathbf{S}_\ell\right]$ is block diagonal.  We will construct a $\hat{\mathbf{\Psi}}_\ell$ that satisfies the KKT conditions while also having a block-diagonal structure.  The construction of such a $\hat{\mathbf{\Psi}}_\ell$ would show that the graph partition in our L1-penalized solution is at least as fine as the graph partition arising from thresholding $\mathbf{S}_\ell$.

    Let us construct $\hat{\mathbf{\Psi}}_\ell$ to be block-diagonal with the same block structure as $\mathrm{thresh}_\rho\left[\mathbf{S}_\ell\right]$.  Denote the blocks to be indexed by $a$ ($\hat{\mathbf{\Psi}}_{\ell,a}$ for $\mathbf{\Psi}$, $\mathbf{S}_{\ell,a}$ for $\mathbf{S}$), and define them to be solutions to the following problem:

    \begin{align}
        0 &\in \mathbf{S}_{\ell,a} - \mathbf{F}\left[\mathbf{\Psi}_{\ell, a}; \left\{\mathbf{\Psi}_{\ell'}\right\}_{\ell'\neq\ell}\right] + \rho \mathrm{sign}\mathbf{\Psi}_{\ell, a} \label{eq:subgradient-submatrix}
    \end{align}

    To see why $\hat{\mathbf{\Psi}}_\ell$ satisfies the KKT conditions, note that for any index $ij$ lying off the block diagonal, $\mathbf{S}_\ell^{ij} < \rho$ by construction. As $\mathbf{F}_{ij}$ preserves the sparsity structure of $\hat{\mathbf{\Psi}}_\ell$, and $\hat{\mathbf{\Psi}}_\ell^{ij} = 0$, $|\mathbf{S}_\ell^{ij}-\mathbf{F}_{ij}| = 0 < \rho$ is satisfied.  For the indices $ij$ lying on the block diagonal, note that the KKT conditions are satisfied via each block's constriction in Equation \ref{eq:subgradient-submatrix}.

    To conclude the proof, it is necessary to show that if our solution $\hat{\mathbf{\Psi}}_\ell$ has a block diagonal structure, then $\mathbf{S}_\ell$ has the same block diagonal structure.  In other words, the L1-penalized solution is not a more fine-grained partition than that which can be attained through covariance thresholding.  Suppose that we have some $\mathbf{\Psi}_\ell$ with block-diagonal structure, and thus that $\mathbf{F}_\ell$ has that same structure.  Then, $\mathbf{F}_\ell^{ij} = 0$ implies $\left|\mathbf{S}_\ell^{ij}\right| \leq \rho$, as by the KKT we know that $\left|\mathbf{S}_\ell^{ij} - \mathbf{F}_\ell^{ij}\right| \leq \rho$.  Thus, $\mathrm{thresh}_\rho\left[\mathbf{S}_\ell\right]$ has the same block diagonal structure as $\mathbf{\Psi}_\ell$.

    We have now shown that both graph partitions are equally fine-grained, and are thus equivalent.  This completes the proof.
\end{proof}

\begin{corollary}
    As the matrix inverse preserves block diagonal structure, Theorem 1 of \textcite{mazumder_exact_2012} follows immediately.
\end{corollary}

\begin{lemma}
    \label{lem:sb-trace-block-diag-preserve}
    $\sum_{\gamma|\ell\in\gamma} \mathrm{tr}^{d_{<\ell}^\gamma}_{d_{>\ell}^\gamma}\left[\left(\bigoplus_{\ell'\in\gamma}\mathbf{\Psi}_{\ell'}\right)^{-1}\right]$ preserves the block diagonal structure of $\mathbf{\Psi}_\ell$.
\end{lemma}
\begin{proof}
    For notational convenience, and without loss of generality, assume that $\ell$ is the first axis in each of the modalities it appears in.  Furthermore, observe that if for all $\gamma$ we have that  $\mathrm{tr}^{d_{<\ell}^\gamma}_{d_{>\ell}^\gamma}\left[\left(\bigoplus_{\ell'\in\gamma}\mathbf{\Psi}_{\ell'}\right)^{-1}\right]$ preserves the block diagonal structure, then the sum $\sum_{\gamma|\ell\in\gamma} \mathrm{tr}^{d_{<\ell}^\gamma}_{d_{>\ell}^\gamma}\left[\left(\bigoplus_{\ell'\in\gamma}\mathbf{\Psi}_{\ell'}\right)^{-1}\right]$ also preserves this structure.  Thus, we consider only one term in the sum and omit reference to $\gamma$.  Doing this, our problem reduces to showing that $\mathrm{tr}_{d_{\backslash\ell}}\left[\left(\mathbf{\Psi}_\ell \otimes \mathbf{I}_{d_{\backslash\ell}} + \mathbf{I}_{d_\ell} \otimes \bigoplus_{\ell'\neq\ell}\mathbf{\Psi}_{\ell'}\right)^{-1}\right]$ preserves the block diagonal structure.

    While $\mathbf{\Psi}_\ell \otimes \mathbf{I}_{d_{\backslash\ell}}$ does not preserve the block diagonal structure in the strict sense (as its blocks have different sizes than those of $\mathbf{\Psi}_\ell$), it is still block diagonal and there is an obvious mapping from the blocks of one to the blocks of the other:

    \begin{align}
        \begin{bmatrix}
            \mathbf{\Psi}_\ell^1 & 0 & \ldots \\
            0 & \mathbf{\Psi}_\ell^2 &  \\
            \vdots &  & \ddots
        \end{bmatrix} \otimes \mathbf{I}_{d_{\backslash\ell}} &= \begin{bmatrix}
            \mathbf{\Psi}_\ell^1 \otimes \mathbf{I}_{d_{\backslash\ell}} & 0 & \ldots \\
            0 & \mathbf{\Psi}_\ell^2 \otimes \mathbf{I}_{d_{\backslash\ell}} &  \\
            \vdots &  & \ddots
        \end{bmatrix}
    \end{align}

    Suppose block $\mathbf{\Psi}_i$ has size $s_i$, then the size of the blocks in $\mathbf{\Psi}_\ell \otimes \mathbf{I}_{d_{\backslash\ell}}$ is $s_i \times d_{\backslash\ell}$.  The second term in the sum is also block diagonal:

    \begin{align}
        \begin{bmatrix}
            1 & 0 & \ldots \\
            0 & 1 &  \\
            \vdots &  & \ddots
        \end{bmatrix} \otimes \bigoplus_{\ell'\neq\ell}\mathbf{\Psi}_{\ell'} &= \begin{bmatrix}
            \bigoplus_{\ell'\neq\ell}\mathbf{\Psi}_{\ell'} & 0 & \ldots \\
            0 & \bigoplus_{\ell'\neq\ell}\mathbf{\Psi}_{\ell'} &  \\
            \vdots &  & \ddots
        \end{bmatrix}
    \end{align}

    But these blocks are of size $1 \times d_{\backslash\ell}$; no larger than the blocks of $\mathbf{\Psi}_\ell \otimes \mathbf{I}_{d_{\backslash\ell}}$.  Thus, overall, $\mathbf{\Psi}_\ell \otimes \mathbf{I}_{d_{\backslash\ell}} + \mathbf{I}_{d_\ell} \otimes \bigoplus_{\ell'\neq\ell}\mathbf{\Psi}_{\ell'}$ has block diagonal structure with block sizes $s_i \times d_{\backslash\ell}$.  It follows that its inverse does as well.  The blockwise trace maps blocks of size $s_i \times d_{\backslash\ell}$ back to size $s_i$, completing the proof.
\end{proof}

\begin{theorem}
    \label{thm:cov-thresh}
    Set all elements of $\mathbf{S}_\ell$ whose absolute value is less than $\rho$ to 0.  This encodes a potentially disconnected graph.  Likewise, consider $\mathbf{\Psi}_\ell$ to be the estimated precision matrix for our model equipped with an L1 penalty of strength $\rho$.  This also encodes a potentially disconnected graph.  If we label the vertices by which disconnected component they are part of, then this labeling is the same in both procedures (the procedure with $\mathbf{S}_\ell$ and the procedure with $\mathbf{\Psi}_\ell$).
\end{theorem}
\begin{proof}
    This follows directly from the convexity of our objective as well as Lemmas \ref{lem:cov-thresh-generalized} and \ref{lem:sb-trace-block-diag-preserve}.
\end{proof}

\section{DEPENDENCIES}
\label{sec:dependencies}

\begin{algorithm}
\begin{algorithmic}[1]
\Require $\{\mathcal{D}^\gamma_i\}$, $\mathrm{tolerance}$
\Ensure $\{\mathbf{\Psi}_\ell\}$
\For $1 \leq \ell \leq K$
\State $\mathbf{S}_\ell \gets \sum_{\gamma | \ell \in \gamma}\frac{1}{n^\gamma}\sum_i^{n^\gamma}\mathrm{mat}_\ell\left[\mathcal{D}^\gamma_i\right]\mathrm{mat}_\ell\left[\mathcal{D}^\gamma_i\right]^T$
\State $\mathbf{V}_\ell \gets \mathrm{eigenvectors}[\mathbf{S}_\ell ]$
\State $\mathbf{e}_\ell \gets \mathrm{eigenvalues}[\mathbf{S}_\ell ]$
\EndFor
\State $\mathbf{\Lambda} \gets \begin{bmatrix}
    1 & ... & 1
\end{bmatrix}^T$
\State $\mu \gets 1$
\While \text{not converged}
    \For $1 \leq \ell \leq K$
        \State $\mathbf{G}^\gamma_\ell \gets \mathrm{proj}_{KS}\left[\left(\bigoplus_{\ell'\in\gamma}\mathbf{\Lambda}_\ell\right)^{-1}\right]$
        \State $\mathbf{\Lambda}'_\ell \gets \mathbf{\Lambda}_\ell - \mu\left[\mathbf{e}_\ell - \sum_{\gamma | \ell\in\gamma}\mathbf{G}^\gamma_\ell\right]$
    \EndFor
    \For $1 \leq \ell \leq K$
        \State $\mathbf{\Lambda}_\ell \gets \mathbf{\Lambda}'_\ell$
    \EndFor
    \For $\gamma \in \mathrm{modalities}$
        \If {$\sum_{\ell\in\gamma}\mathrm{min}\mathbf{\Lambda}_\ell < \mathrm{tolerance}$}
            \State \text{decrease $\mu$ to make sum far from zero}
        \EndIf
    \EndFor
\EndWhile
\For $1 \leq \ell \leq K$
    \State $\mathbf{\Psi}_\ell \gets \mathbf{V}_\ell \mathbf{\Lambda}_\ell \mathbf{V}_\ell^T$
\EndFor
\end{algorithmic}
\caption*{\textbf{The GmGM algorithm}}
\label{alg:GmGM}
\end{algorithm}

All tests and figures were generated on a Mac with an M1 chip and 8GB of RAM.  Along with our code, we provide an environment file ($\mathrm{environment.yml}$) that contains full details of all the dependencies used.  In our GitHub repository \href{https://github.com/AIStats-GmGM/GmGM/}{https://github.com/AIStats-GmGM/GmGM/}), we give precise and simple instructions on how to create a conda environment with the same setup as ours.  Most of the packages used were specific to the experiments we ran.  The dependencies necessary for our algorithm were Python 3.9 and NumPy 1.23.5.

\section{EXPERIMENTS}
\label{sec:experiments}

\subsection{Synthetic data}
\label{sec:synthetic-data}

For various precision-recall performances, see Figures \ref{fig:synthetic-matrix}, \ref{fig:synthetic-shared}, and \ref{fig:synthetic-tensor}.  For runtimes, see Figure \ref{fig:all-runtimes}.  We can see that our algorithm performs comparably to prior work; it does slightly worse without the restricted L1 penalty.  We focused on experimenting with the 1-sample case, as that is typically the case with real-world data (all of the five real-world experiments in Section \ref{sec:experiments} are an example of this).

For the runtimes, we cut off our tests once GmGM reached the one-minute runtime mark, or if the size of the synthetic data exceeded a gigabyte; see Table \ref{tab:runtime-bounds} for an overview, and Section \ref{sec:complexity} for the computational complexity.  Note that the runtime is almost linear in the total size of the dataset $d_\forall$ ($O(Kd_\forall^{1+\frac{1}{K}}$).  As $K$ increases, $\frac{1}{K}$ vanishes.  The cost of Gram matrix computation and eigendecomposition compared to the iterations also levels out; this is likely what explains the higher-order tensors all taking roughly the same amount of time to compute 1GB-sized datasets (~10 seconds).  These results clearly show that for high-axis-data ($K>3$), our algorithm can handle all smaller-than-RAM datasets in reasonable time on a personal computer.  For low-axis data ($K<4$), it may take more than a minute but, as we can see from our algorithm's performance on large real world datasets such as the 10x Genomics dataset, its runtime will still be measured in minutes, not hours, before reaching the limits of RAM.  When we consider the covariance-thresholding trick (Figure 7b in the main paper), we can push these limits even further.

\begin{table}[h]
    \centering
    \begin{tabular}{c|ccc}
         Amount of Axes ($K$) &  $d_\mathrm{GB}$ & $d_{\mathrm{min}}$ & Runtime at $d_\mathrm{GB}$\\
         \hline
         2 & 11,000 & 4,000 & ??? \\
         3 & 500 & 400 & ??? \\
         4 & 100 & ??? & 7.1 seconds \\
         5 & 40 & ??? & 9.5 seconds \\
         6 & 20 & ??? & 7.4 seconds
    \end{tabular}
    \caption{Comparing the point at which a $K$-axis tensor of double precision floats surpasses a gigabyte of data ($\sqrt[K]{\frac{1000000000}{64}} = d_\mathrm{GB}$) to the point at which GmGM's runtime surpasses a minute ($d_\mathrm{min}$).  Values are rounded.}
    \label{tab:runtime-bounds}
\end{table}

In all tests, our regularized algorithm took about the same time as the unregularized algorithm.  This is not unexpected, as our restricted regularization only affects the iterations, not the eigendecomposition.  In the real world, where the data may be `harder' than our synthetic data (i.e. take more iterations to converge), we would expect our regularized algorithm to be slower.

\begin{figure}[h]
    \centering
    \begin{subfigure}{0.45\textwidth}
        \includegraphics[width=\textwidth]{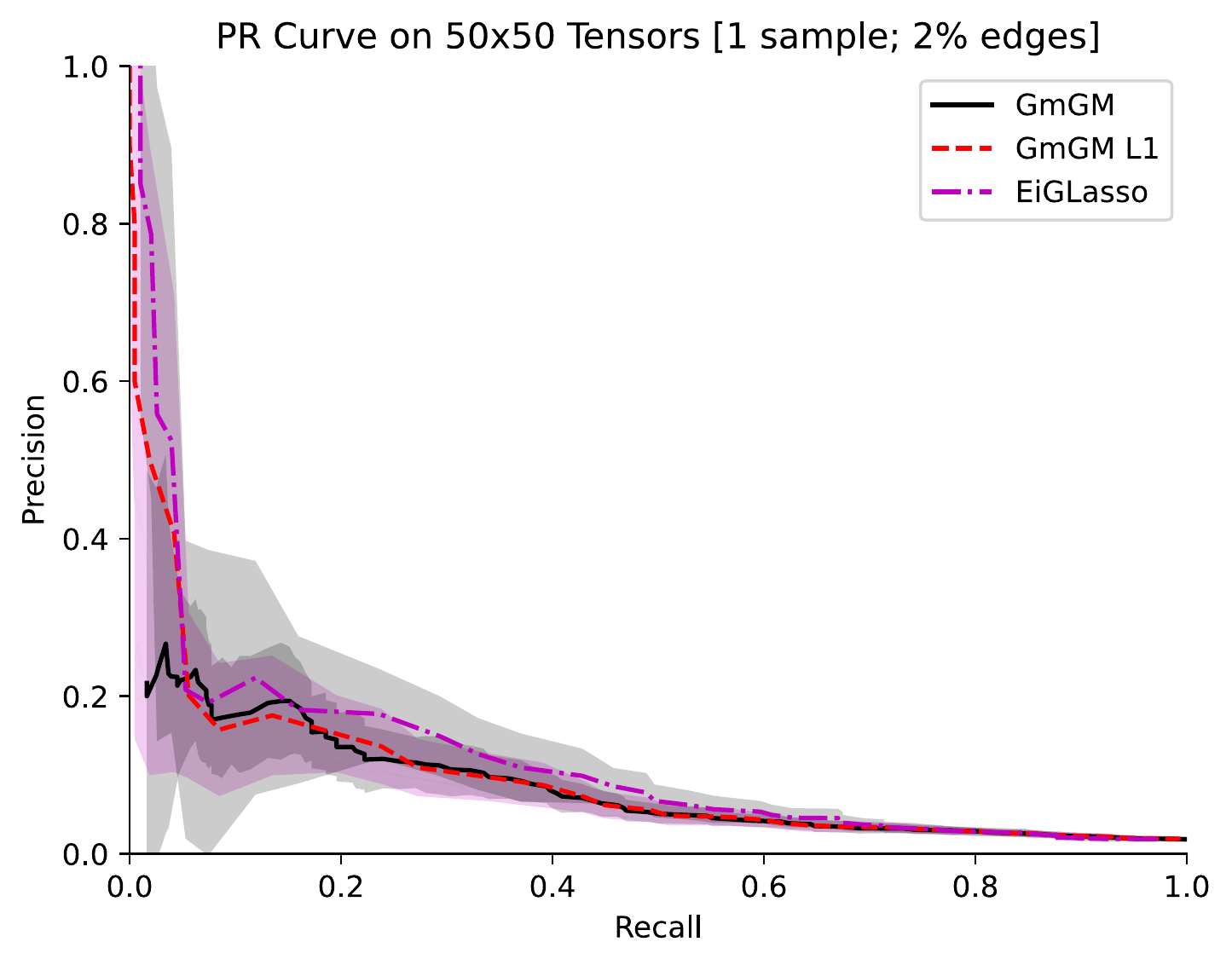}
        \caption{}
    \end{subfigure}
    \begin{subfigure}{0.45\textwidth}
        \includegraphics[width=\textwidth]{final-experiments/pr-curve-50x50-samples_1-ar_1.pdf}
        \caption{}
    \end{subfigure}
    \begin{subfigure}{0.45\textwidth}
        \includegraphics[width=\textwidth]{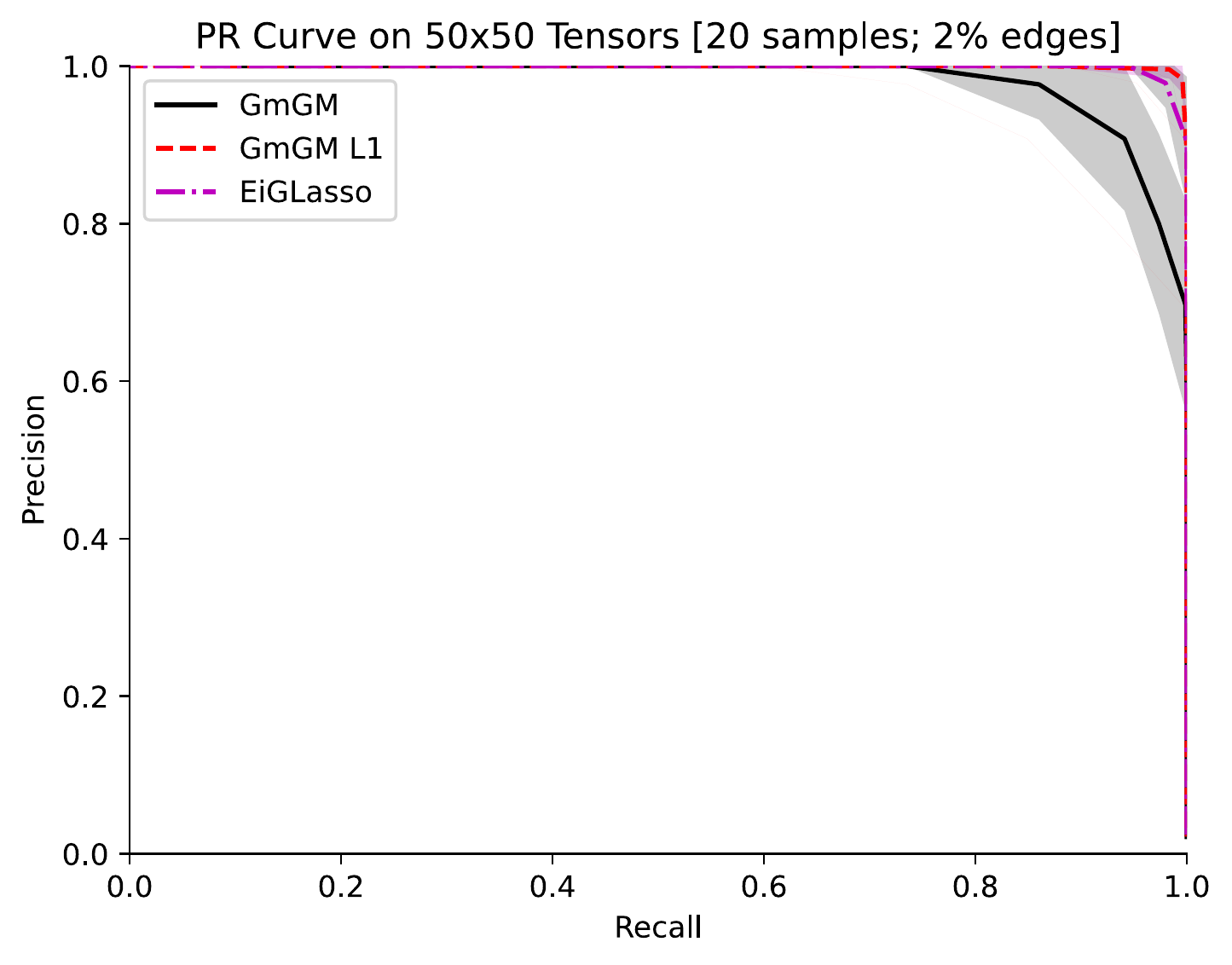}
        \caption{}
    \end{subfigure}
    \begin{subfigure}{0.45\textwidth}
        \includegraphics[width=\textwidth]{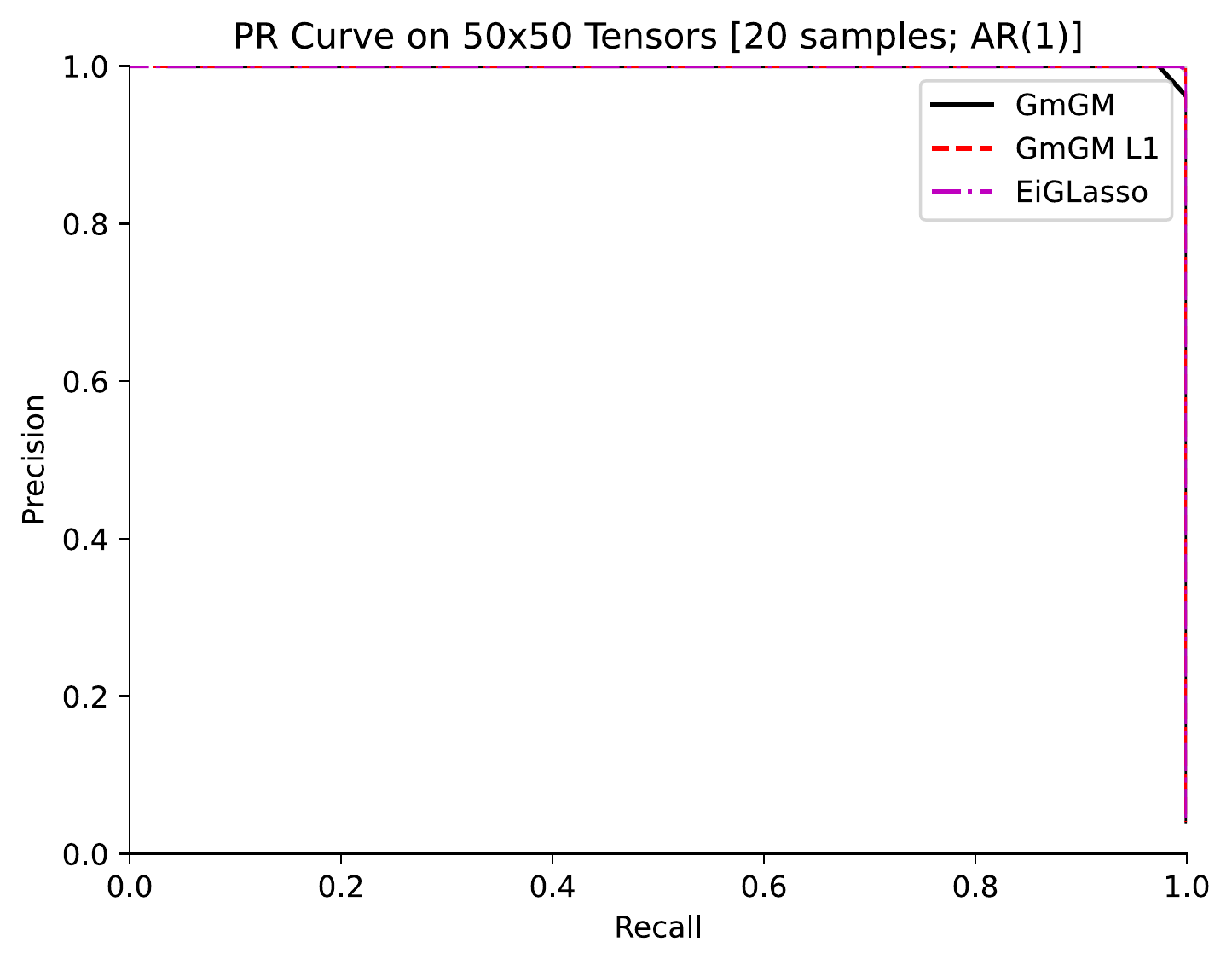}
        \caption{}
    \end{subfigure}
    \caption{Various precision-recall curves on synthetic matrix-variate data, for both an Erdos-Renyi with a 2\% per-edge probability and AR(1) distribution.}
    \label{fig:synthetic-matrix}
\end{figure}

\begin{figure}[h]
    \centering
    \includegraphics[width=0.5\textwidth]{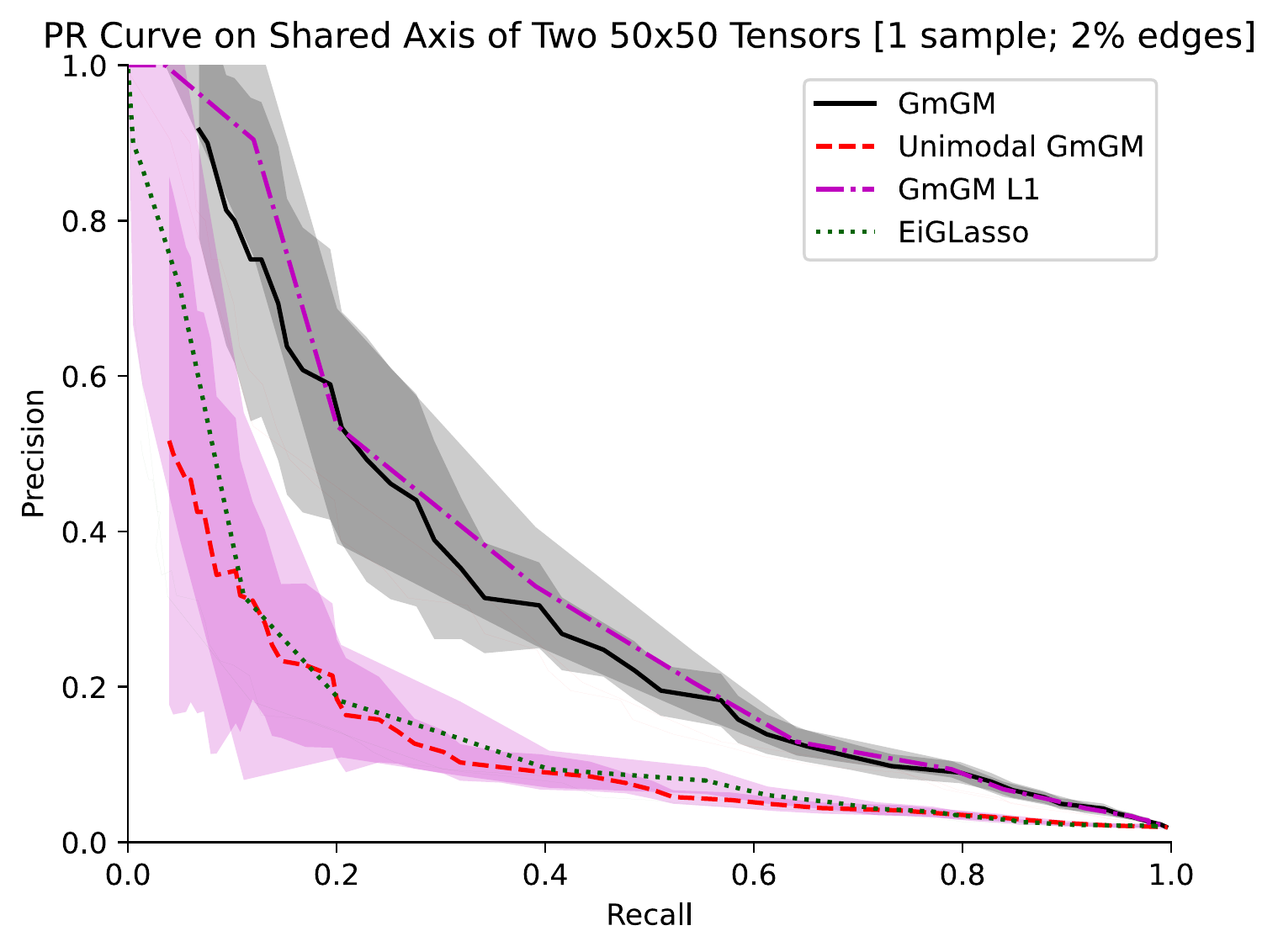}
    \caption{Performance on two 50x50 datasets with one shared axis, whose true graphs were drawn from an Erdos-Renyi distribution (2\% probability for each edge to exist, independently).  Unimodal GmGM and EiGLasso only consider one of the two datasets.}
    \label{fig:synthetic-shared}
\end{figure}

\begin{figure}[h]
    \centering
    \begin{subfigure}{0.45\textwidth}
        \includegraphics[width=\textwidth]{final-experiments/pr-curve-50x50x50-samples_1-percent_2.pdf}
        \caption{}
    \end{subfigure}
    \begin{subfigure}{0.45\textwidth}
        \includegraphics[width=\textwidth]{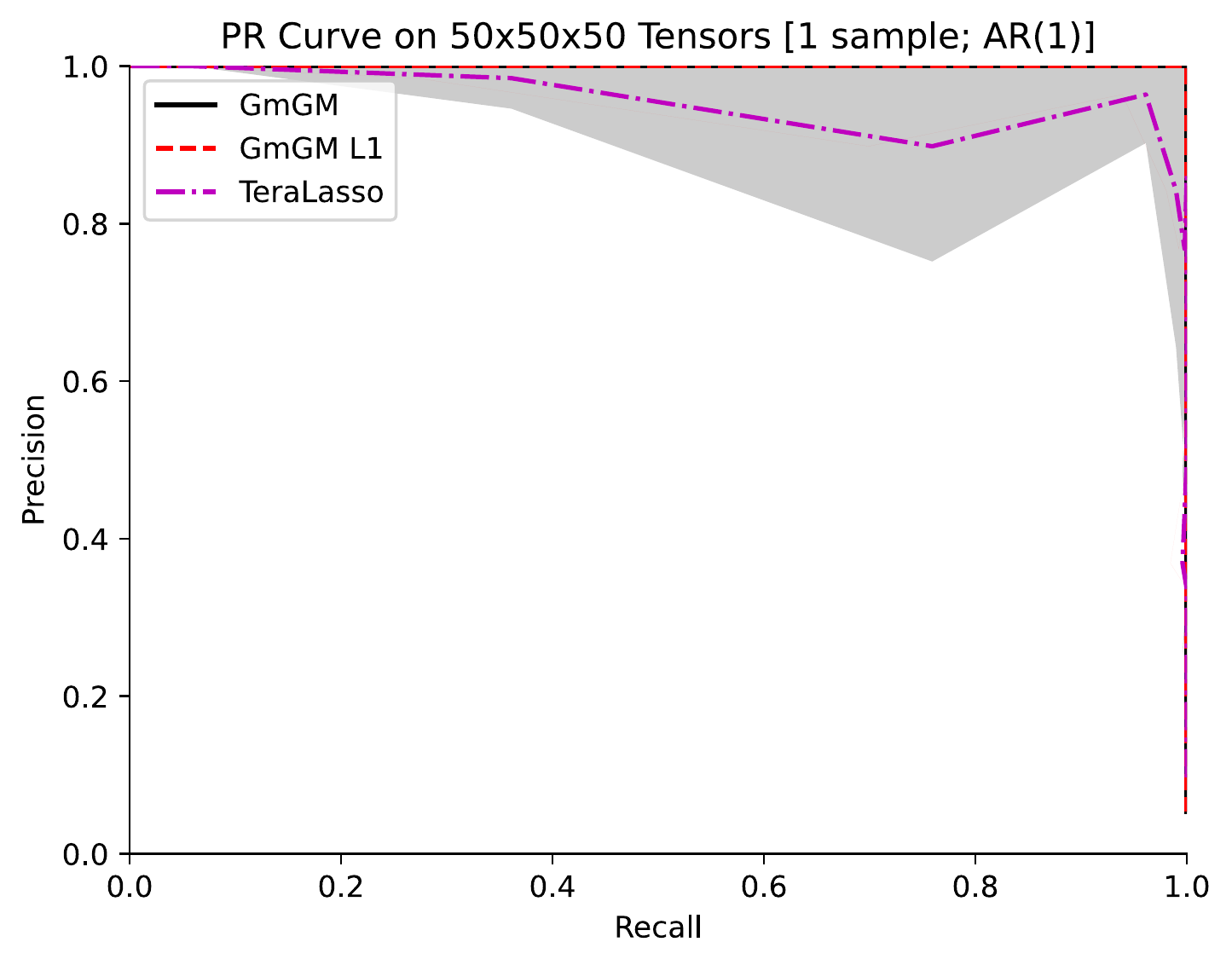}
        \caption{}
    \end{subfigure}
    \caption{Performance on tensor-variate (50x50x50) data whose true graphs were drawn from an Erdos-Renyi distribution (2\% edge probability) and an AR(1) distribution.}
    \label{fig:synthetic-tensor}
\end{figure}

\begin{figure}[h]
    \centering
    \begin{subfigure}{0.45\textwidth}
        \includegraphics[width=\textwidth]{final-experiments/matrix-runtime.pdf}
        \caption{}
    \end{subfigure}
    \begin{subfigure}{0.45\textwidth}
        \includegraphics[width=\textwidth]{final-experiments/tensor-runtime.pdf}
        \caption{}
    \end{subfigure}
    \begin{subfigure}{0.32\textwidth}
        \includegraphics[width=\textwidth]{final-experiments/4-axis-runtime.pdf}
        \caption{}
    \end{subfigure}
    \begin{subfigure}{0.32\textwidth}
        \includegraphics[width=\textwidth]{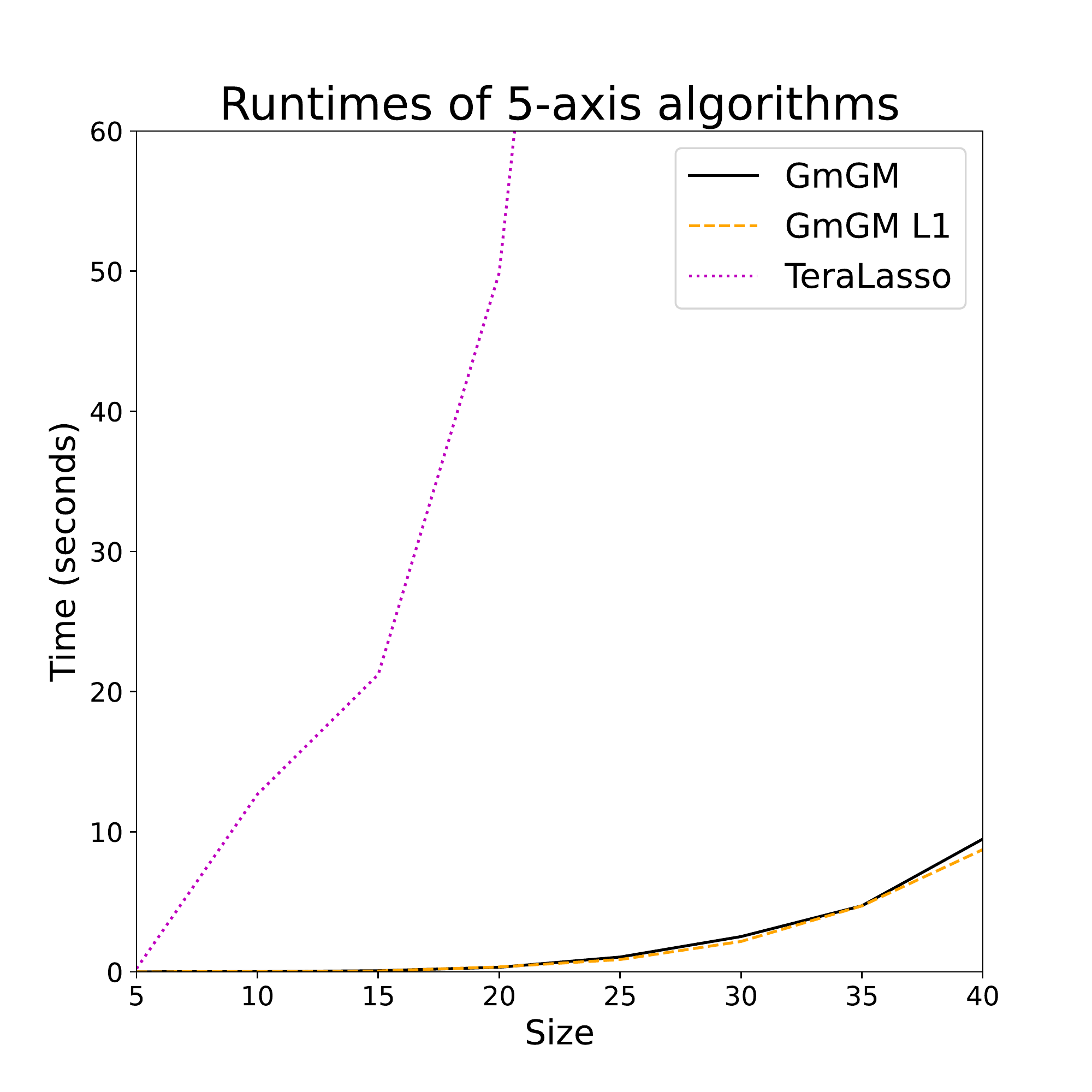}
        \caption{}
    \end{subfigure}
    \begin{subfigure}{0.32\textwidth}
        \includegraphics[width=\textwidth]{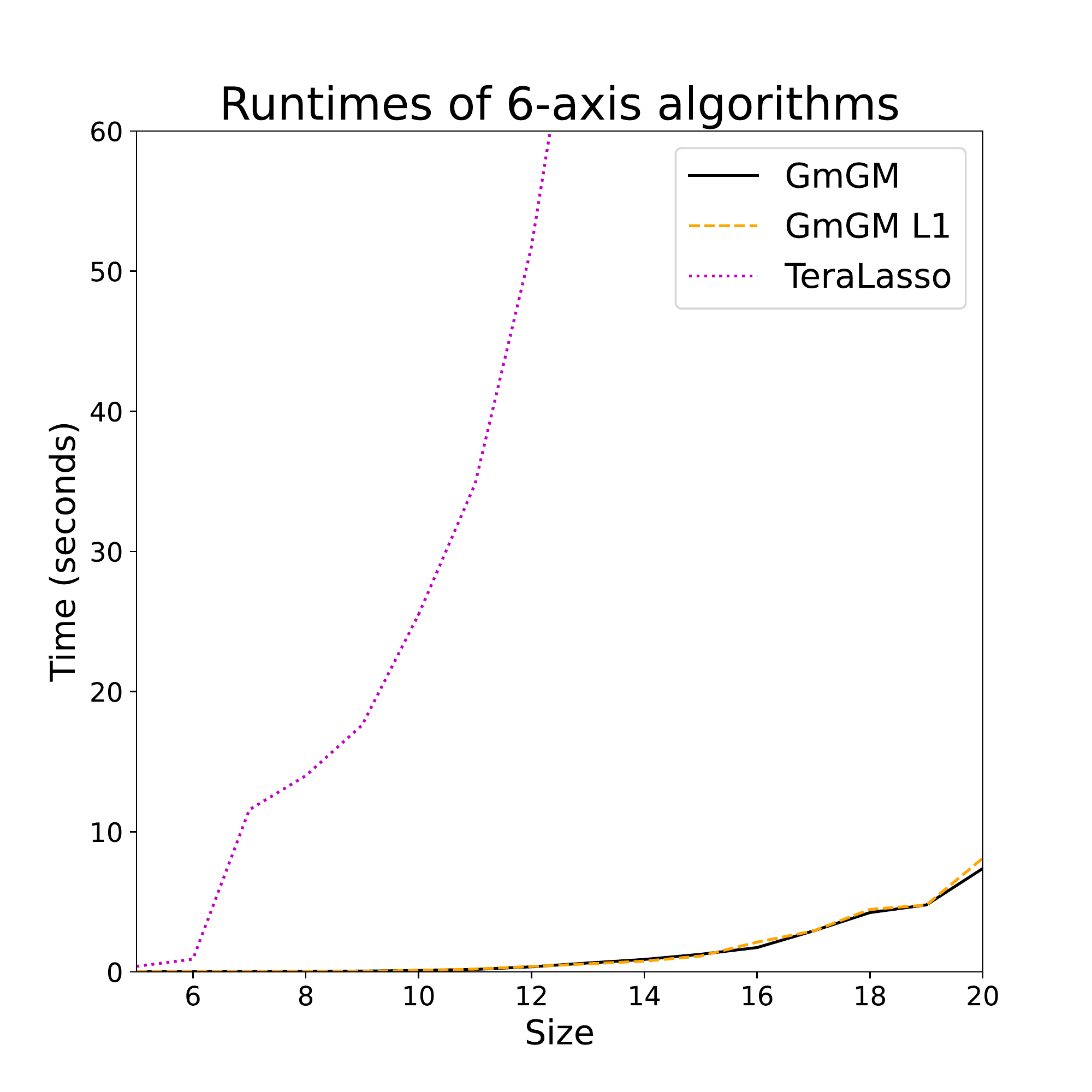}
        \caption{}
    \end{subfigure}
    \caption{Runtimes for 2-axis (a) 3-axis (b) 4-axis (c) 5-axis (d) and 6-axis (e) data, whose true graphs come from the Erdos-Renyi distribution with 2\% edge weights.}
    \label{fig:all-runtimes}
\end{figure}

\subsection{COIL video}
\label{sec:coil}

\subsubsection{The Dataset}

We downloaded the processed COIL-20 \parencite{nene_columbia_nodate} dataset, and tested our model on it.  It is available at \href{https://cave.cs.columbia.edu/repository/COIL-20}{https://cave.cs.columbia.edu/repository/COIL-20}; we used the `processed' data.  They do not formally license it, but state that it is available for academic purposes and provide the download links with no restrictions.

\subsubsection{Experiment Justification}

In the original paper introducing Kronecker-sum-structured normal distributions for graphical models, \textcite{kalaitzis_bigraphical_2013} showed that their BiGraphical Lasso (BiGLasso) could reconstruct the ordering of the frames of the video.  Thus, it seemed natural to choose this as an experiment, to show that our lack of L1 regularization was not a significant inhibitor for the model in the real world setting.

BiGLasso was a significant algorithm in that it was the first of this type of model and proved that it could work, but it was also incredibly slow (taking around 15 minutes on 100x100 matrix data) and limited to two axes.  Thus, they reduced the resolution of the image from 72 frames and 128x128 pixels to 36 frames and 9x9 pixels and flattened the rows and columns axis into one axis.  We wanted to show that our algorithm could handle the full resolution in negligible time while still reconstructing the video.

\subsubsection{Results}

We wanted to see if our model could understand the structure of a video, which we expected to consist of two linear graphs (for the rows and columns, i.e. each row is connected only to its neighbor rows) and a circular graph (for the frames, because the video is of a 360° rotation).  To generate these graphs, we ran our algorithm on the duck video from the dataset, and then greedily kept the largest edge from each vertex such that vertices in the final graph had at most two edges.  If we shuffled our data (shuffle rows, columns, and frames) and try to reconstruct it with these graphs, we get decent without the nonparanormal skeptic (Figure \ref{fig:duck-still}) and great results with it (Figure \ref{fig:duck-still-nonpara}).  While in both cases the algorithm has clearly reconstructed most of the duck, when using the nonparanormal skeptic our algorithm reconstructs it nearly perfectly.  The only human-noticeable issue is that it does not know which row/column the duck should start on - if we tessellate the image, the reconstruction would be indistinguishable from the original image.

\begin{figure}
    \begin{subfigure}[b]{0.5\textwidth}
        \centering
        \includegraphics[width=\textwidth]{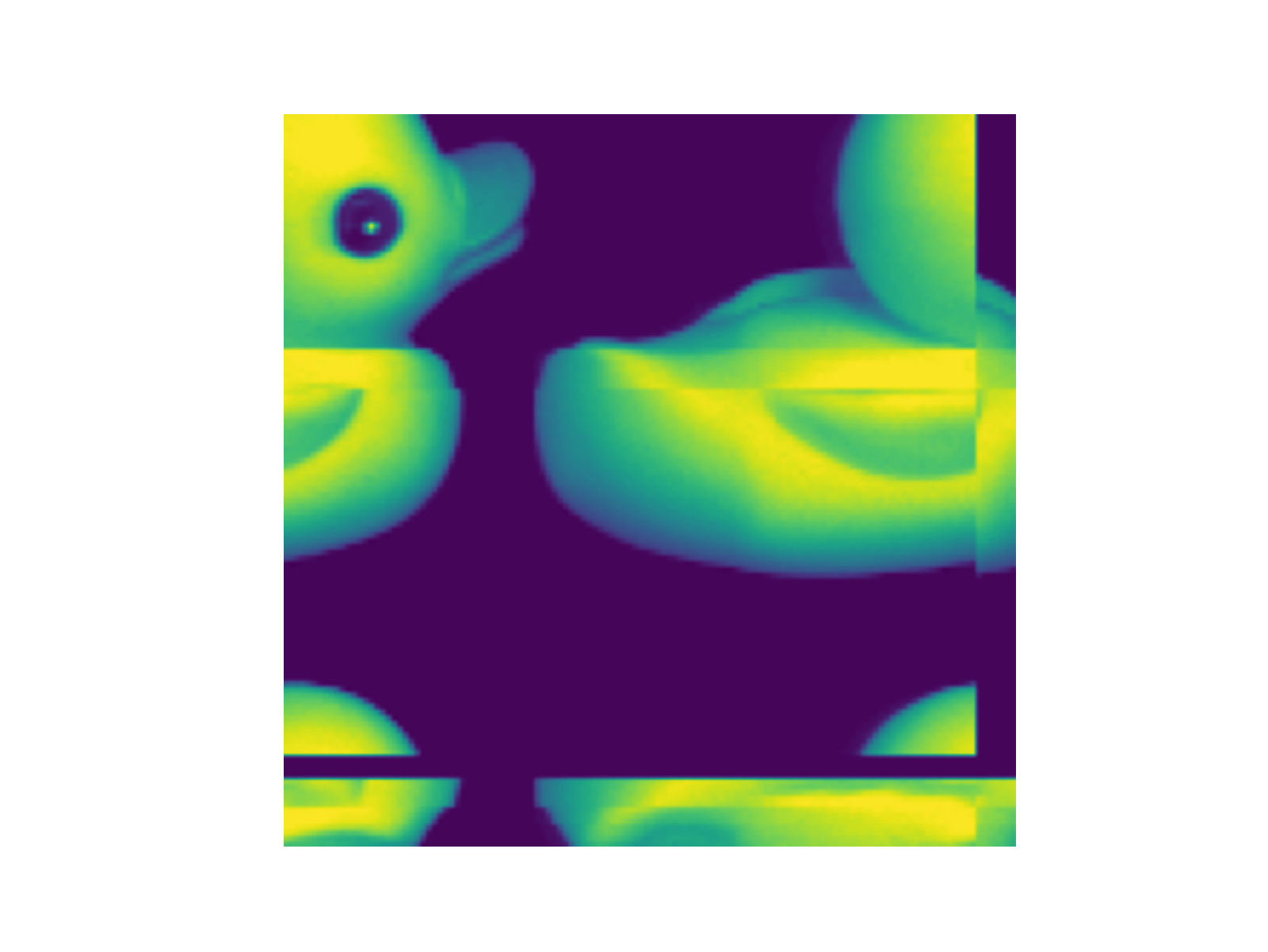}
        \caption{}
        \label{fig:duck-still}
    \end{subfigure}
    \begin{subfigure}[b]{0.5\textwidth}
        \centering
        \includegraphics[width=\textwidth]{final-experiments/gmgm-duck-nonpara.pdf}
        \caption{}
        \label{fig:duck-still-nonpara}
    \end{subfigure}
    \caption{A reconstruction of the COIL-20 duck video after shuffling the rows, columns, and frames, using GmGM.  (a) without the nonparanormal skeptic. (b) with the nonparanormal skeptic.}
\end{figure}

We can put a numeric value to the reconstruction, by measuring the percentage of the time that our reconstructed edges connect two adjacent rows/columns/frames.  Without the nonparanormal skeptic, we get an accuracy of 88\% for the rows, 94\% for the columns, and 93\% for the frames.  With it, the accuracy becomes 99\% for all three.  These values are without regularization; we did experiment with regularization, but it did not seem to help.  TeraLasso had the same results.

\begin{table}[t]
    \centering
    \begin{tabular}{c|c|c}
        Algorithm & Preprocessing & Runtime \\
        \hline
         GmGM & Center Data & 0.044 seconds \\
         GmGM & Nonparanormal Skeptic & 0.23 seconds \\
         TeraLasso & Center Data & 33 seconds \\
         TeraLasso & Nonparanormal Skeptic & 5.2 seconds \\
    \end{tabular}
    \caption{A comparison of the runtimes our algorithm and TeraLasso.  Runtimes were averaged over 10 runs, and include the preprocessing.}
    \label{tab:duck-runtimes}
\end{table}

On this dataset, our algorithm runs in 0.044 seconds whereas TeraLasso takes 33 seconds, showing that our algorithm is highly performant on real world data.  For data of this size, the nonparanormal skeptic is expensive compared to our algorithm - but our algorithm is still an order of magnitude faster than TeraLasso.  If we do not include the nonparanormal skeptic in our runtime calculation, the speed of GmGM is comparable regardless of the input. 
 However, convergence is much quicker for TeraLasso when it is fed the nonparanormal skeptic (Table \ref{tab:duck-runtimes}).

\subsection{EchoNet-Dynamic ECGs}
\label{sec:echonet}

\subsubsection{The Dataset}

We downloaded all of the EchoNet-Dynamic \parencite{ouyang_video-based_2020} data.  The dataset is available at \href{https://echonet.github.io/dynamic/index.html}{https://echonet.github.io/dynamic/index.html}; one has to sign a Research Use Agreement to access it.  It is available for personal, non-commercial research purposes only.  It prohibits distribution of portions of the data, which is why we do not show any frames from the echocardiograms in this paper.  The dataset has labels (for example, end systolic volume), but these labels are not relevant to the task of predicting the frames the heartbeats occur on.  For this, we picked the opening of the mitral valve as an arbitrary but easily visible occurrence, and manually recorded the frames of these openings for 5 videos in the dataset.  These videos were picked at random.

\subsubsection{Experiment Justification}

We chose the COIL duck video because of its use in prior work, but these methods were already known to perform quite well on it.  Thus, we wanted to try a similar but harder task.  Echocardiograms have a periodic structure (the heartbeat), and hence we would expect the learned graph of an echocardiogram to be similar to that of the duck video, with the addition of extra connections corresponding to the same phase of the heartbeat across different heartbeats.  We saw that our algorithm was able to find this more complicated structure, and prove it by using that structure to detect future heartbeats.

\subsubsection{Results}

We downloaded all of the EchoNet-Dynamic \parencite{ouyang_video-based_2020} data.  This dataset did not have heartbeats labeled, so we picked a few videos at random and labeled them ourselves as a proof of concept.  Specifically, we labeled every frame in which the mitral valve opened.  Our goal was to see if the graphs produced by our algorithm could predict this opening.  Table \ref{tab:echonet} contains the videos we picked, the labels we gave, and the labels we predicted.  When we used the nonparanormal skeptic as a preprocessing step, we got broadly similar empirical results (Table \ref{tab:echonet-nonpara}), although the precision matrices arguably more clearly show the periodic structure.

\begin{figure}
    \centering
    \includegraphics[width=0.5\textwidth]{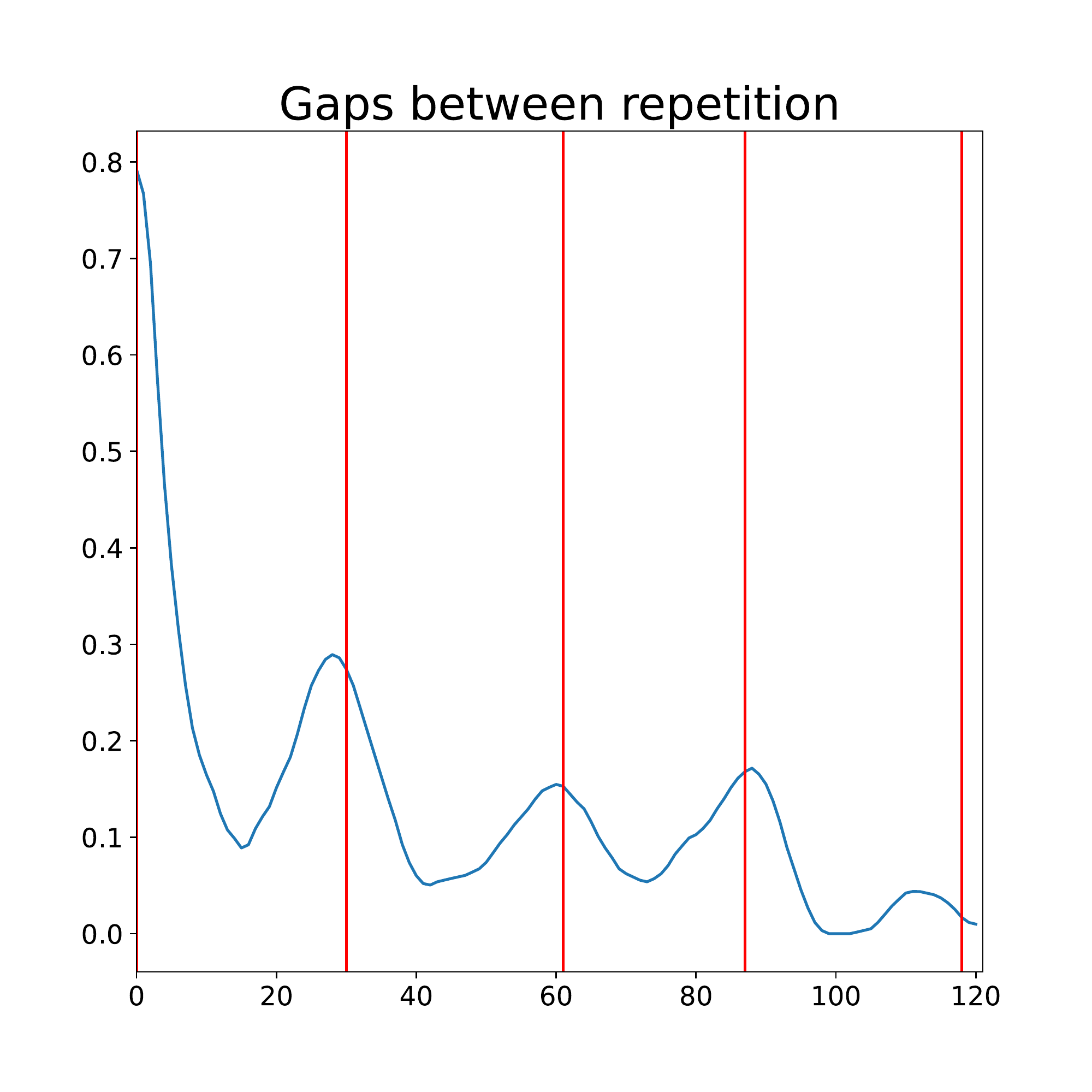}
    \caption{An example heartbeat offset detection, from EchoNet-Dynamic video 0XFE6E32991136338.  The blue curve represents our Gaussian-blurred diagonal mass (if x=10, it represents the blurred mass of the 10th diagonal to the right of the main diagonal).  The red lines represent the predicted peaks via a continuous wavelet transform peak detection algorithm.  These represent offsets from the first mitral valve opening.  For this video, the mitral valve opened on frame 17 and our first offset was on the 30th diagonal.  Hence, we would predict the second mitral valve opening to occur at frame 47 (which, in this case, was correct).}
    \label{fig:heartbeat_series}
\end{figure}

Mitral valve predictions were done by taking GmGM's output frames graph in precision matrix form, and measuring the mass along the diagonals.  We treated this as a time series (since each diagonal corresponds to an increasing time offset from all frames simultaneously).  We applied gaussian blur and then a continuous wavelet transform peak detection algorithm \parencite{du_improved_2006} to find which diagonals had the most mass (Figure \ref{fig:heartbeat_series}).  These represent the offsets corresponding with a heartbeat.  Given the first mitral valve opening and these offsets, we predict the remaining openings.  From Tables \ref{tab:echonet} and \ref{tab:echonet-nonpara}, we can see that our algorithm performs well.  Our algorithm is much quicker than TeraLasso; see Table \ref{tab:heart-runtimes}.

\begin{table}[t]
    \centering
    \begin{tabular}{c|c|c}
         Algorithm & Preprocessing & Runtime \\
         \hline
         GmGM & None & 0.067 seconds \\
         GmGM & Nonparanormal Skeptic & 0.33 seconds \\ 
         TeraLasso & None & 40 seconds \\
         TeraLasso & Nonparanormal Skeptic & 12 seconds \\ 
    \end{tabular}
    \caption{Runtimes of our algorithm and TeraLasso, with various preprocessing methods, on the EchoNet-Dynamic dataset.  Runtimes given are the average runtime over the five videos considered.}
    \label{tab:heart-runtimes}
\end{table}

\begin{table}[]
    \centering
    \begin{tabular}{cccc}
         Video ID & Label & Predicted & Precision Matrix\\ \hline
         0XFE6E32991136338 & [17, 47, 77, 106] & [17, 47, 78, 104] & \includegraphics[width=0.25\textwidth]{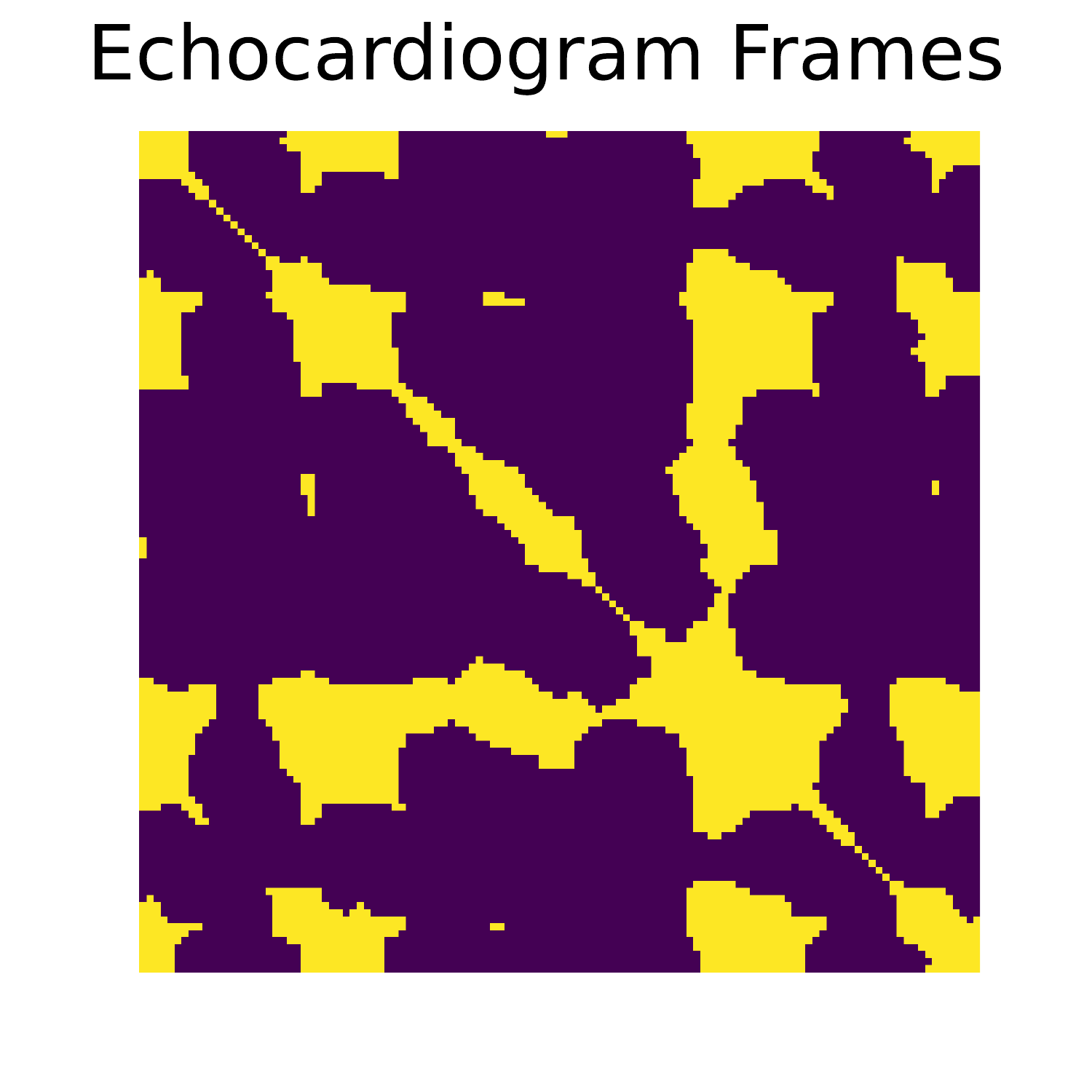}\\
         0XF072F7A9791B060 & [24, 56, 100] & [24, 59, 90] & \includegraphics[width=0.25\textwidth]{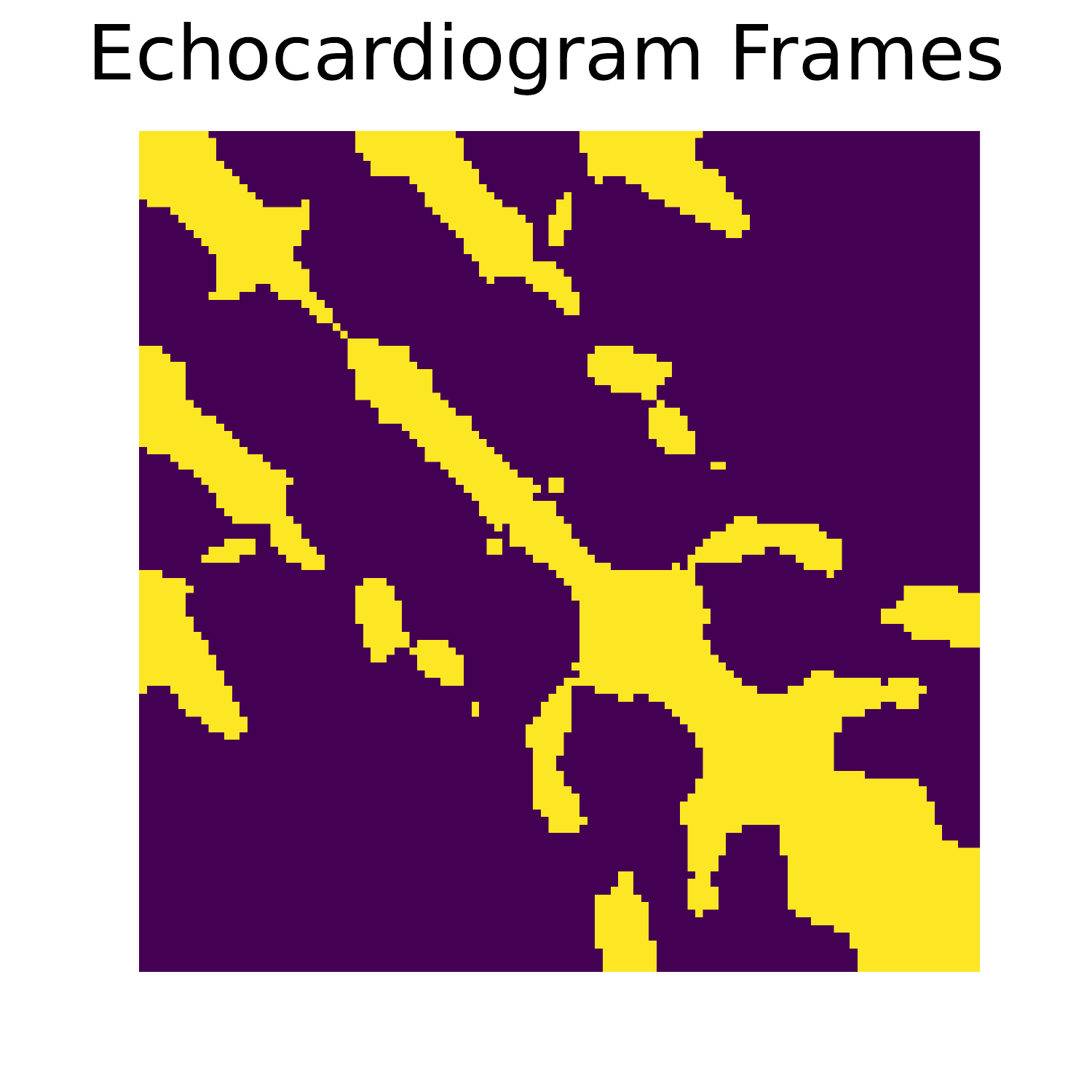}\\
         0XF70A3F712E03D87 & [22, 66, 110] & [22, 67, 111] & \includegraphics[width=0.25\textwidth]{echonet-plots/EchoNet-precision-0XF70A3F712E03D87.pdf}\\
         0XF60BBEC9C303C98 & [19, 67, 114, 162] & [19, 66, 115, 162] & \includegraphics[width=0.25\textwidth]{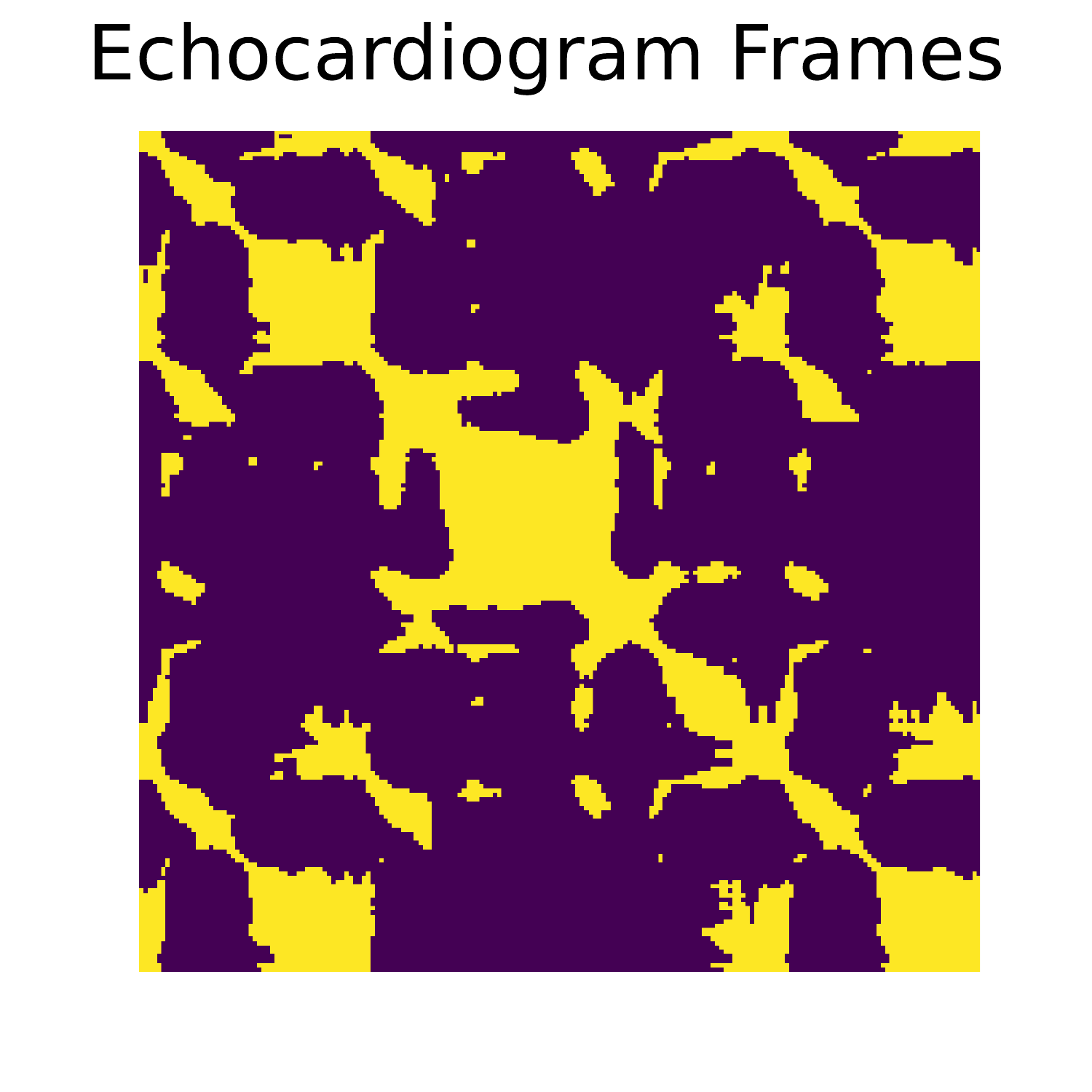}\\
         0XF46CF63A2A1FA90 & [25, 79, 134, 188] & [25, 80, 133, 184] & \includegraphics[width=0.25\textwidth]{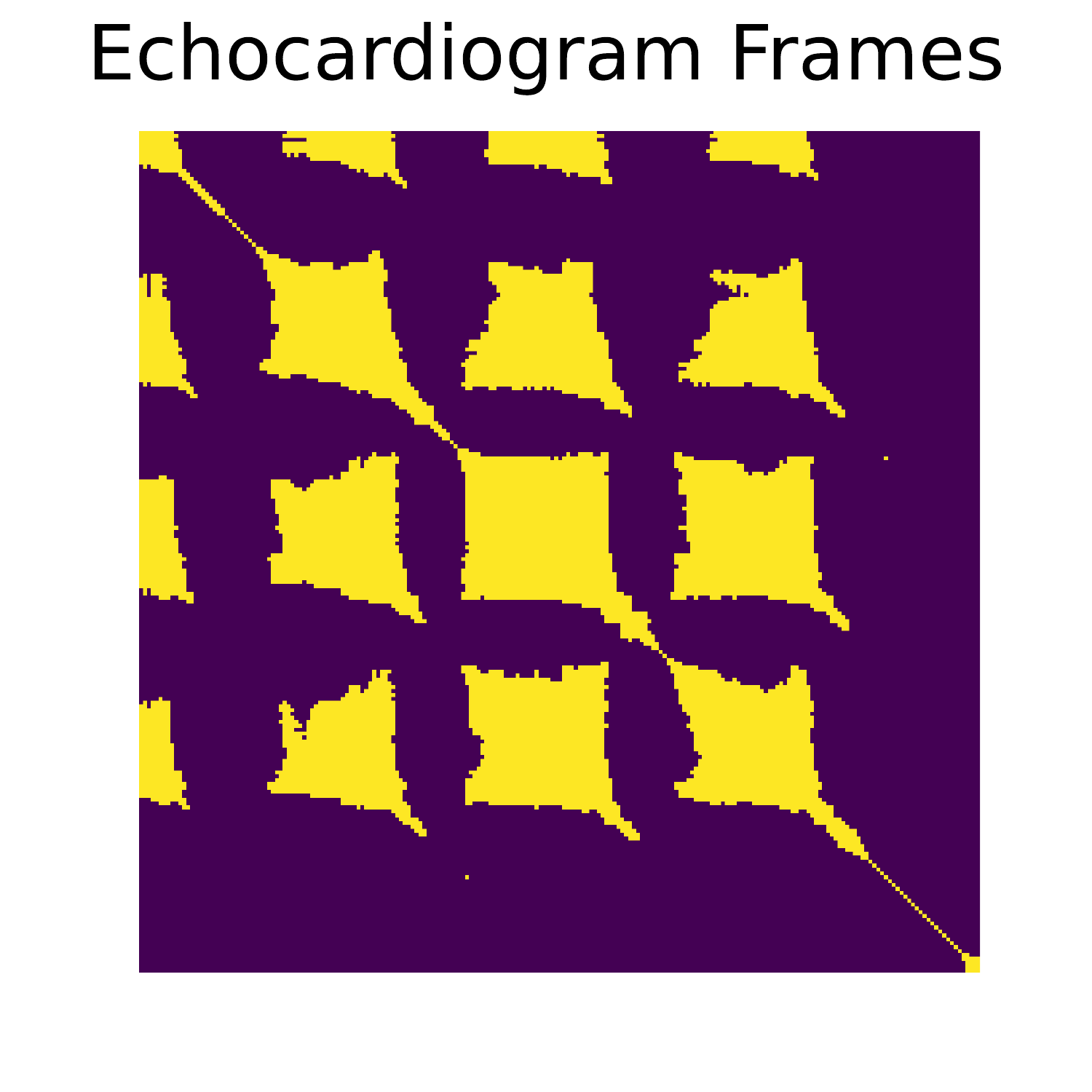}
    \end{tabular}
    \newline
    \caption{Mitral valve labelings and precision matrices for the EchoNet-Dynamic dataset.  The precision matrices, for the most part, seem to have clear off-diagonal structures, as expected, and the mitral valve prediction is generally quite good; it is only significantly off for the last opening in 0XF072F7A9791B060.  The precision matrices are shown with the top 20\% of the edges kept.}
    \label{tab:echonet}
\end{table}

\begin{table}[]
    \centering
    \begin{tabular}{cccc}
         Video ID & Label & Predicted & Precision Matrix\\ \hline
         0XFE6E32991136338 & [17, 47, 77, 106] & [17, 48, 77, 106] & \includegraphics[width=0.2\textwidth]{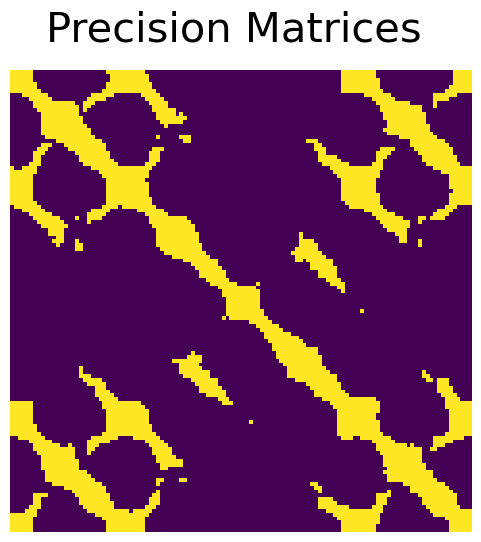}\\
         0XF072F7A9791B060 & [24, 56, 100] & [24, 58, 91] & \includegraphics[width=0.2\textwidth]{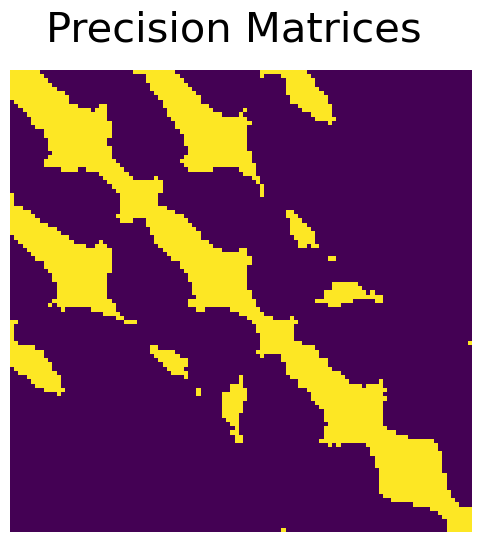}\\
         0XF70A3F712E03D87 & [22, 66, 110] & [22, 67, 112] & \includegraphics[width=0.2\textwidth]{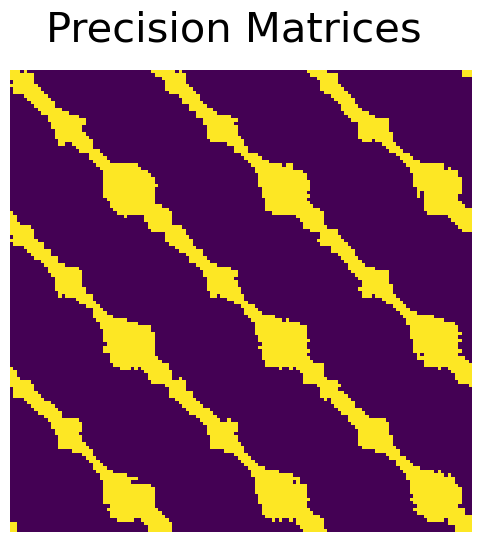}\\
         0XF60BBEC9C303C98 & [19, 67, 114, 162] & [19, 66, 116, 162] & \includegraphics[width=0.2\textwidth]{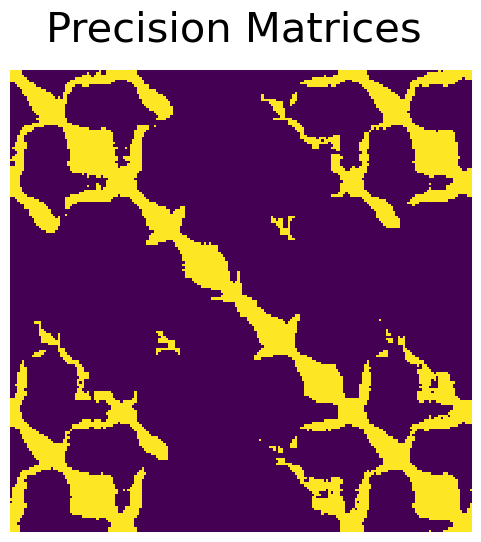}\\
         0XF46CF63A2A1FA90 & [25, 79, 134, 188] & [25, 80, 134, 188] & \includegraphics[width=0.2\textwidth]{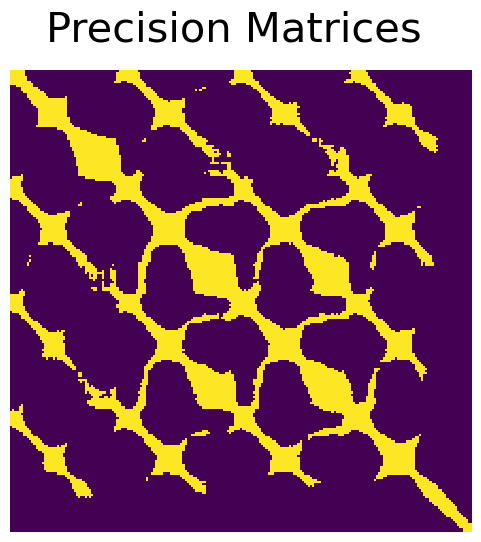}
    \end{tabular}
    \newline
    \caption{Mitral valve labelings and precision matrices for the EchoNet-Dynamic dataset, using the nonparanormal skeptic to preprocess.  The precision matrices, for the most part, seem to have clear off-diagonal structures, as expected, and the mitral valve prediction is generally quite good; it is only significantly off for the last opening in 0XF072F7A9791B060.  The precision matrices are shown with the top 20\% of the edges kept.}
    \label{tab:echonet-nonpara}
\end{table}

\subsection{Mouse embryo stem cell transcriptomics}
\label{sec:mouse}

\subsubsection{The Dataset}

We used the mouse embryo stem cell dataset E-MTAB-2805 \parencite{buettner_computational_2015}, avaliable at \href{https://www.ebi.ac.uk/biostudies/arrayexpress/studies/E-MTAB-2805}{https://www.ebi.ac.uk/biostudies/arrayexpress/studies/E-MTAB-2805}, under a \href{https://creativecommons.org/publicdomain/zero/1.0/legalcode}{Creative Commons Zero license}.  This dataset is by what stage of the cell cycle each cell was in (G1, S, and G2M).  We consider a subset of the genes available in this dataset, as that was the case in the scBiGLasso paper \parencite{li_scalable_2022}; this subset is given in the text file \href{https://github.com/luisacutillo78/Scalable_Bigraphical_Lasso/blob/main/CCdata/Nmythosis.txt}{on their github page for the algorithm}.

\subsubsection{Experiment Justification}

We chose this dataset as it had been considered by multi-axis methods before (albeit in a non-quantitative way) and had a relatively clear task associated with it (cell cycle stage clustering).

\subsubsection{Results - Thresholding Methods}

\begin{figure}
    \begin{subfigure}[b]{0.5\textwidth}
        \centering
        \includegraphics[width=\textwidth]{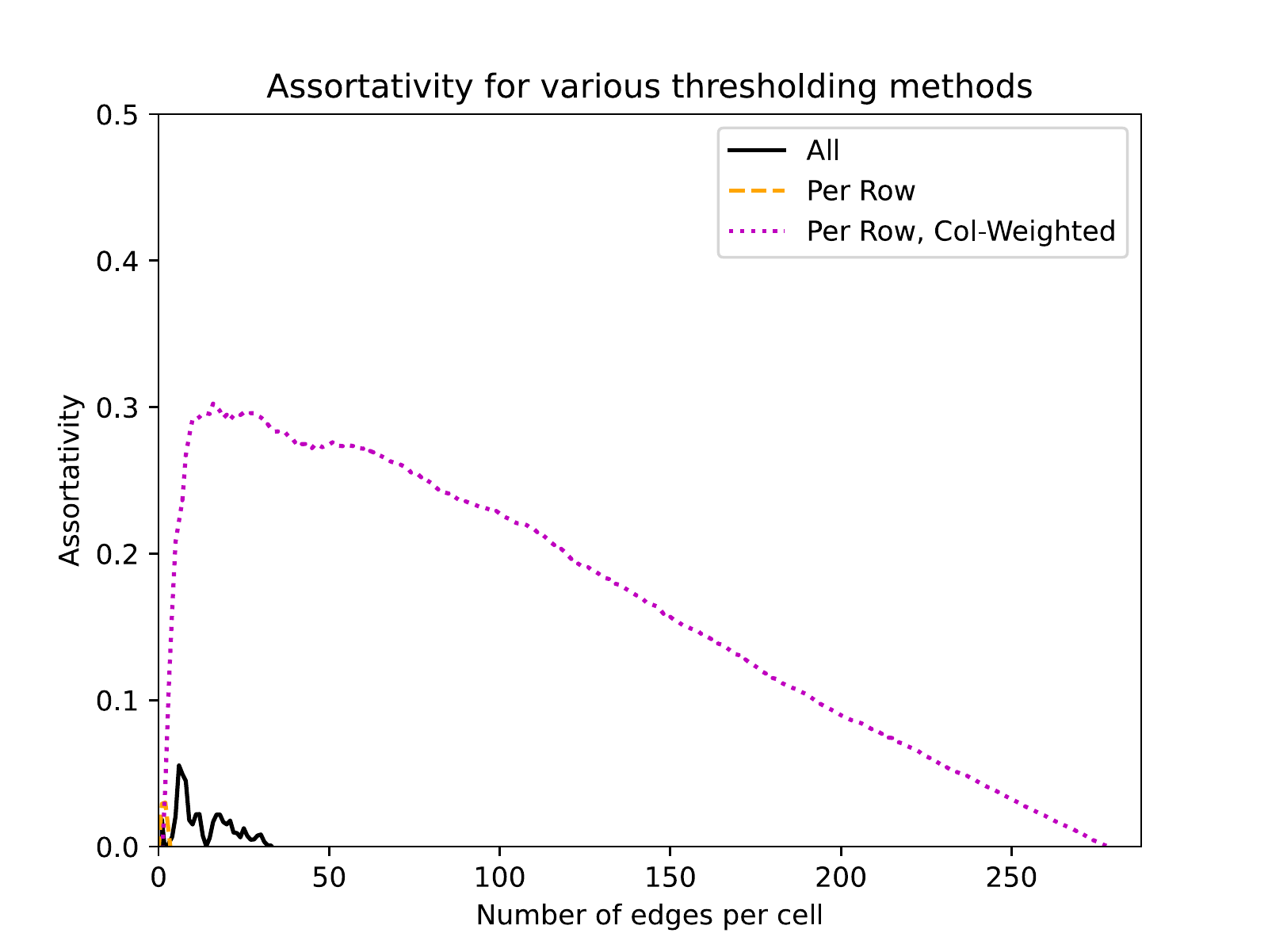}
        \caption{}
    \end{subfigure}
    \begin{subfigure}[b]{0.5\textwidth}
        \centering
        \includegraphics[width=\textwidth]{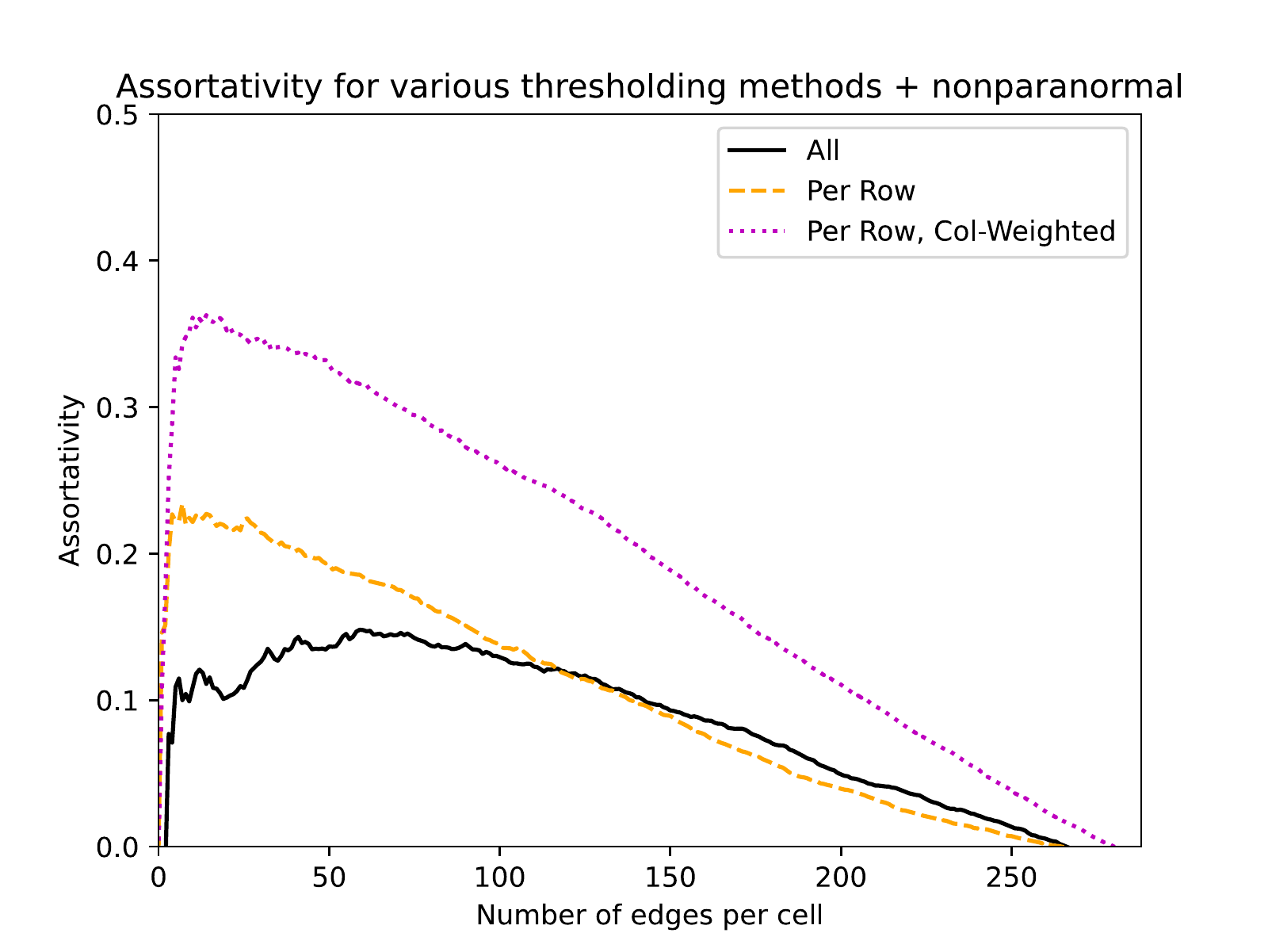}
        \caption{}
    \end{subfigure}
    \caption{(a) Various thresholding methods with log-transformed data.  (b) Various thresholding methods with nonparanormal-skeptic-transformed data.  In both cases, we did not center the data.}
    \label{fig:mouse-assortativity-prelim}
\end{figure}

We experimented with various thresholding methods on this dataset.  One type of thresholding is global, in which one knocks out all edges whose magnitude is below a certain value.  We found that this meant that clusters that tended to have high edge values were preserved relative to clusters who, while distinct, had a lower average edge value.  One attempt to fix this was by thresholding per row; this has the interpretation of keeping the top $n$ edges per vertex.  Per-row thresholding mitigated this problem for some datasets (such as the 10x Genomics one), but not this dataset.  The main problem was that some vertices tended to have higher edges than others - this meant that low-value vertices that were connected to a high-value vertex would always keep their connections to it, even if those connections were not very important from the perspective of the high-value vertex.  To address this, we used a strategy of normalizing each column to sum to one before thresholding per row.

\begin{figure}
    \centering
    \begin{subfigure}[b]{0.45\textwidth}
        \centering
        \includegraphics[width=\textwidth]{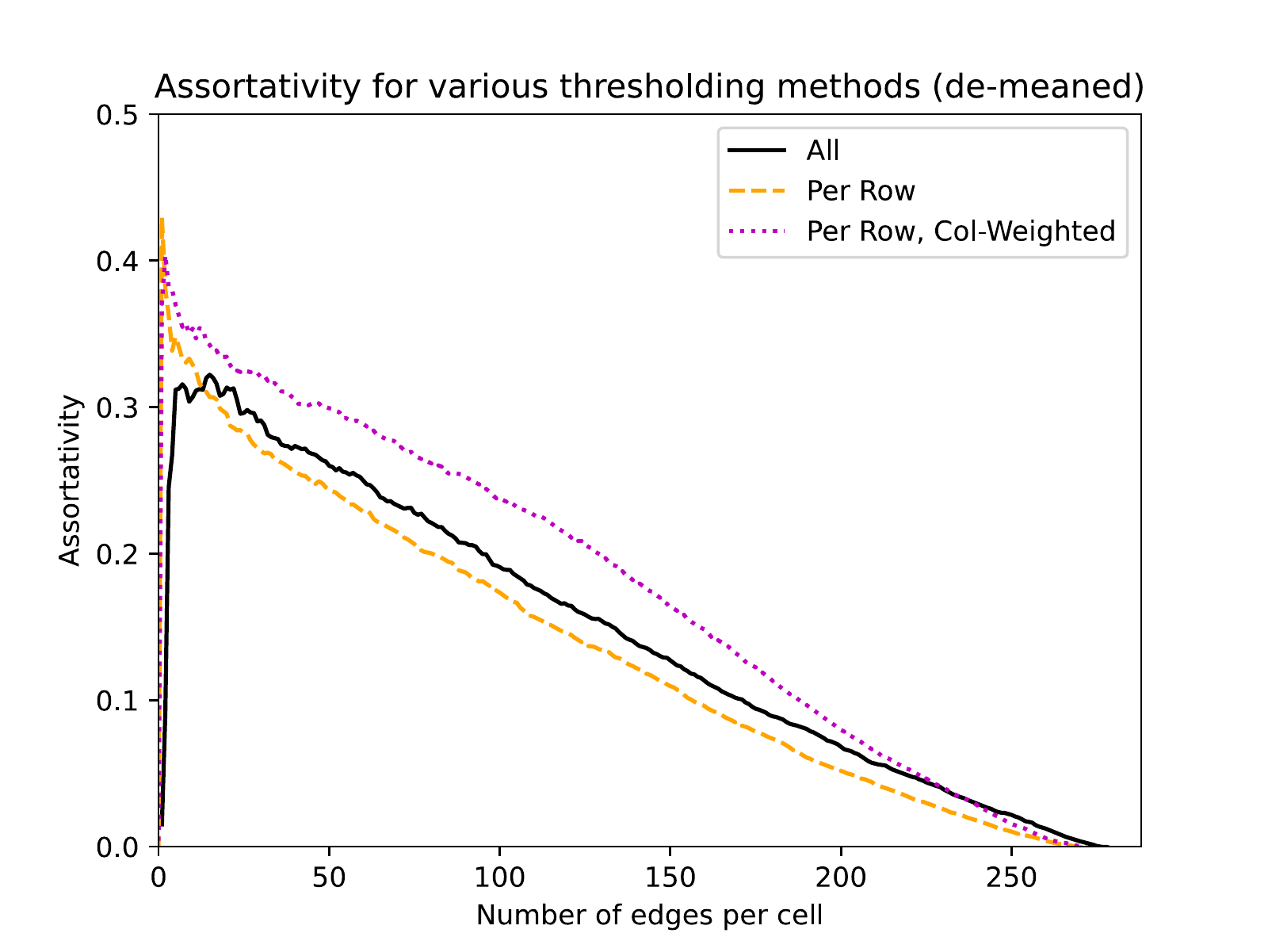}
        \caption{}
    \end{subfigure}
    \begin{subfigure}[b]{0.45\textwidth}
        \centering
        \includegraphics[width=\textwidth]{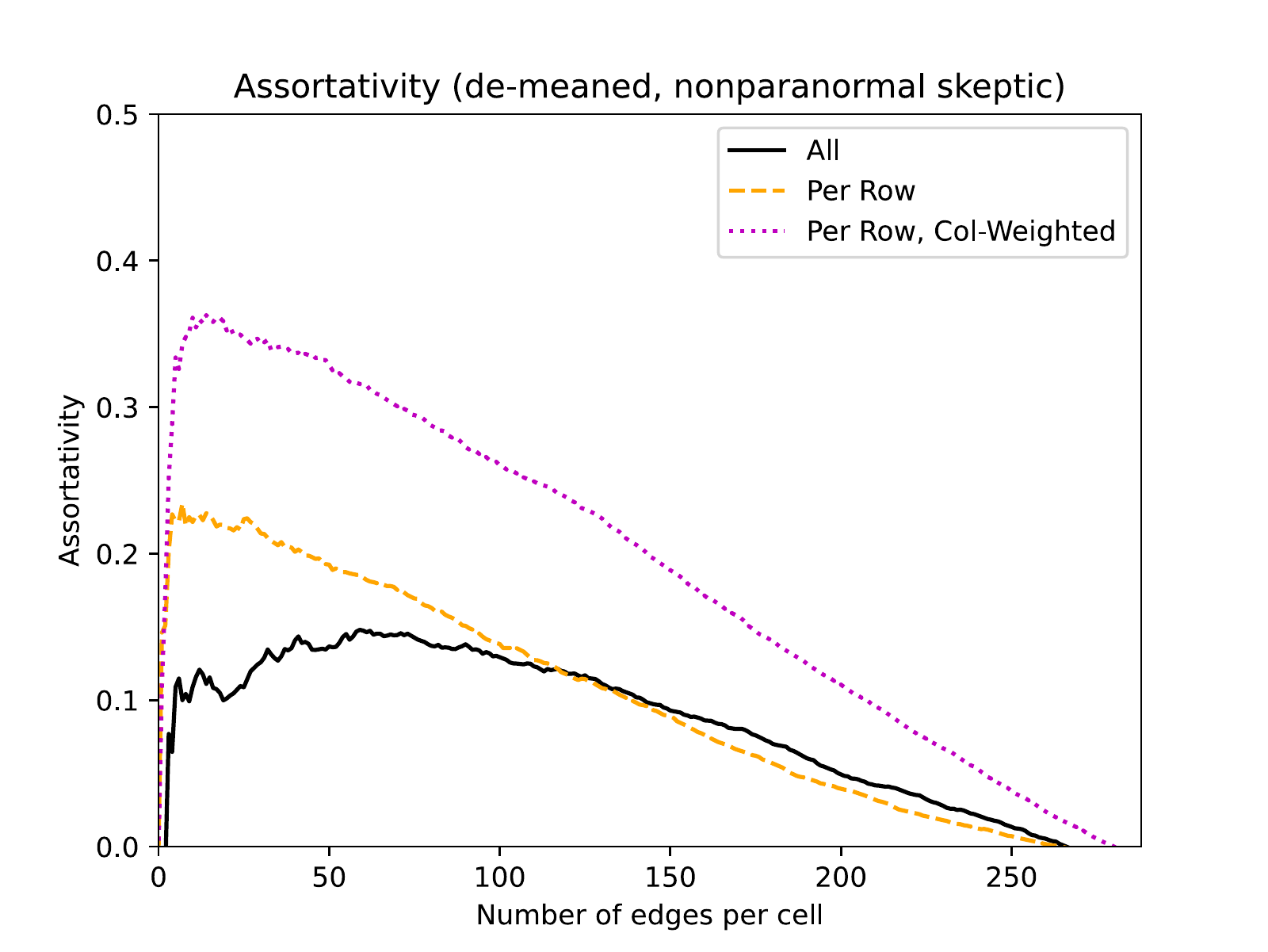}
        \caption{}
    \end{subfigure}
    \caption{Assortativity of the various methods once we center the data. (a) log-transformed (b) nonparanormal skeptic.}
    \label{fig:demeaned-mouse}
\end{figure}

From Figure \ref{fig:mouse-assortativity-prelim}, we can see that, with uncentered data, the third thresholding method is clearly better.  In fact, without the nonparanormal skeptic, the other methods are practically worthless.  However, if we centered the data as described in Section \ref{sec:centering}, all methods perform similarly; the third is still best.  This makes sense, as the model assumes zero-mean data.  Performance doesn't change under the nonparanormal transformation if we center the data, as centering is a monotonic transformation.  Notably, this means that the nonparanormal skeptic does not always yield the best performance.  We can see this in Figure \ref{fig:demeaned-mouse}.

\subsubsection{Results - Regularization}

We wanted to check whether our restricted L1 regularizer would improve performance.  To choose the best regularization parameter, we devised the following test:

\begin{enumerate}
    \item Run GmGM with regularization parameter $\rho$.
    \item Threshold according to the third method (normalize the columns, then pick the top $n$ cells from each row).  We let $n=1$ as we wanted to optimize the parameter for the very sparse case.
    \item Repeat for a range of values of $\rho$.
    \item Measure the percentage of within-group connections.  I.e, if 30\% of the connections from cells in S stage were to other cells in S stage, we would report 30\% for this metric.
    \item Pick the $\rho$ that optimizes the average percentage of within-group connections over all groups.  Run GmGM with that $\rho$.
    \item Calculate the assortativities over all thresholding methods.
\end{enumerate}

We chose this setup as it is a distinct but related task to assortativity.  Its results can be seen in Figure \ref{fig:mouse-reg-vary}.  The best parameters were around 0.0038 for the log-transformed data, and 0.0048 when using the nonparanormal skeptic.  Figure \ref{fig:mouse-reg-assort} demonstrates the advantage of regularization.

\begin{figure}
    \centering
    \begin{subfigure}[b]{0.45\textwidth}
        \includegraphics[width=\textwidth]{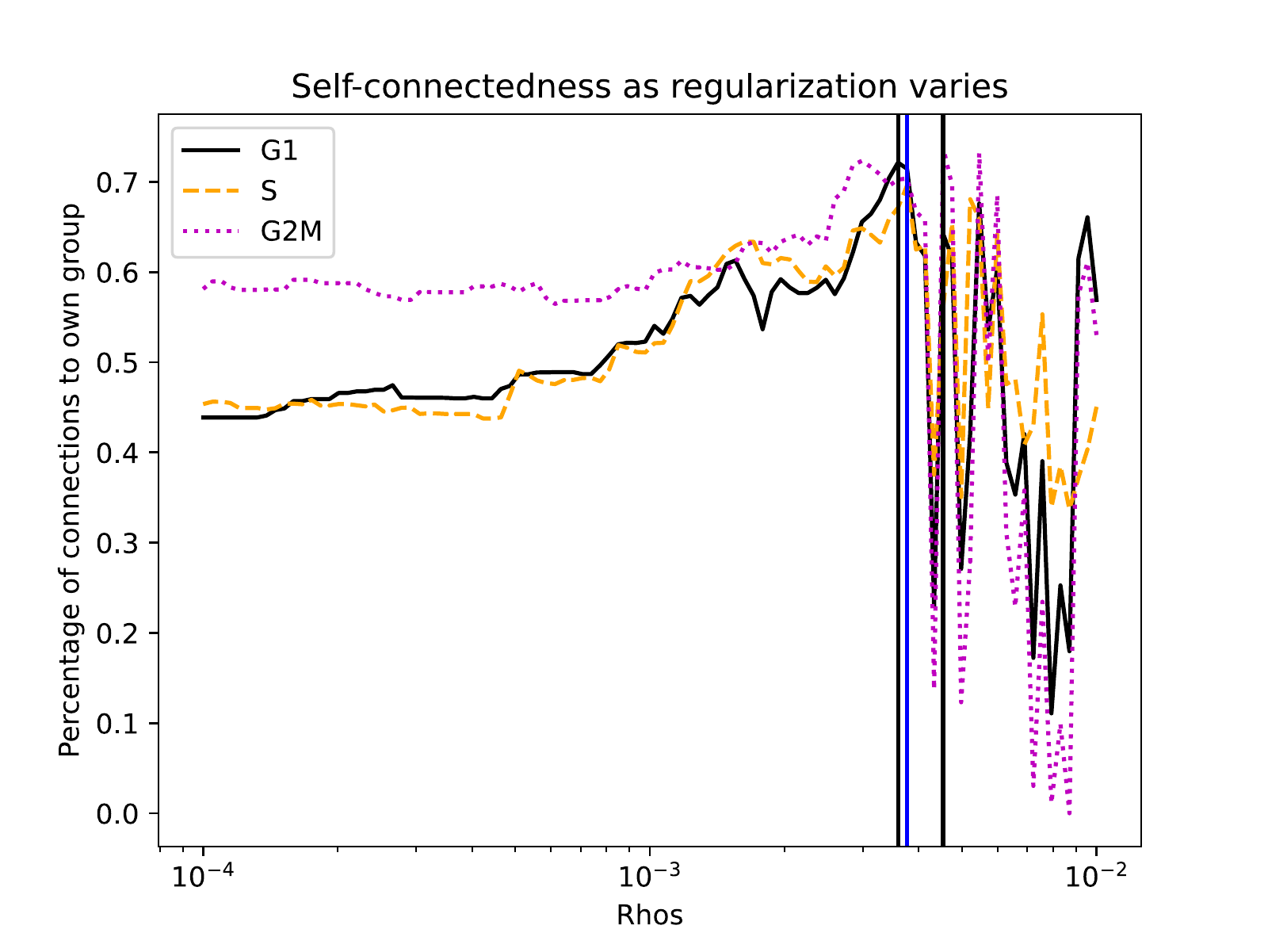}
        \caption{}
    \end{subfigure}
    \begin{subfigure}[b]{0.45\textwidth}
        \includegraphics[width=\textwidth]{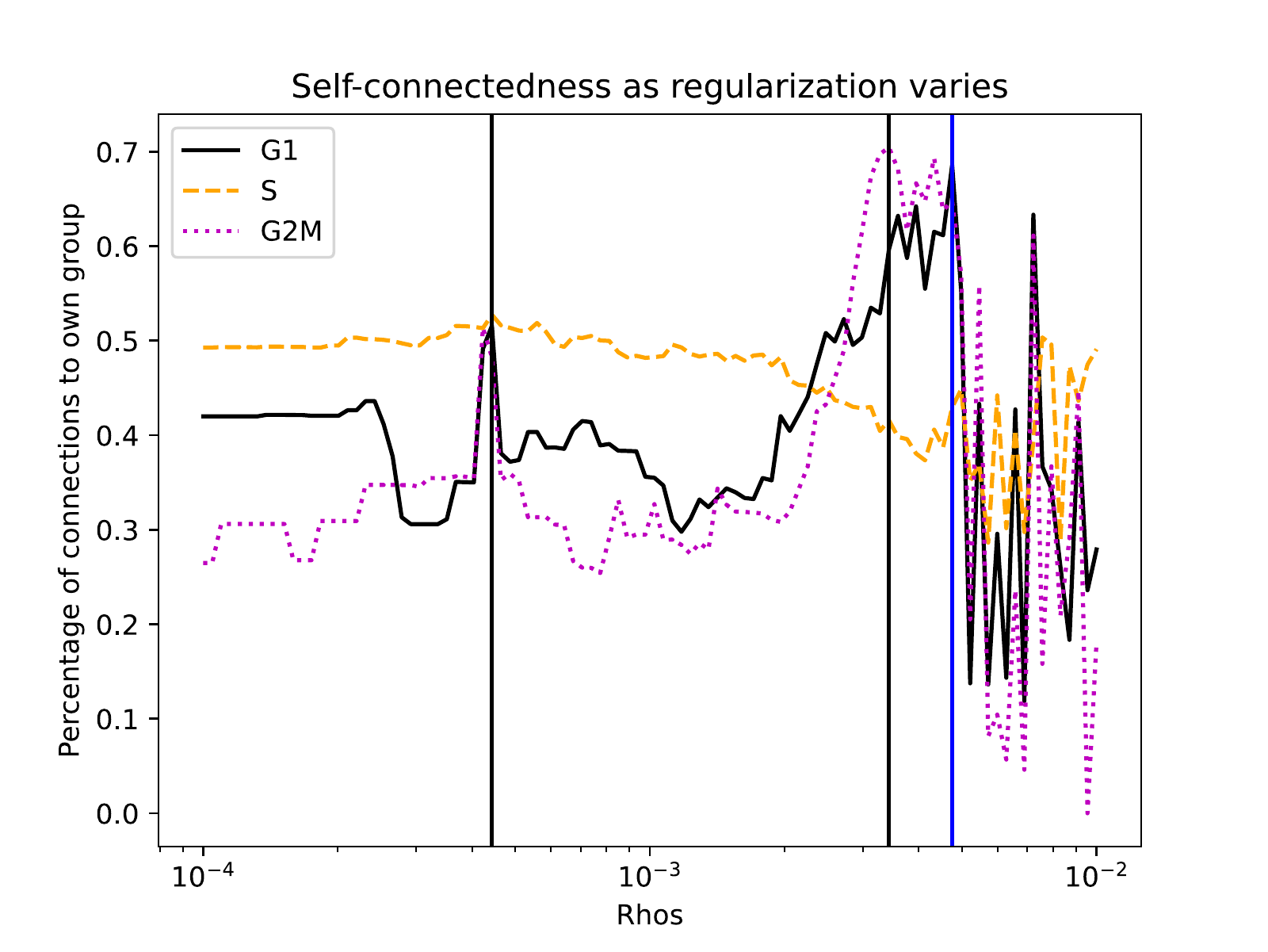}
        \caption{}
    \end{subfigure}
    \caption{The percentage of within-group connections, varied as the regularization parameter increases.  (a) with a log transform (b) with the nonparanormal skeptic.  Vertical lines represent the best values for each class; the blue vertical line represents the best value overall.}
    \label{fig:mouse-reg-vary}
\end{figure}

\begin{figure}
    \centering
    \begin{subfigure}[b]{0.45\textwidth}
        \includegraphics[width=\textwidth]{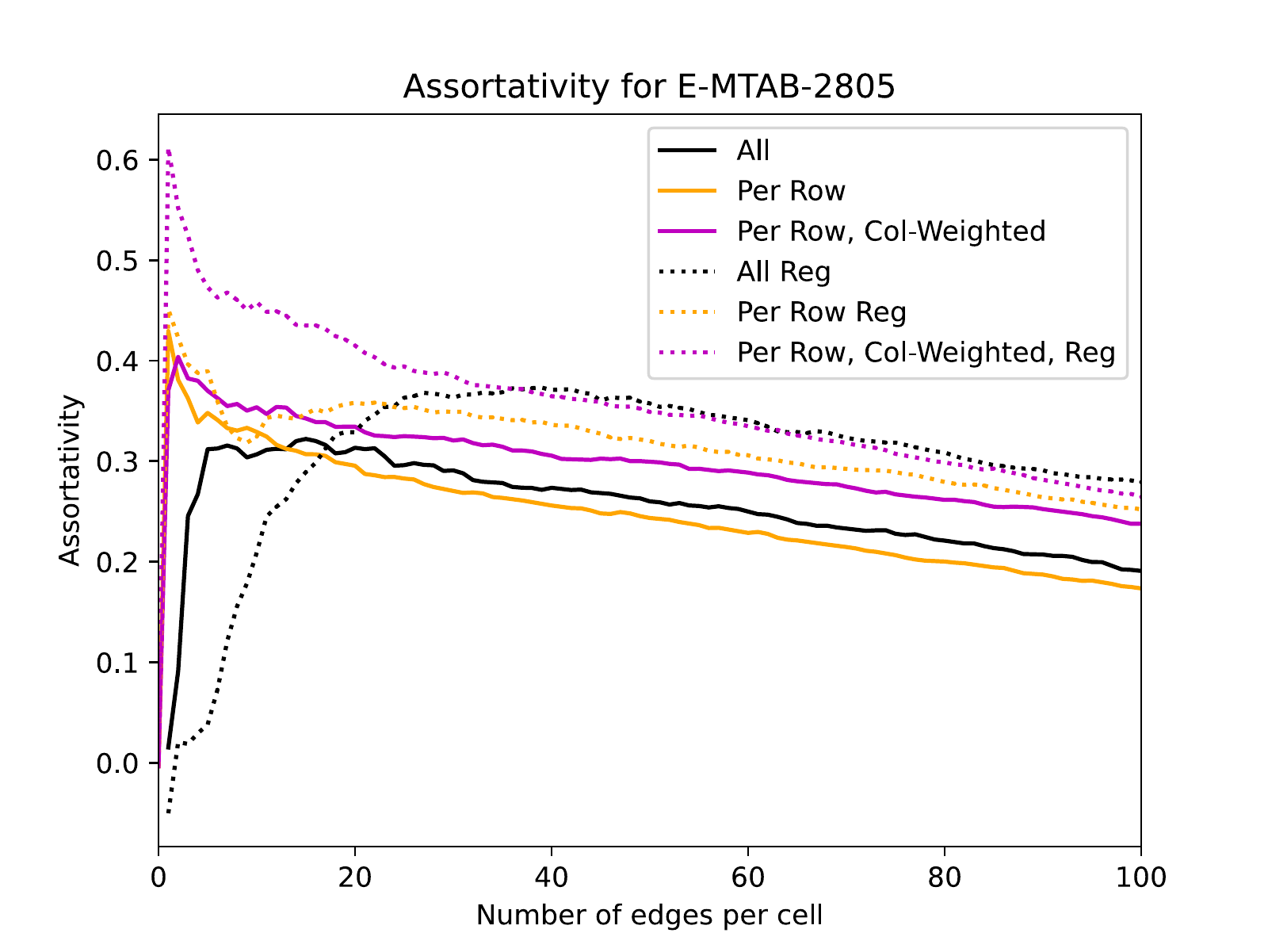}
        \caption{}
    \end{subfigure}
    \begin{subfigure}[b]{0.45\textwidth}
        \includegraphics[width=\textwidth]{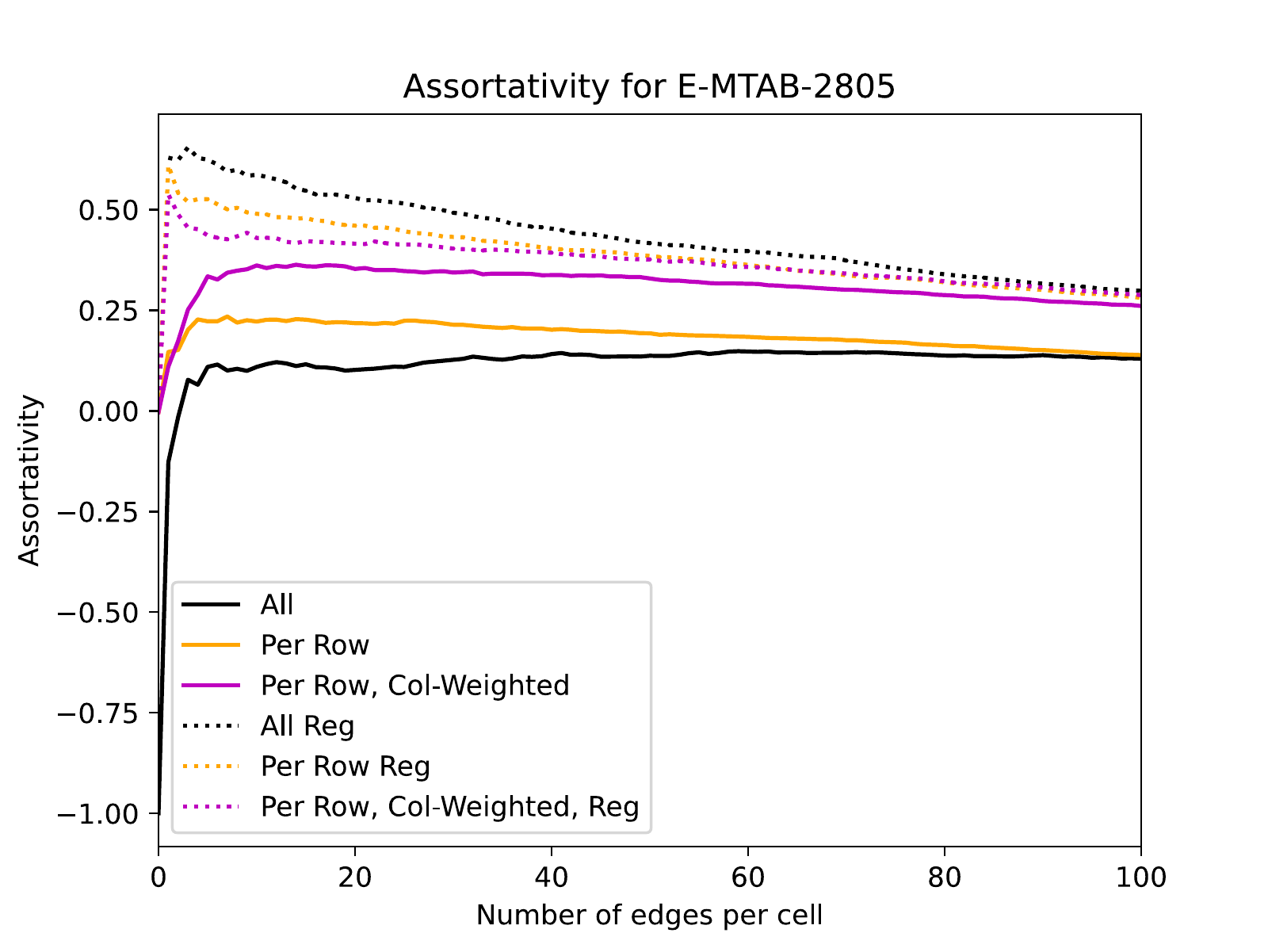}
        \caption{}
    \end{subfigure}
    \caption{Assortativity for GmGM L1, (a) with a log transform and (b) with the nonparanormal skeptic.  Note that the scales between the two plots are different.}
    \label{fig:mouse-reg-assort}
\end{figure}

\subsubsection{Results - Comparison with EiGLasso}

\begin{figure}
    \centering
    \includegraphics{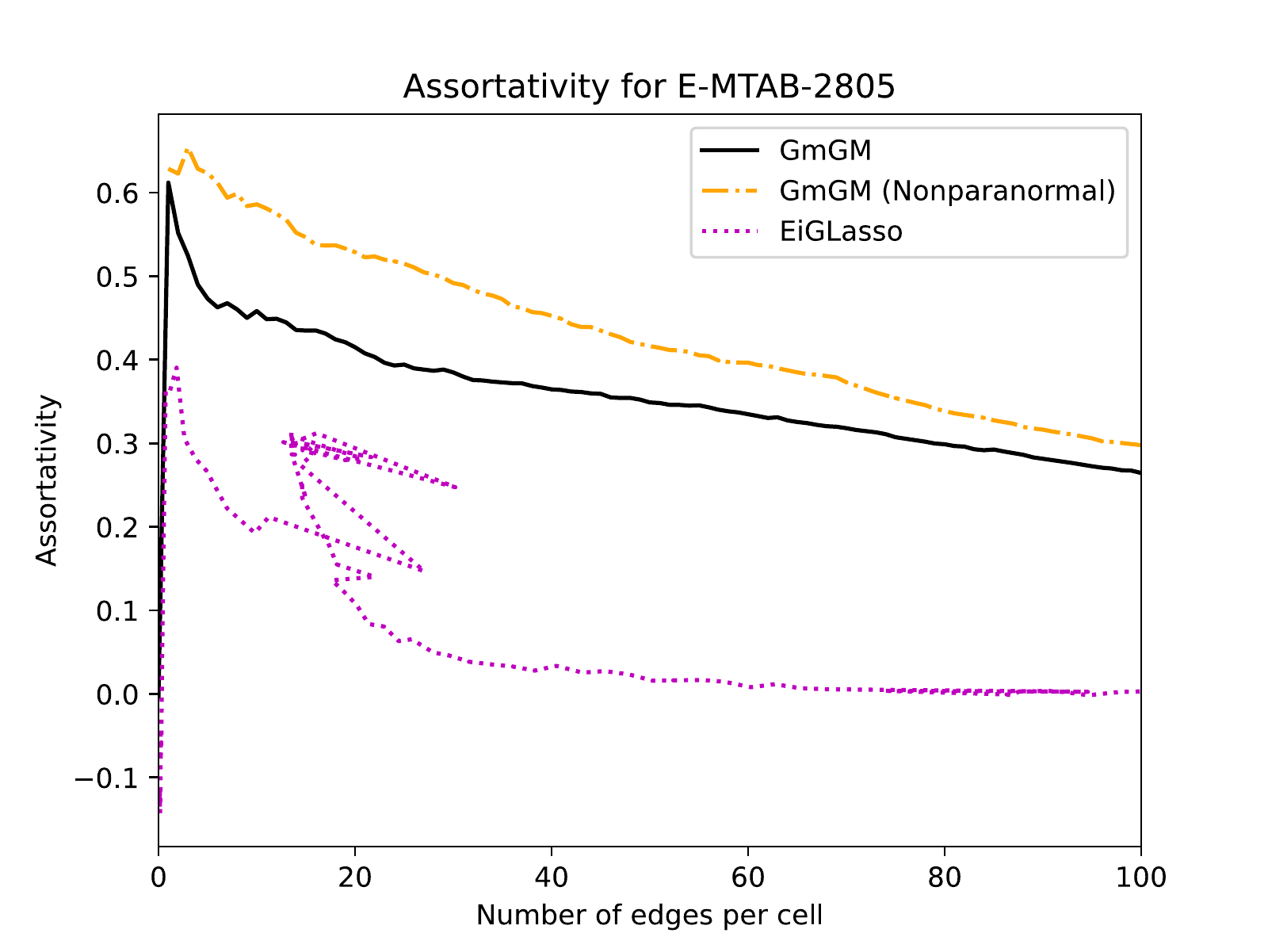}
    \caption{A comparison of GmGM with and without the nonparanormal skeptic, and EiGLasso.  As EiGLasso was much slower than GmGM, we did not rerun the calculation for it with the nonparanormal skeptic.}
    \label{fig:mouse-eiglasso}
\end{figure}

We compare the assortativities of the best GmGM L1 parameters we discovered in the last section with the assortativities of EiGLasso as we varied its regularization parameter.  We found a surprisingly clear advantage when using our method.  This was unexpected, as our method and EiGLasso are quite similar, differing only in how we enforce sparsity.  Judging by the `chink' in EiGLasso's curve (Figure \ref{fig:mouse-eiglasso}, it could be that EiGLasso simply needed more time to converge to a good result.  However, generating this curve for EiGLasso already took a long time, so we did not investigate further.  To compare runtimes, see Table \ref{tab:mouse-runtimes}.  We did not perform a UMAP consistency plot, as cells in UMAP-space formed a homogenous blob.

\begin{table}[t]
    \centering
    \begin{tabular}{c|c|c}
         Algorithm & Preprocessing & Runtime \\
         \hline
         GmGM & Log Transform, Centered & 0.0079 seconds \\
         GmGM & Nonparanormal Skeptic & 0.043 seconds \\
         EiGLasso & Log Transform, Centered & 30 seconds \\
         EiGLasso & Nonparanormal Skeptic & 108 seconds
    \end{tabular}
    \caption{Runtimes of our algorithm and EiGLasso, with various preprocessing techniques, on the E-MTAB-2805 dataset.  Runtimes given are an average over 10 runs.}
    \label{tab:mouse-runtimes}
\end{table}

\subsection{LifeLines-DEEP metagenomics + metabolomics}
\label{sec:lifelines}

\subsubsection{The Dataset}

We used the LifeLines-DEEP metagenomics and metabolomics datasets \parencite{tigchelaar_cohort_2015}.  We did not do any pre-processing to the metabolomics, and we used the already pre-processed version of the metagenomics data from \textcite{prost_zero_2021}.  This dataset is available on the European Genome-Phenome Archive (EGA) under Study ID \href{https://ega-archive.org/studies/EGAS00001001704}{EGAS00001001704}; to access it, one has to agree to an LL-DEEP-specific data access agreement.

\subsubsection{Experiment Justification}

We chose this dataset as it had been considered by prior single-axis work (ZiLN; \textcite{prost_zero_2021}), as well as being multi-omic.  It had a more complicated class structure (taxonomy) than the E-MTAB-2805 dataset, so it was a natural next step in our investigations of the performance of this algorithm.

\subsubsection{Results}

We kept only patients that appeared in both datasets.  We compared our model's results to the model given by \textcite{prost_zero_2021} in the main paper.  We report runtimes in Table \ref{tab:lifelines-runtimes}, and give a UMAP consistency plot as in Figure \ref{fig:lifelines-umap-consistency}.  We can see that our algorithm is the fastest on the metagenomics dataset, even outperforming a single-axis method.  Furthermore, our results seem quite sensible in UMAP-space.

\begin{table}[b]
    \centering
    \begin{tabular}{c|ccc}
        Algorithm &  Axes & Speed & Per-Axis Speed\\
        \hline
        ZiLN & Species & 3.2 & 3.2 \\
        GmGM & Species, People & 2.59 & 1.30 \\
        GmGM & Species, People, Metabolites & 22.18 & 7.39 \\
        TeraLasso & Species, People &  1299.33 & 649.67
    \end{tabular}
    \caption{Runtimes of various algorithms on the LifeLines-DEEP dataset.  Speed is measured in seconds}
    \label{tab:lifelines-runtimes}
\end{table}

\begin{figure}[h]
    \centering
    \includegraphics[width=0.8\textwidth]{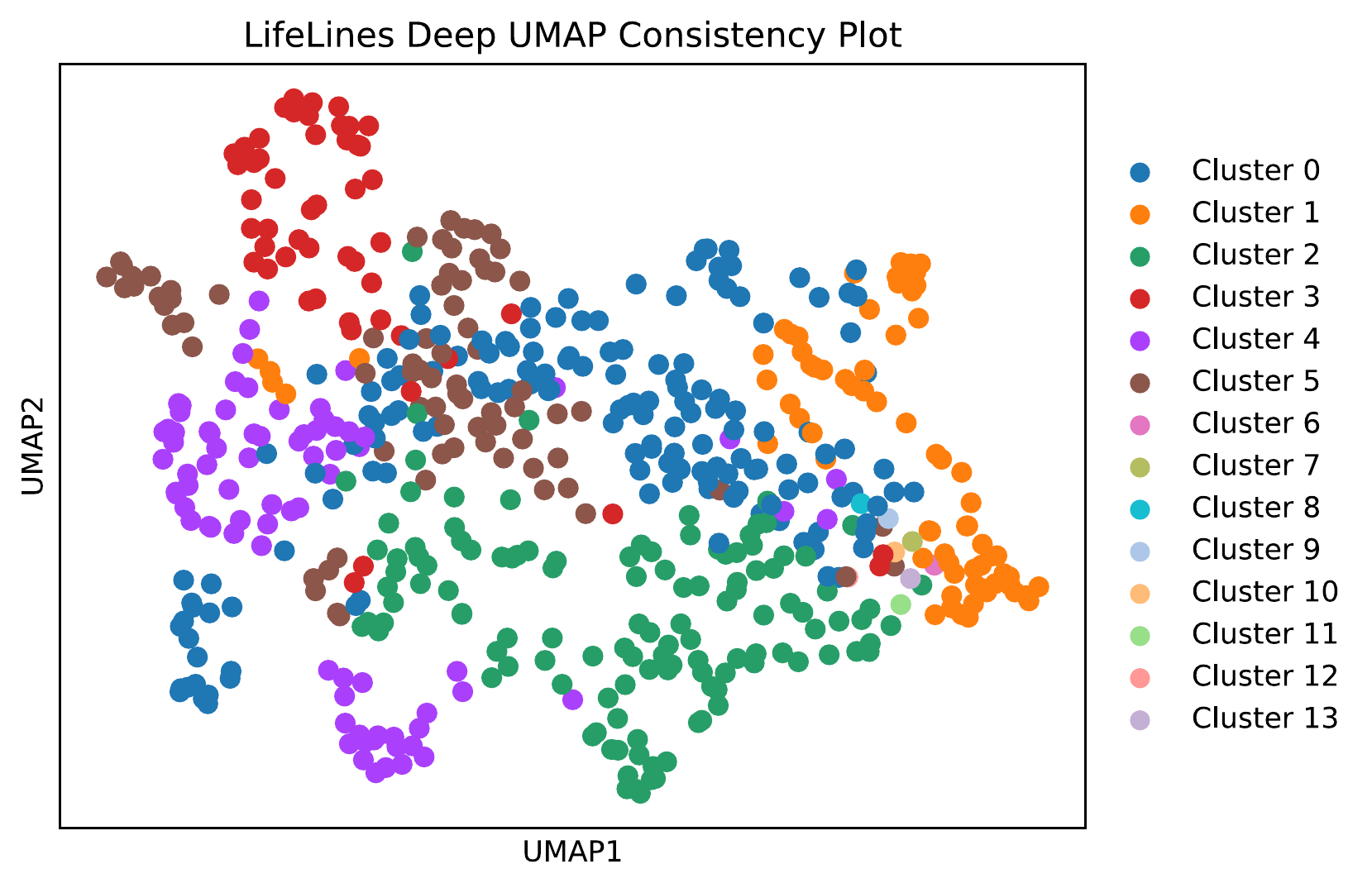}
    \caption{UMAP consistency plot for our algorithm applied to the LifeLines-DEEP dataset.  The UMAP was constructed from log-transformed metagenomics data; the coloring represents a Louvain clustering of the output of GmGM performed on the centered log-transformed metagenomics+metabolomics data, thresholded by keeping the top 7\% of edges overall.}
    \label{fig:lifelines-umap-consistency}
\end{figure}

Finally, we compared our assortativity to that of the Zero-Inflated Log Normal model by \textcite{prost_zero_2021}.  We found that our method performed similarly to theirs; see Figure \ref{fig:lifelines-assort}.

\begin{figure}[h]
    \centering
    \begin{subfigure}[b]{0.45\textwidth}
        \includegraphics[width=\textwidth]{final-experiments/gmgm-metagenomics-assortativity.pdf}
        \caption{}
    \end{subfigure}
    \begin{subfigure}[b]{0.45\textwidth}
        \includegraphics[width=\textwidth]{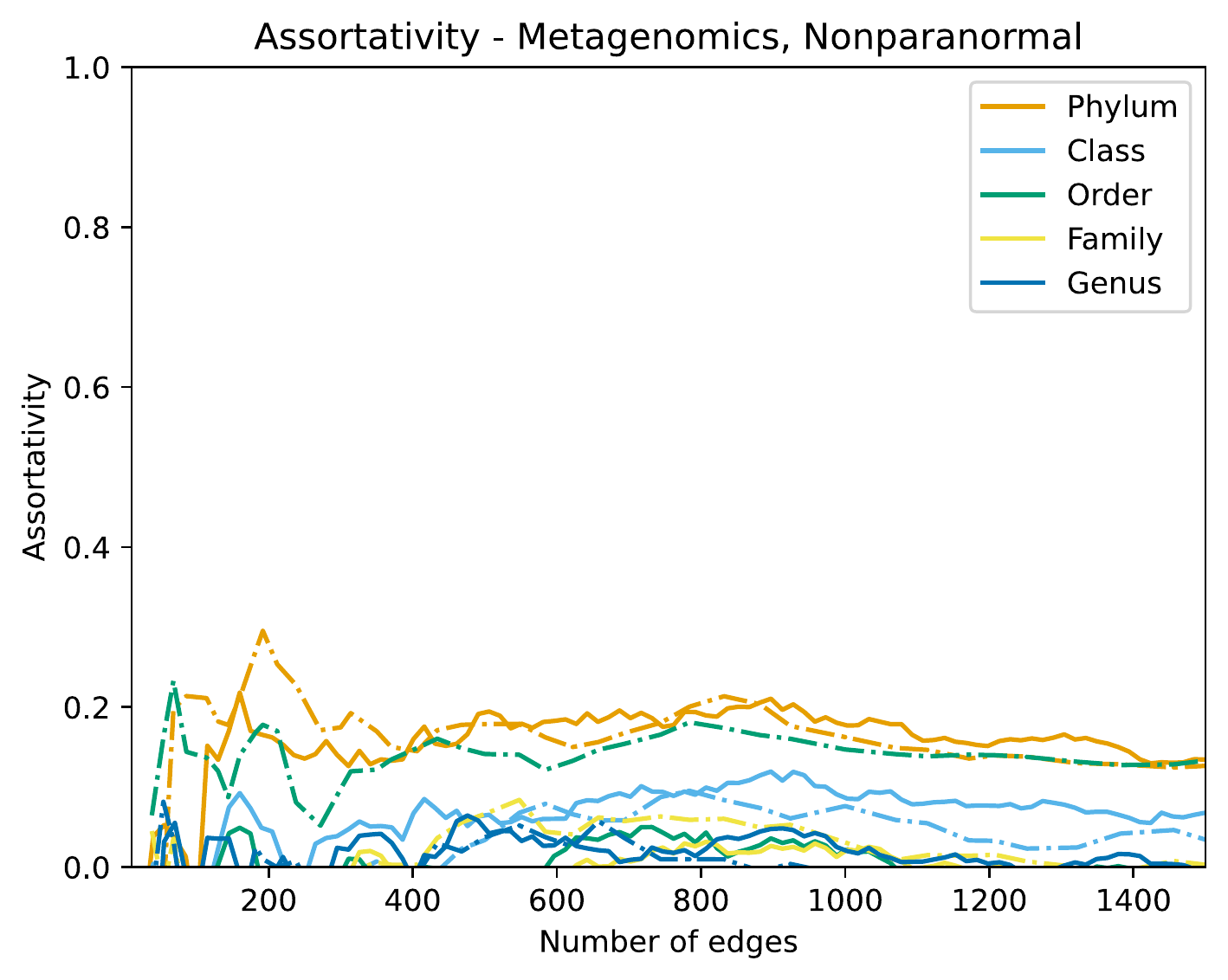}
        \caption{}
    \end{subfigure}
    \begin{subfigure}[b]{0.45\textwidth}
        \includegraphics[width=\textwidth]{final-experiments/gmgm-both-assortativity.pdf}
        \caption{}
    \end{subfigure}
    \begin{subfigure}[b]{0.45\textwidth}
        \includegraphics[width=\textwidth]{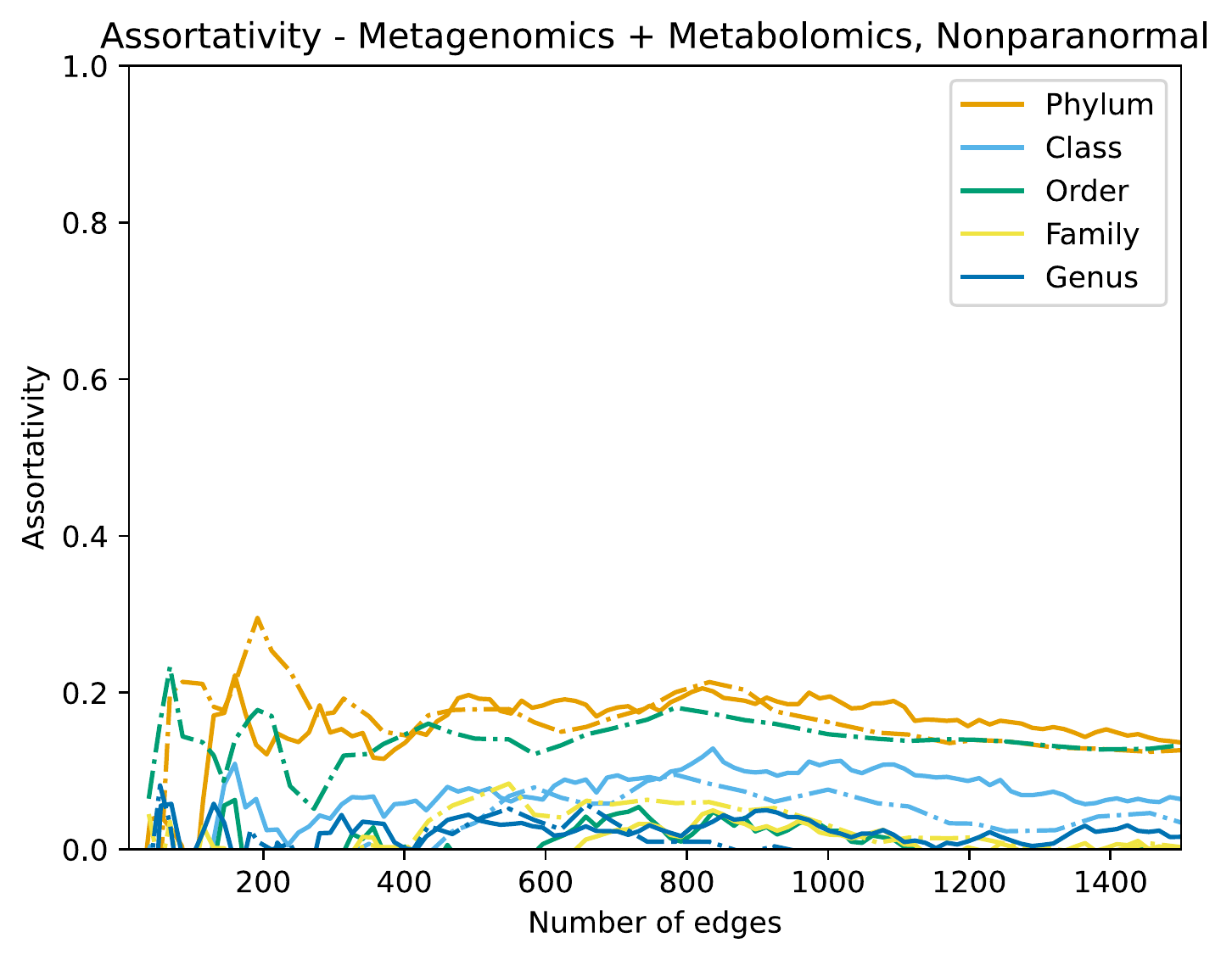}
        \caption{}
    \end{subfigure}
    \caption{Comparison of our method (solid lines) to that of \textcite{prost_zero_2021} (dashed lines).  (a) Just metagenomics, log-transformed. (b) Metagenomics using the nonparanormal skeptic.  (c) The full multi-omic dataset, log-transformed. (d) The full multi-omic dataset using the nonparanormal skeptic.}
    \label{fig:lifelines-assort}
\end{figure}

Finally, we experimented with incorporating priors in our model.  It is not unreasonable to find oneself in a situation in which one knows the highest-level taxonomic categorization of an organism (its phylum), but not know its lower-level categorizations.  Thus, we chose a Wishart prior whose covariance matrix encoded the adjacency matrix of the following graph:

\begin{enumerate}
    \item If two species are in the same phylum, they are connected
    \item If two species are in a different phylum, they are not connected.
\end{enumerate}

This \textit{vastly} improves performance (Figure \ref{fig:lifelines-prior}).  One might worry that the algorithm is just `memorizing' the prior, and that the lower taxonomic ranks are improving because no genuses are cross-phylum, for example.  To show that this is not the case, we conduct a quick experiment with synthetic data, in which we feed the true graph in as the prior.  If our algorithm were memorizing priors, then we would expect performance with the prior to be perfect.  However, what we instead see is a huge but imperfect gain in performance (Figure \ref{fig:synthetic-prior}).

\begin{figure}[h]
    \centering
    \begin{subfigure}{0.45\textwidth}
        \includegraphics[width=\textwidth]{final-experiments/gmgm-both-assortativity-prior.pdf}
        \caption{}
        \label{fig:lifelines-prior}
    \end{subfigure}
    \begin{subfigure}{0.45\textwidth}
        \includegraphics[width=\textwidth]{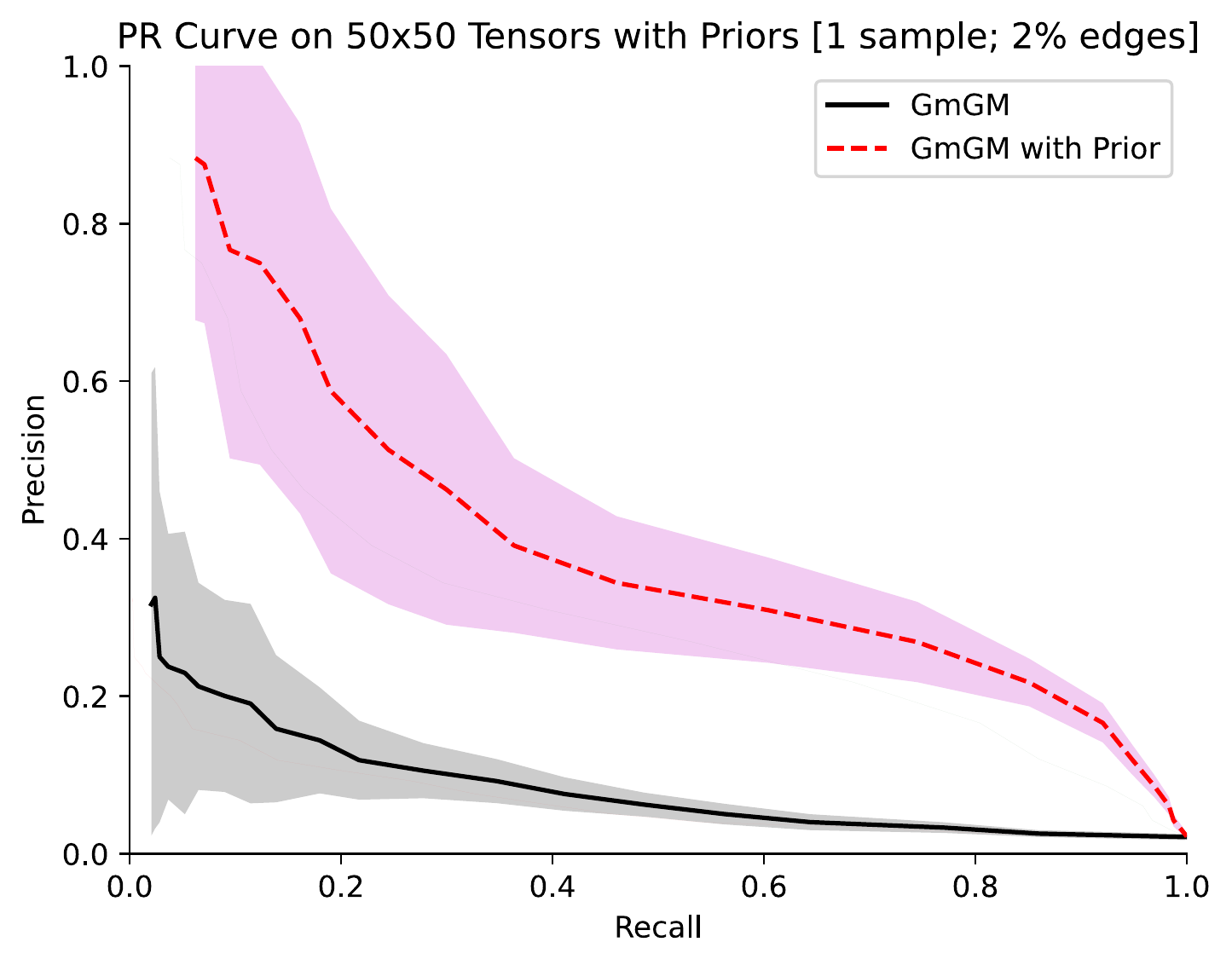}
        \caption{}
        \label{fig:synthetic-prior}
    \end{subfigure}
    \caption{(a) The assortativity when incorporating prior knowledge. 
 (b) A test on 50x50 1-sample synthetic data with 2\% edges connected, with the true graphs fed in as the parameter to a Wishart prior to our algorithm.}
    \label{fig:priors}
\end{figure}

\begin{figure}
    \begin{subfigure}[b]{0.45\textwidth}
        \centering
        \includegraphics[width=\textwidth]{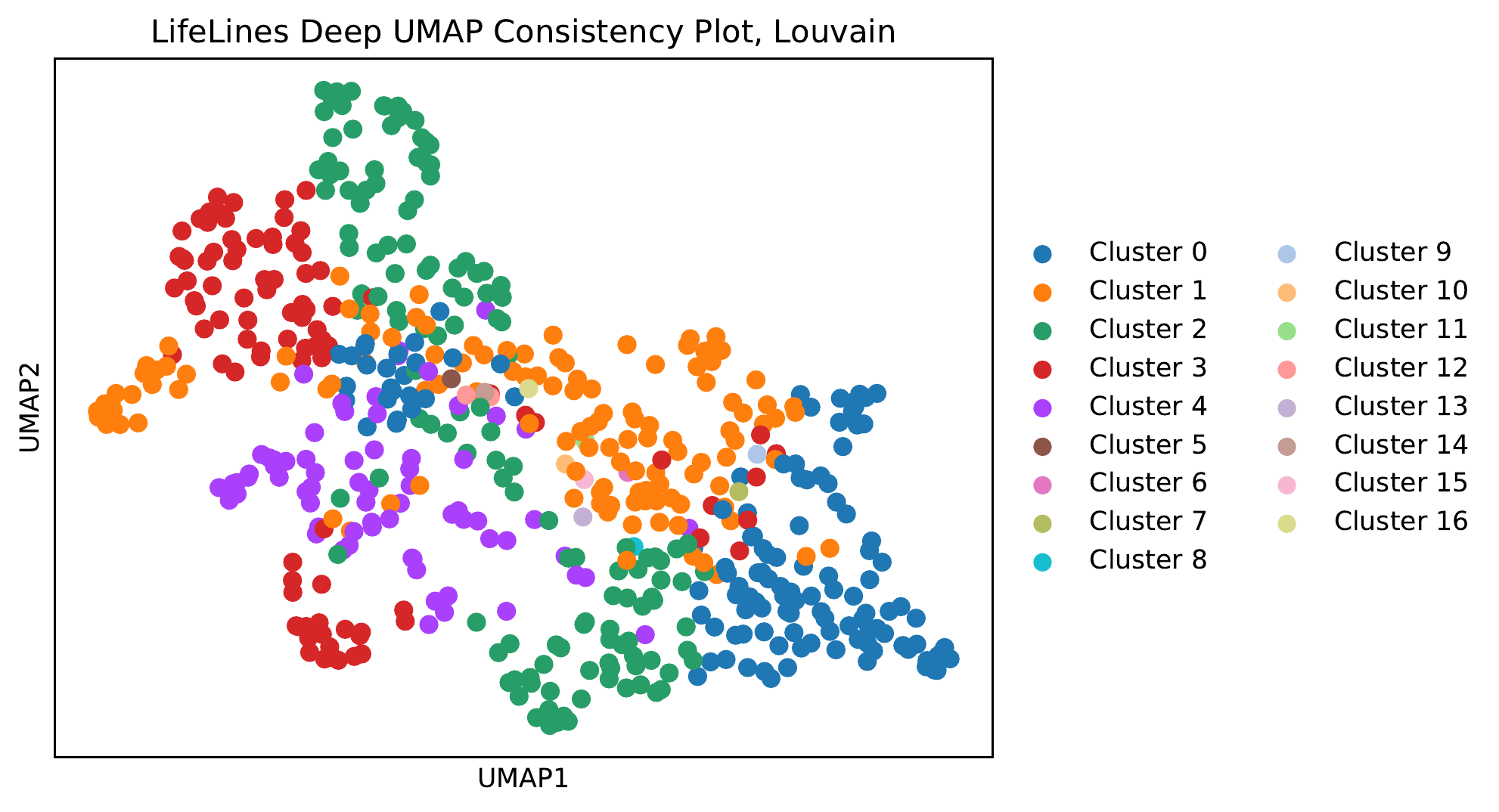}
        \caption{}
    \end{subfigure}
    \begin{subfigure}[b]{0.45\textwidth}
        \centering
        \includegraphics[width=\textwidth]{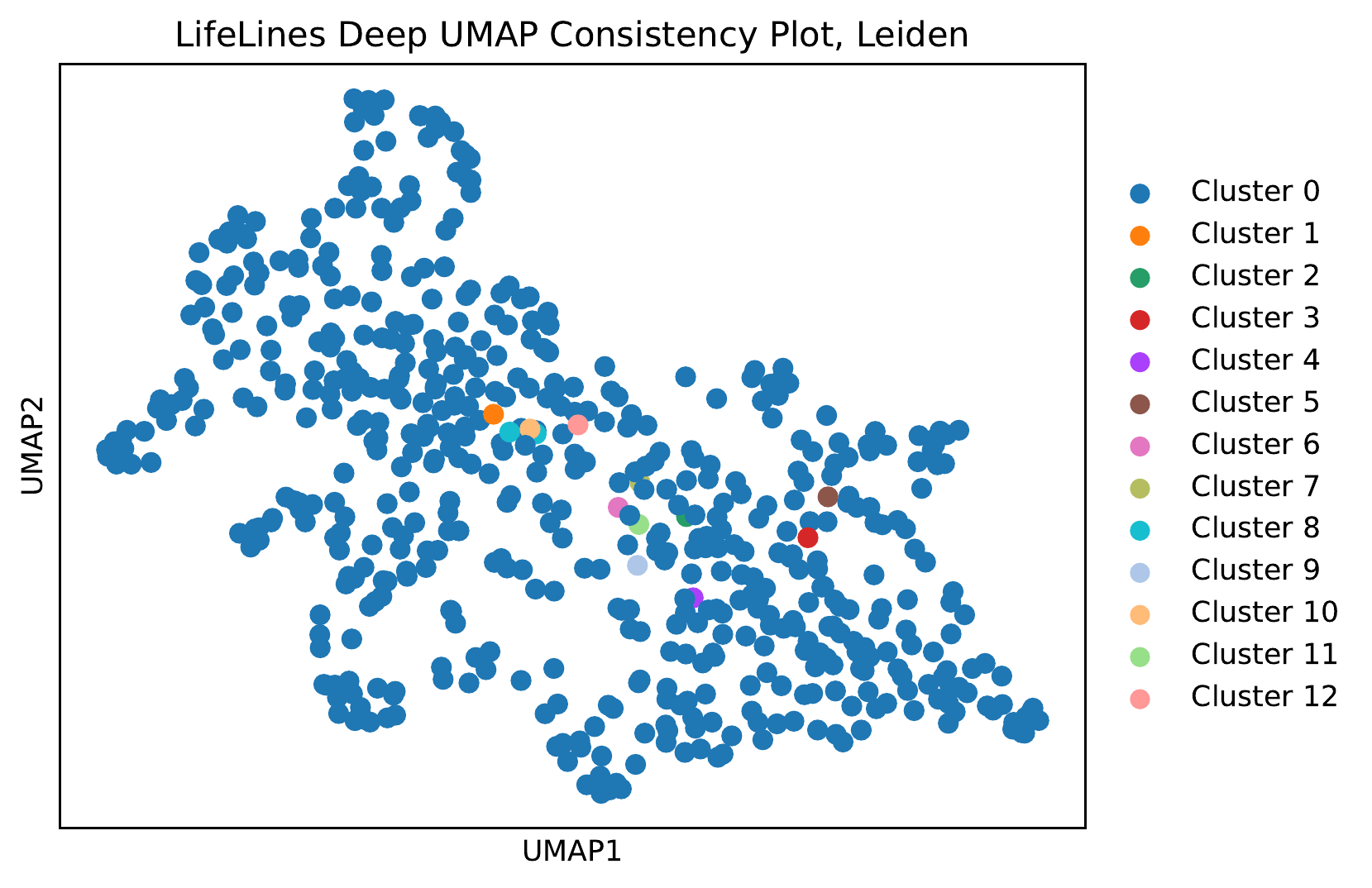}
        \caption{}
    \end{subfigure}
    \begin{subfigure}[b]{0.45\textwidth}
        \centering
        \includegraphics[width=\textwidth]{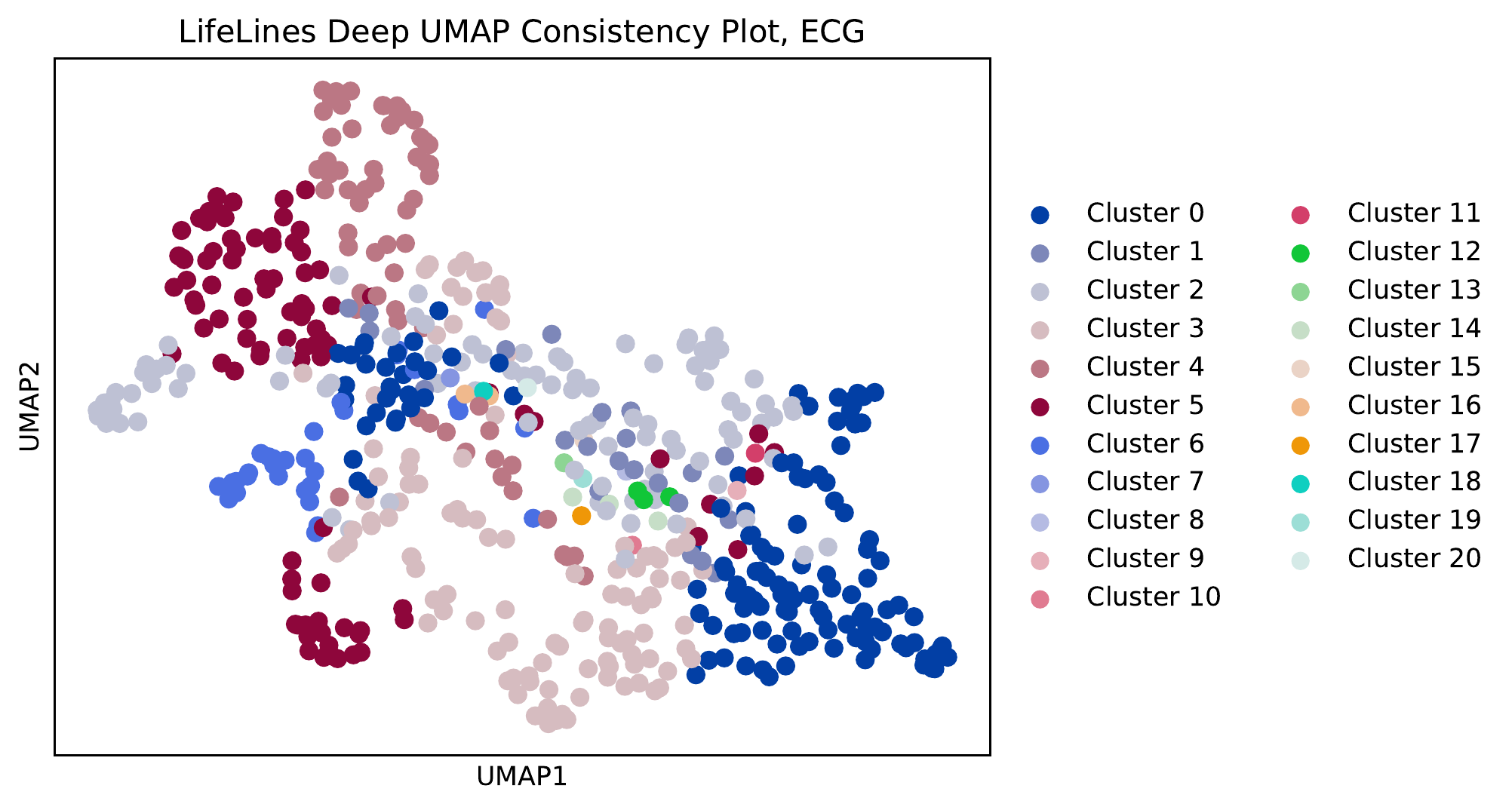}
        \caption{}
    \end{subfigure}
    \begin{subfigure}[b]{0.45\textwidth}
        \centering
        \includegraphics[width=\textwidth]{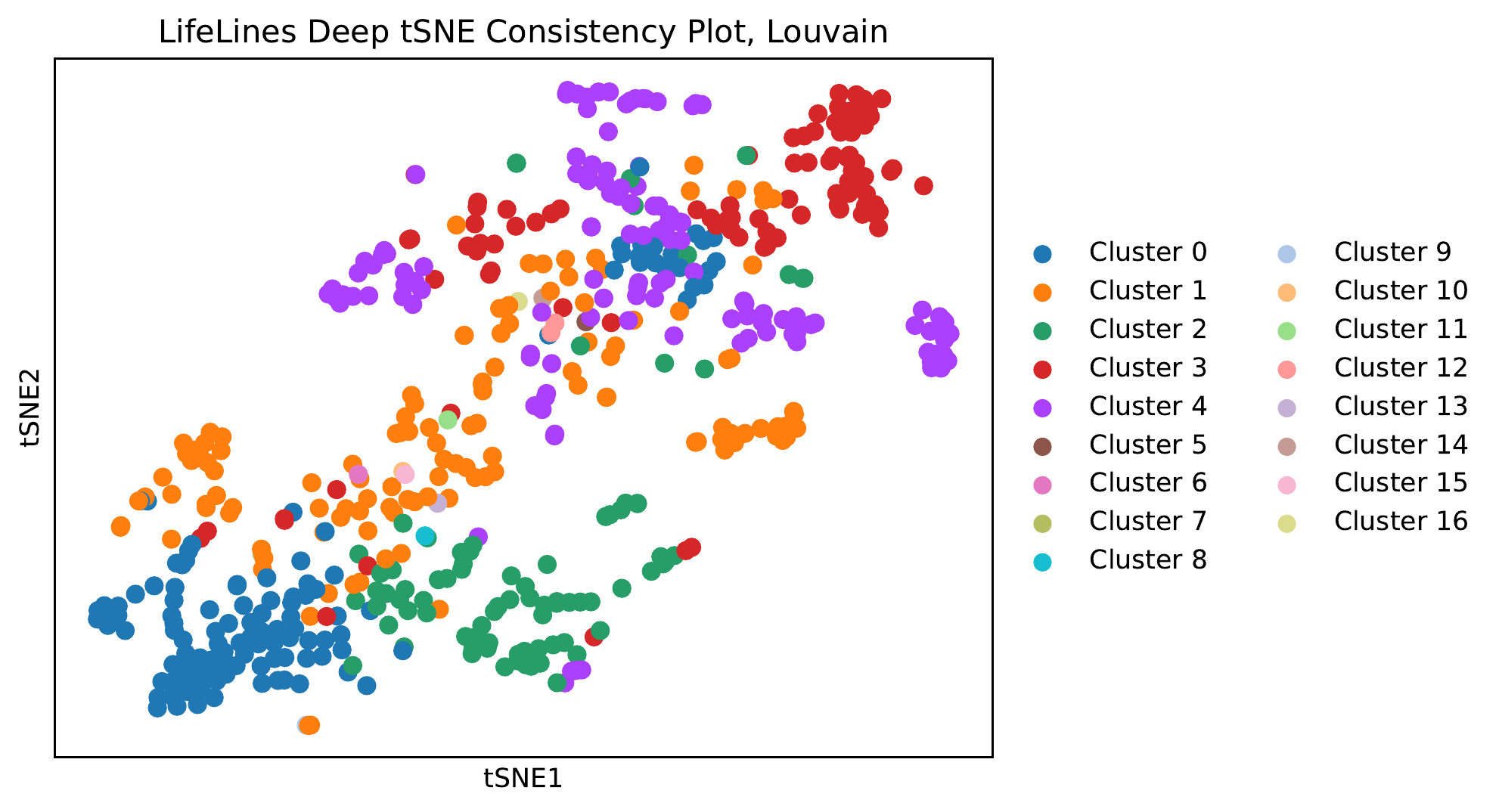}
        \caption{}
    \end{subfigure}
    \begin{subfigure}[b]{0.45\textwidth}
        \centering
        \includegraphics[width=\textwidth]{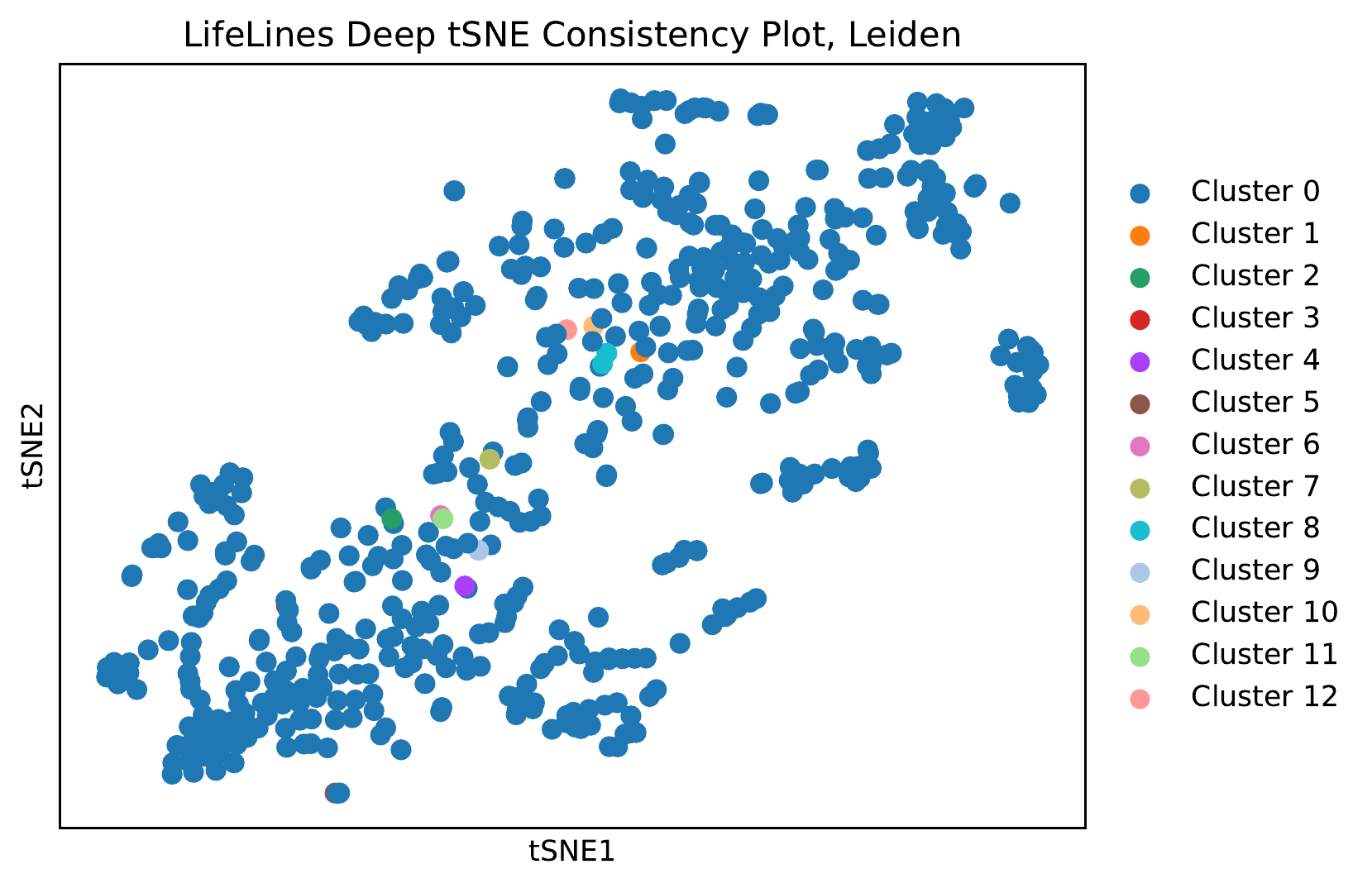}
        \caption{}
    \end{subfigure}
    \begin{subfigure}[b]{0.45\textwidth}
        \centering
        \includegraphics[width=\textwidth]{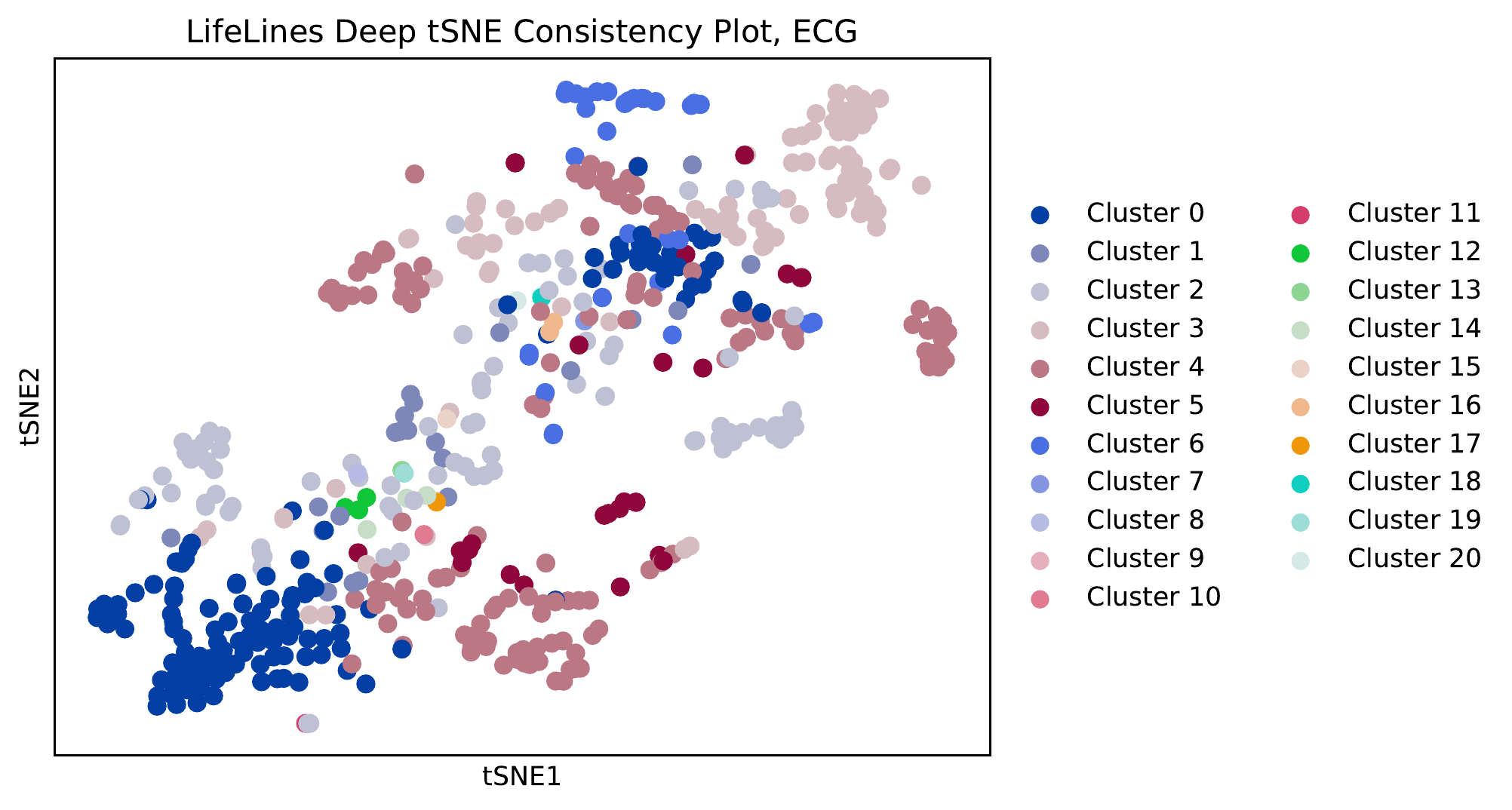}
        \caption{}
    \end{subfigure}
    \caption{UMAP and tSNE consistency plots with Louvain, Leiden, and ECG clustering algorithms on the Lifelines-Deep data.}
    \label{fig:lifelines-deep-all}
\end{figure}

\subsection{10x Genomics flash frozen lymph nodes}
\label{sec:10x}

\subsubsection{The Dataset}
For this experiment, we looked at a single-cell RNA-sequencing+ATAC-sequencing dataset from 10x Genomics \parencite{10x_genomics_flash-frozen_2021}.  It is available \href{https://www.10xgenomics.com/resources/datasets/fresh-frozen-lymph-node-with-b-cell-lymphoma-14-k-sorted-nuclei-1-standard-2-0-0}{here}.  The dataset is publicly available, but one does need to fill in a brief form for data collection purposes.  They do not specify a license.

\subsubsection{Experiment Justification}

It is a realistically-sized multi-omics dataset, and by far the largest considered in this paper.  It is so large that the use of prior work becomes infeasible, so demonstrating that our algorithm works on it would represent a significant leap forward.  The dataset is unlabeled, so we evaluated our algorithm by means of a UMAP consistency plot.

\subsubsection{Results}
Before performing the experiment, we removed cells whose library size was three median absolute deviations from the median, and similarly removed genes and peaks if the the total amount of times they were expressed was three median absolute deviations from the median.  In our output graphs, we kept the top 5 edges per vertex.  We can see the UMAP consistency plot in Figure \ref{fig:10x-umap} and the runtimes in Table \ref{tab:10x-runtimes}.

\begin{figure}
    \centering
    \includegraphics[width=0.8\textwidth]{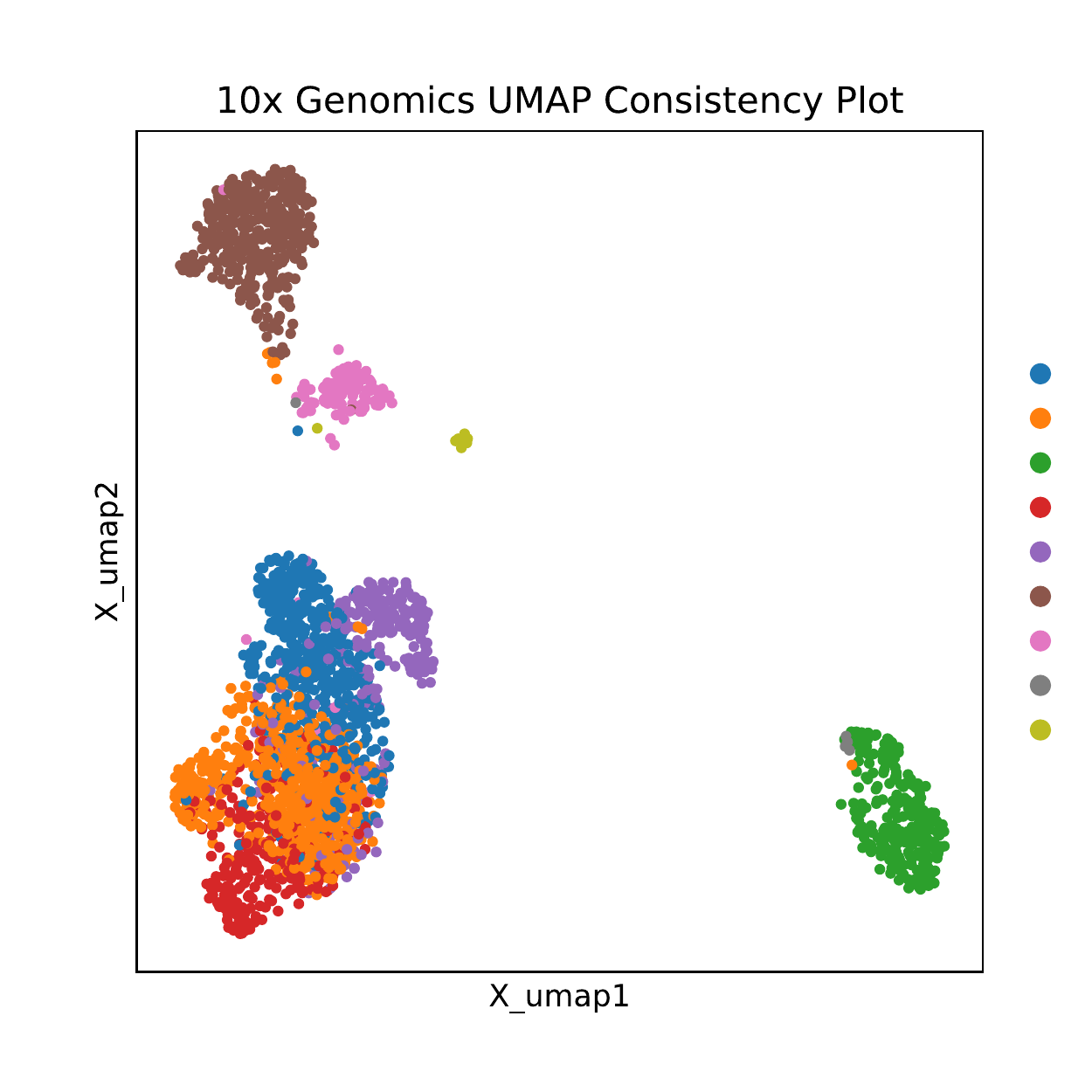}
    \caption{A UMAP consistency plot for the 10x Genomics dataset.  Both UMAP and GmGM were used on the full multi-omic log-transformed dataset.  The coloring represents a Louvain clustering on the output graph of GmGM, with the top 5 edges per vertex being kept after column normalization (the third thresholding method from Section \ref{sec:mouse}).}
    \label{fig:10x-umap}
\end{figure}

\begin{figure}
    \begin{subfigure}[b]{0.3\textwidth}
        \centering
        \includegraphics[width=\textwidth]{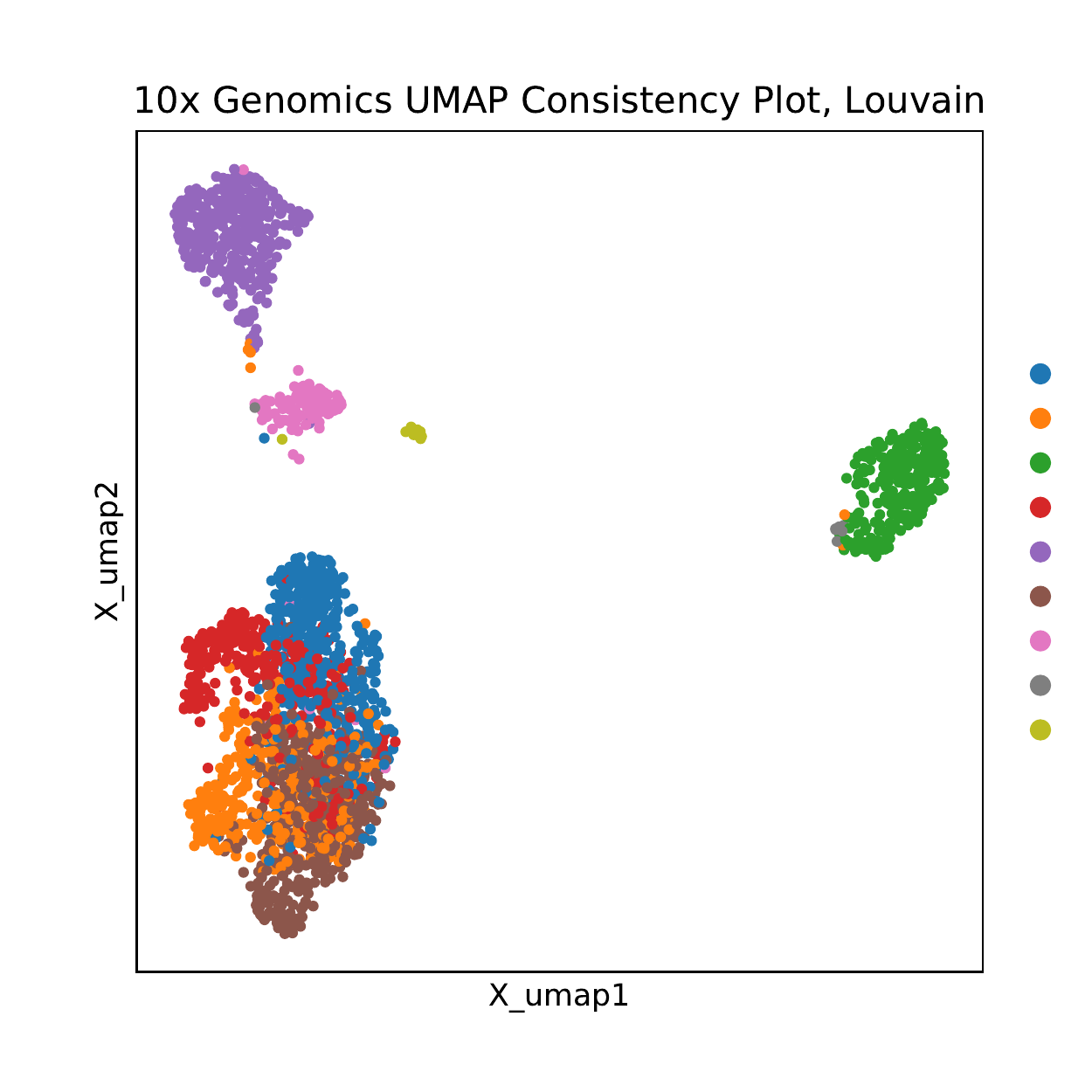}
        \caption{}
    \end{subfigure}
    \begin{subfigure}[b]{0.3\textwidth}
        \centering
        \includegraphics[width=\textwidth]{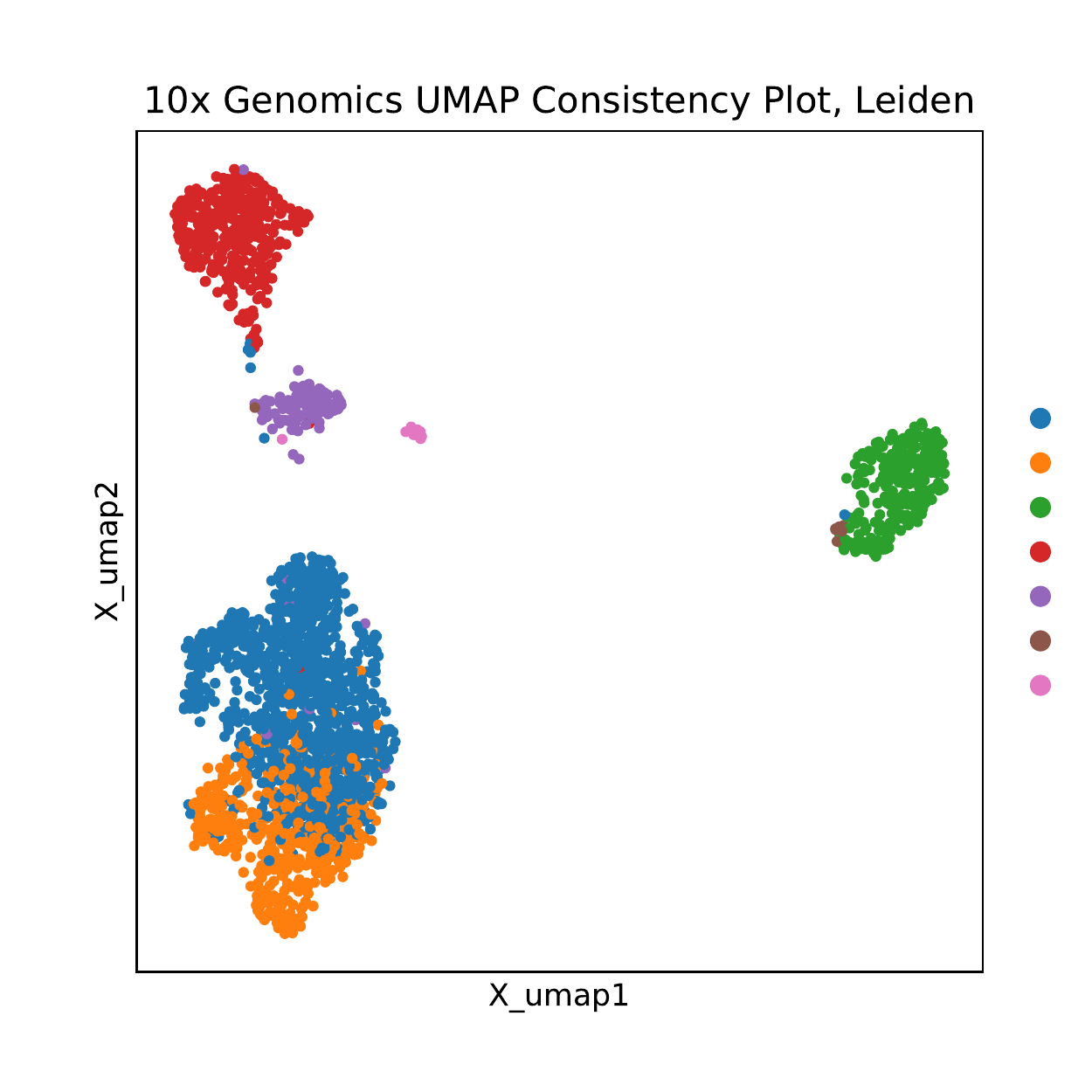}
        \caption{}
    \end{subfigure}
    \begin{subfigure}[b]{0.3\textwidth}
        \centering
        \includegraphics[width=\textwidth]{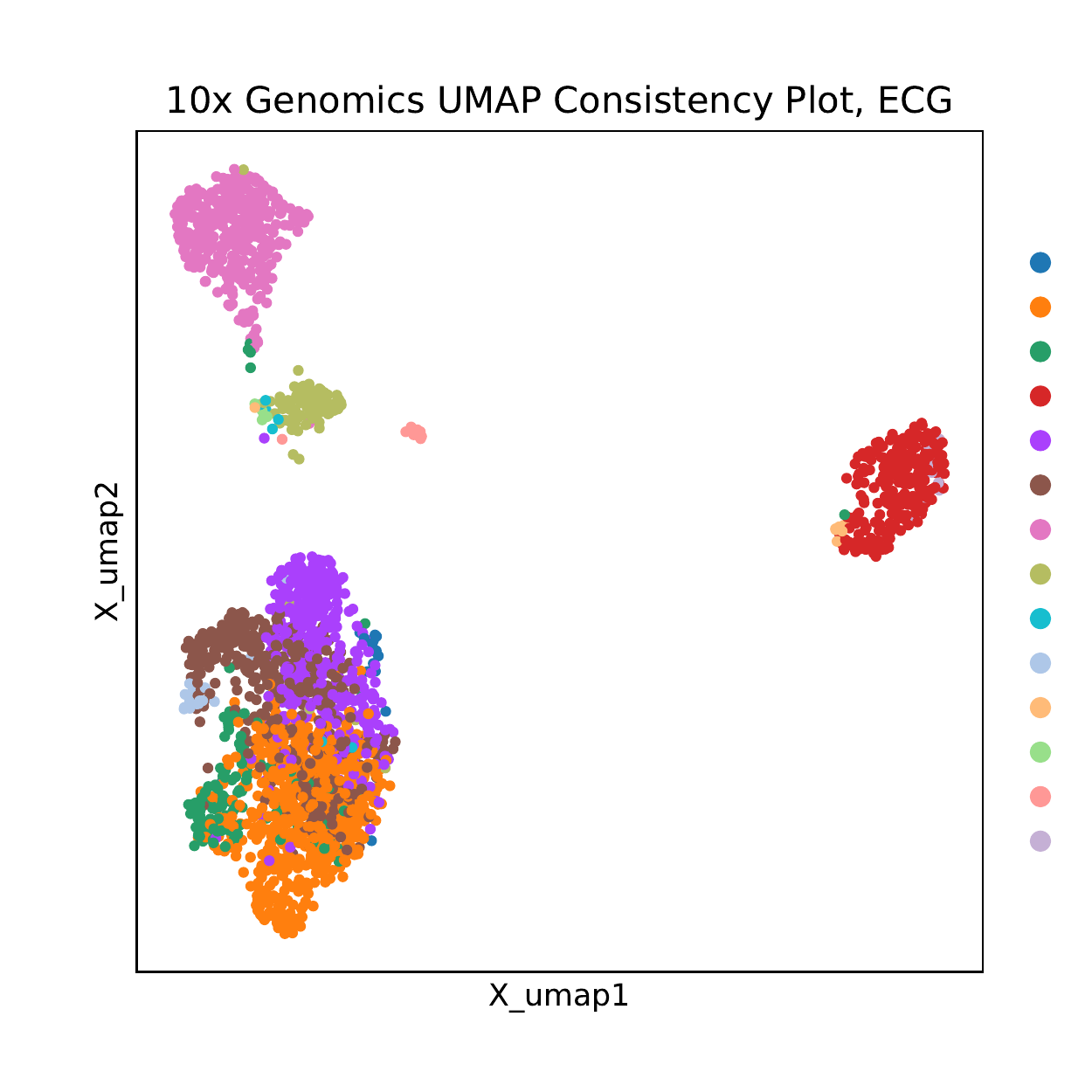}
        \caption{}
    \end{subfigure}
    \caption{UMAP consistency plots with Louvain, Leiden, and ECG clustering algorithms on the 10x Genomics data.}
    \label{fig:10x-all-umap}
\end{figure}

\begin{figure}
    \begin{subfigure}[b]{0.3\textwidth}
        \centering
        \includegraphics[width=\textwidth]{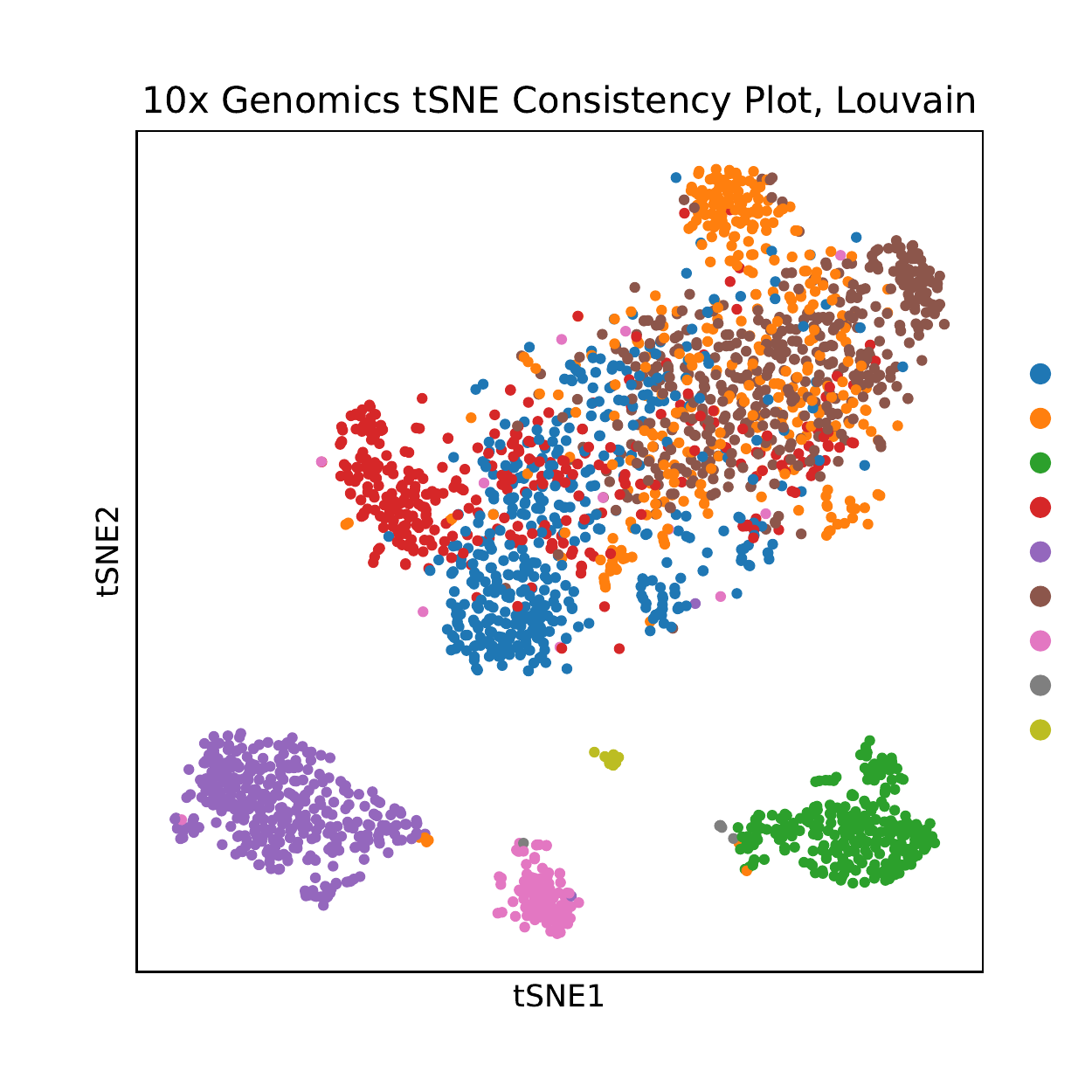}
        \caption{}
    \end{subfigure}
    \begin{subfigure}[b]{0.3\textwidth}
        \centering
        \includegraphics[width=\textwidth]{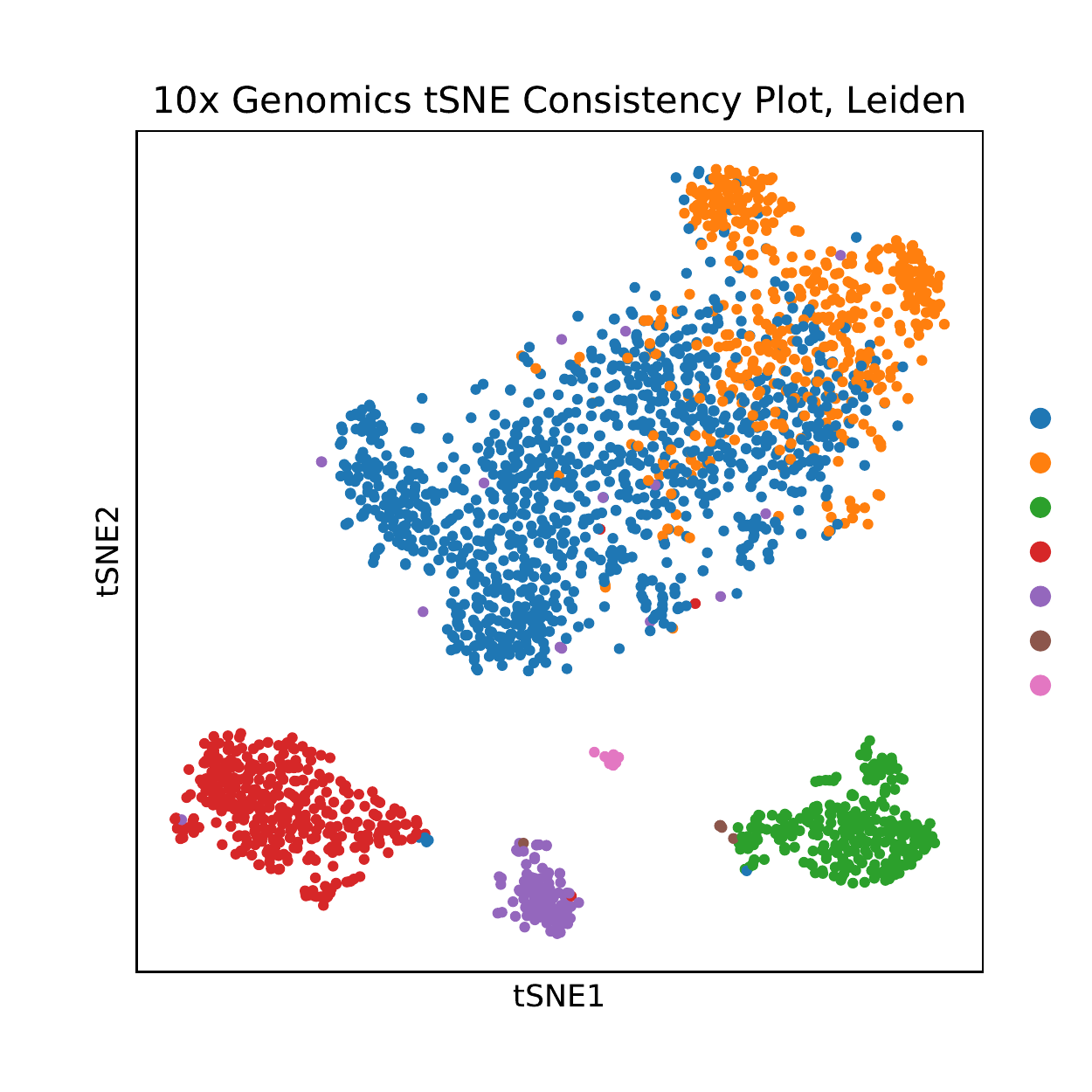}
        \caption{}
    \end{subfigure}
    \begin{subfigure}[b]{0.3\textwidth}
        \centering
        \includegraphics[width=\textwidth]{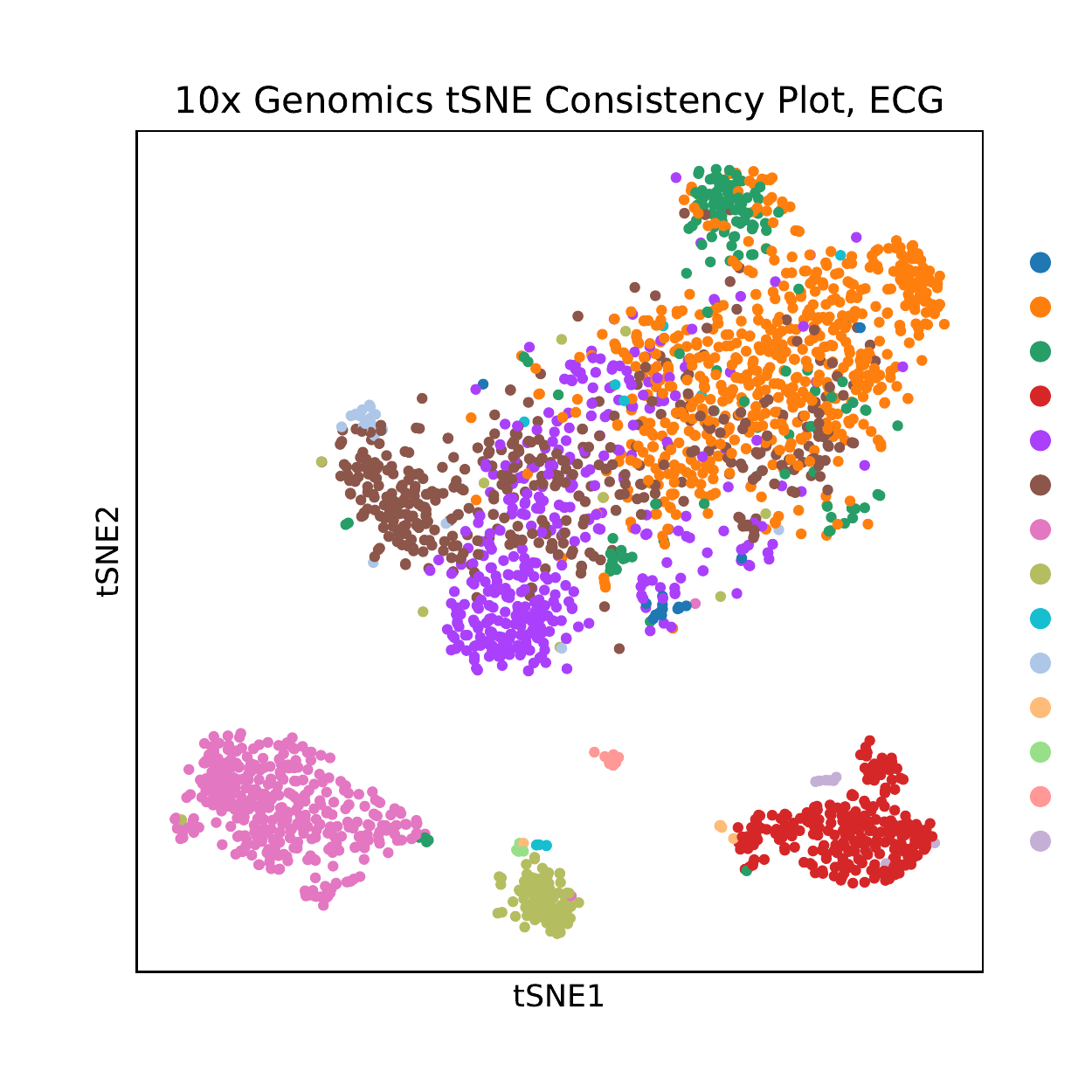}
        \caption{}
    \end{subfigure}
    \caption{tSNE consistency plots with Louvain, Leiden, and ECG clustering algorithms on the 10x Genomics data.}
    \label{fig:10x-all-tsne}
\end{figure}

\begin{table}[t]
    \centering
    \begin{tabular}{c|c|c}
         Algorithm & Preprocessing & Runtime \\
         \hline
         GmGM & Log transform, just RNA & 52 seconds \\
         GmGM & Log transform, just ATAC & 607 seconds \\
         GmGM & Log transform, both modalities & 590 seconds \\
         EiGLasso & Log transform, just RNA & >60,000 seconds
    \end{tabular}
    \caption{Runtimes of our algorithm and EiGLasso, with various preprocessing techniques, on the 10x Genomics data.}
    \label{tab:10x-runtimes}
\end{table}

In Figure \ref{fig:10x-umap}, we can see that clusters 0, 1, 3, and 4 all inhabit the large contiguous region in the lower left.  Clusters 2 and 7 inhabit the contiguous region in the lower right.  In the top right, there are three regions; 5, 6, and 8 (in order of left to right).  Figure \ref{fig:10x-dotplot} implies that clusters 5 and 6 have some similarities, whereas cluster 8 is largely different from all the other clusters.  A GO term analysis shows cluster 8 is associated with blood coagulation and the integrin signalling pathway, whereas none of the other clusters have such an association.  All other clusters, except 2 and 7, are highly associated with B cell activation-related genes.  Cluster 2 is distinct in that it is related to the CCKR signalling map and the apoptosis signalling pathways.  Cluster 7 was too small to perform the analysis.  Full details of the significant go terms are available on the github repository.

\begin{figure}
    \centering
    \includegraphics[width=0.5\textwidth]{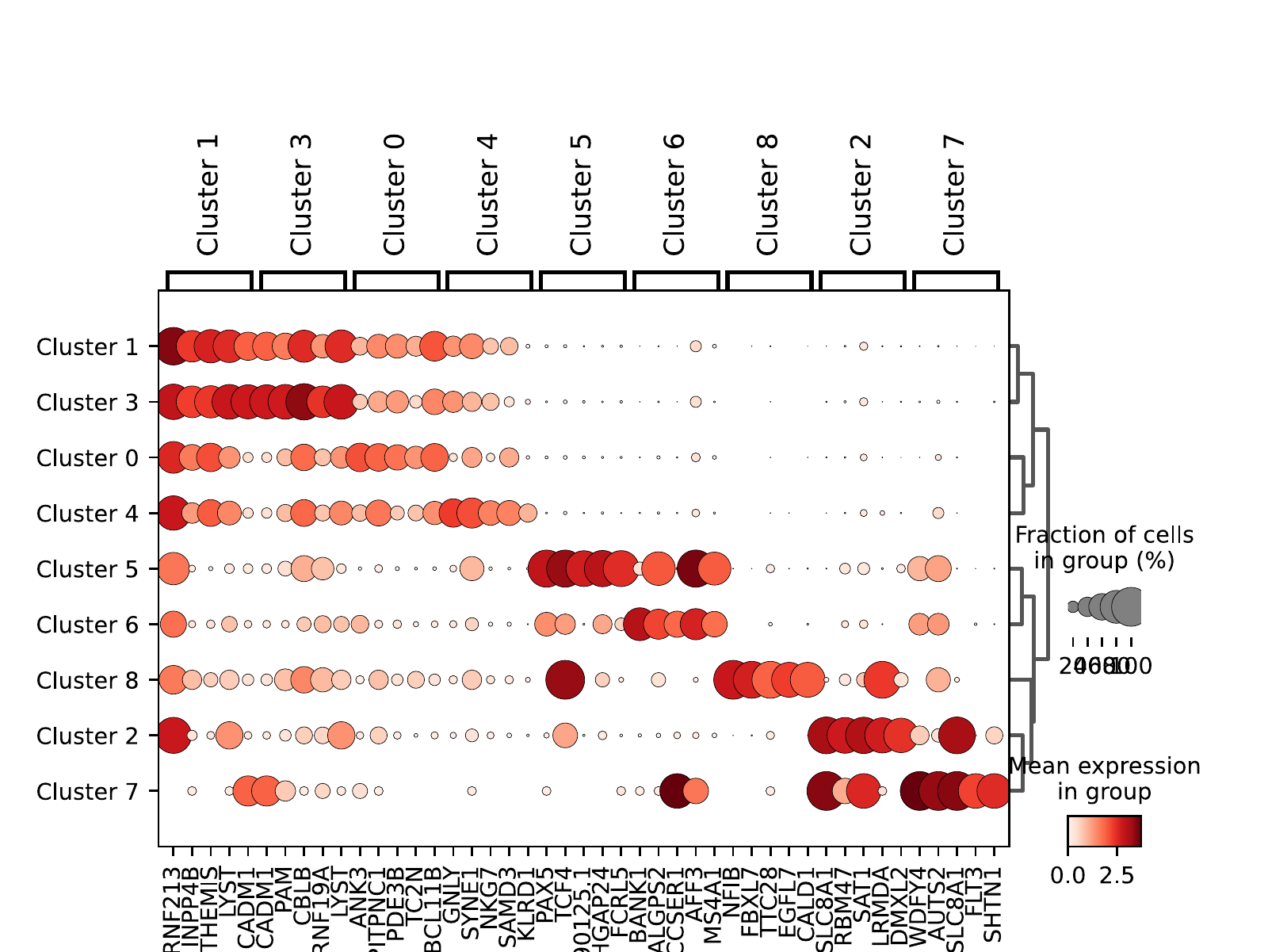}
    \caption{Expression of top differentially-expressed genes by cluster.}
    \label{fig:10x-dotplot}
\end{figure}

\section{REGULARIZATION}
\label{sec:regularization}

As remarked in the main paper, our algorithm by default includes no regularization.  This is because our algorithm leverages the fact that we have a closed-form expression for the eigenvectors of the maximum likelihood estimate to avoid costly eigendecompositions every iteration.  We do not have a closed-form expression for the eigenvectors in the regularized case.

\begin{figure}
    \begin{subfigure}[b]{0.45\textwidth}
        \centering
        \includegraphics[width=\textwidth]{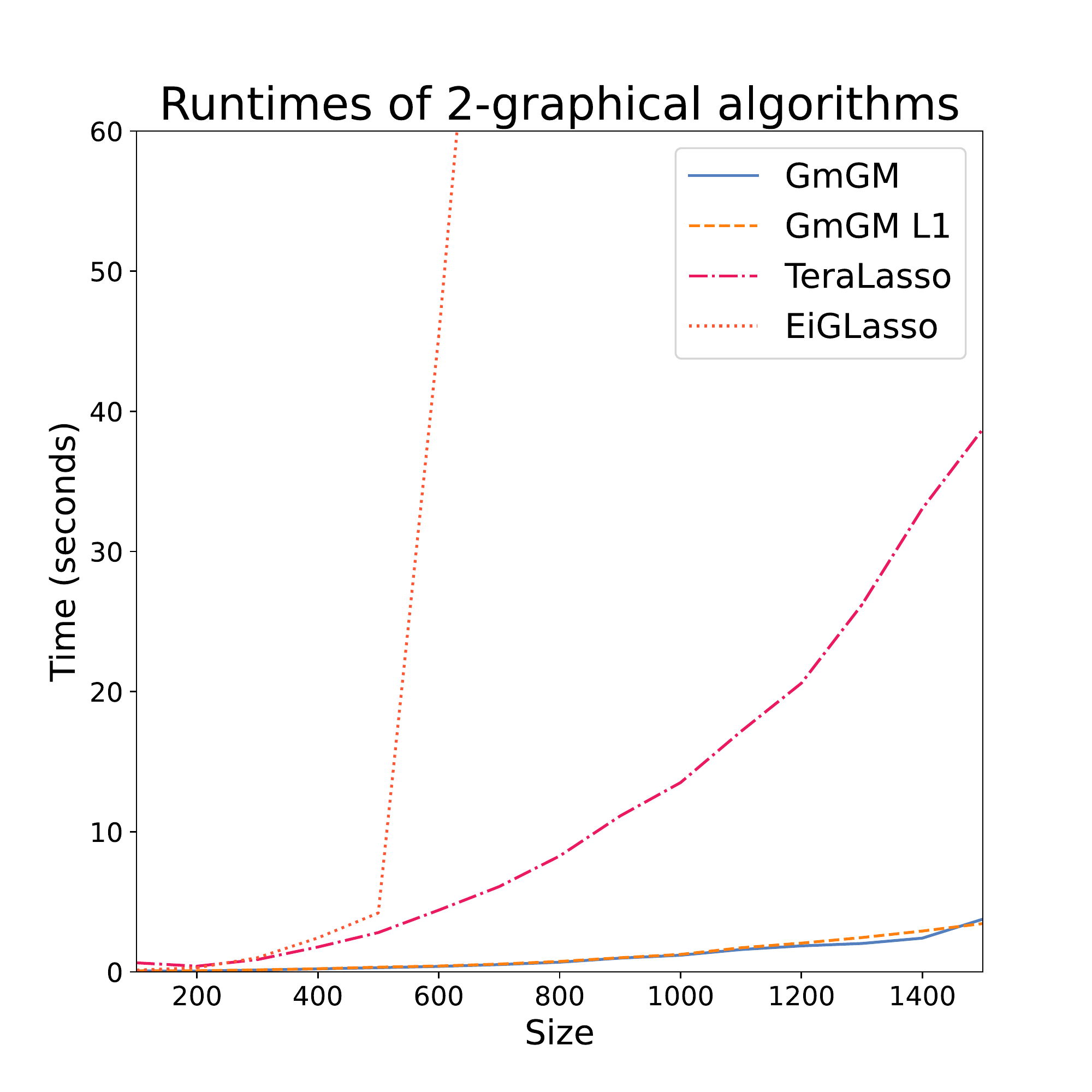}
        \caption{}
        \label{fig:regularization_runtimes}
    \end{subfigure}
    \begin{subfigure}[b]{0.45\textwidth}
        \centering
        \includegraphics[width=\textwidth]{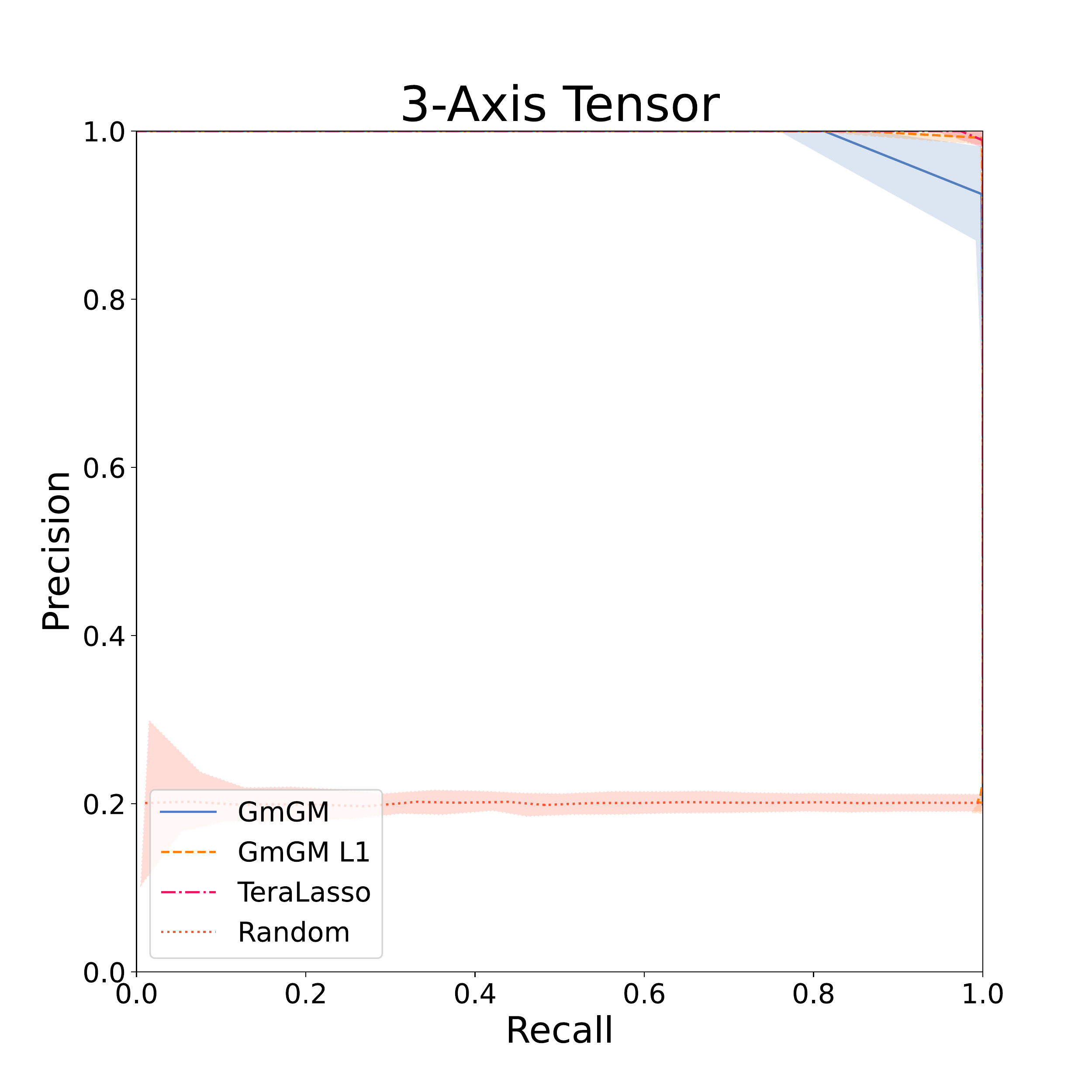}
        \caption{}
        \label{fig:regularization_tensor}
    \end{subfigure}
    \caption{(a) Runtimes of our algorithm and prior work on matrix-variate data.  Our regularized algorithm is denoted ``GmGM L1'', and takes about the same time as the unregularized ``GmGM''.  (b) Precision-recall curves for tensor-variate data.  TeraLasso and our regularized ``GmGM L1'' perform almost perfectly.}
\end{figure}

Nevertheless, we can add regularization to the eigenvalue estimation by performing an eigenrecomposition and regularizing that.  Eigenrecomposition requires a matrix multiplication, which is quite costly compared to the cost of an unregularized iteration - both in practice, and asymptotically in the matrix-variate case (matrix multiplication is $O(\sum_\ell d_\ell^3)$ whereas an unregularized iteration is $O(\prod_\ell d_\ell$)).  Thus, to regularize we first let our algorithm converge to the MLE before considering the penalty term.  This allows us to avoid a major increase in runtime; our regularized algorithm runs in roughly the same time as the unregularized one (Figure \ref{fig:regularization_runtimes}).

It is important to note that this estimator is slightly different than the standard Lasso estimator, as the standard estimator would minimize $\lVert \mathbf{\Psi}_\ell \rVert_1$ and our estimator minimizes $\lVert \hat{\mathbf{V}}_\ell \mathbf{\Lambda}_\ell \hat{\mathbf{V}}_\ell^T \rVert_1$ (where only the eigenvalues $\mathbf{\Lambda}_\ell$ are free to vary).  This restriction prevents it from being able to drive elements exactly to zero, so thresholding is still needed afterwards.  It can be derived as follows:

\begin{align}
    \frac{\partial}{\partial\lambda_i}\lVert\mathbf{V}\mathbf{\Lambda}\mathbf{V}^T\rVert_1 &= \frac{\partial}{\partial\lambda_i}\lVert\sum_{j}\lambda_j v_{ja} v_{bj}\rVert_1 \\
    &= \left[\frac{\partial}{\partial\lambda_i}\left|\sum_{j}\lambda_j v_{ja} v_{bj}\right|\right]_{ab} \\
    &= \left[\frac{\partial}{\partial\lambda_i}\mathrm{sign}\left[\sum_{j}\lambda_j v_{ja} v_{bj}\right]v_{ia}v_{bi}\right]_{ab} \\
    &= \left[ \mathrm{sign}\left[\mathbf{V} \mathbf{\Lambda} \mathbf{V}^T\right]_{ab} v_{ia} v_{bi} \right]_{ab} \\
    &= \mathbf{v}_i^T \mathrm{sign}\left[\mathbf{V} \mathbf{\Lambda} \mathbf{V}^T\right] \mathbf{v}_i
\end{align}

Despite this difference, it performs comparably to prior work.  We show in Figure \ref{fig:regularization_tensor} the precision-recall curves for the 3-axis case, and observe that it performs almost perfectly.  This is notable as it was the case that the unregularized algorithm performed worse than prior work.

To test whether regularization improved the algorithm in the real world, we performed another test on the E-MTAB-2805 dataset.  We first looked at the metric of `percentage of connections to own group'.  We varied the regularization parameter, and for each estimated graph we thresholded the graph to keep only the largest edge per cell.  We plotted the results in Figure \ref{fig:mouse-self-connected}.

The optimal parameter arises right near the end of the range of considered regularization parameters; when we considered larger parameters, we entered a region of instability in which convergence took far longer and results became sporadic (Figure \ref{fig:mouse-self-connected-chaos}).  We are not sure what causes this, but hypothesize that our restricted regularizer no longer has the degrees of freedom to give useful information to the problem, rather making it harder to find a good solution.  The fact that the optimal parameters occur specifically at the end of the stable region implies that the algorithm would likely have benefited from a non-restricted regularizer - although our restricted regularizer is also definitively better than no regularizer at all.

Our second test, shown in Figure \ref{fig:mouse-reg}, was to take the optimal regularing parameter found in the first test, and perform the same assortativity test as was performed in Section \ref{sec:mouse}.  We can see that the regularized methods tend to outperform the unregularized methods.

\begin{figure}
    \begin{subfigure}[b]{0.45\textwidth}
        \centering
        \includegraphics[width=\textwidth]{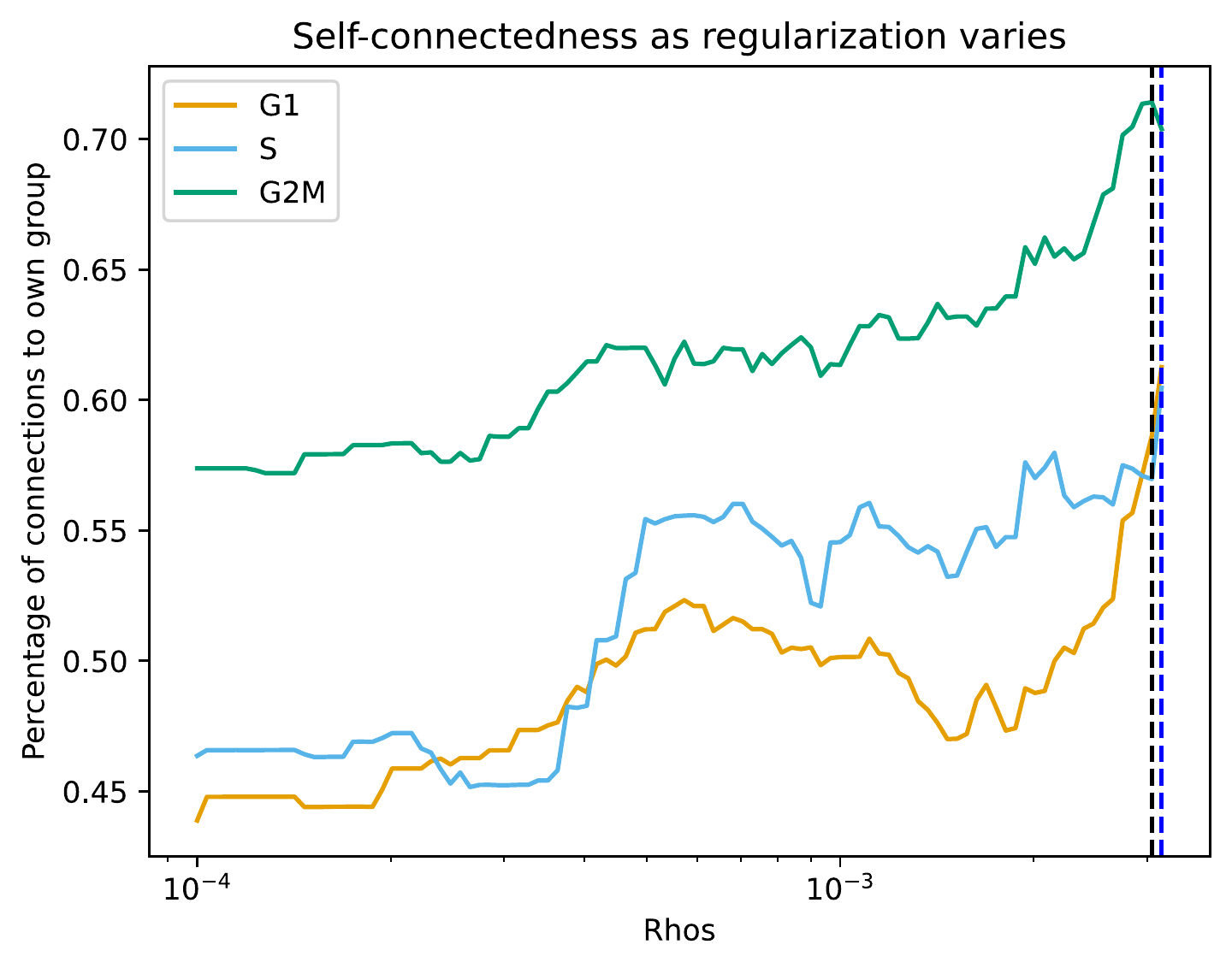}
        \caption{}
        \label{fig:mouse-self-connected}
    \end{subfigure}
    \begin{subfigure}[b]{0.45\textwidth}
        \centering
        \includegraphics[width=\textwidth]{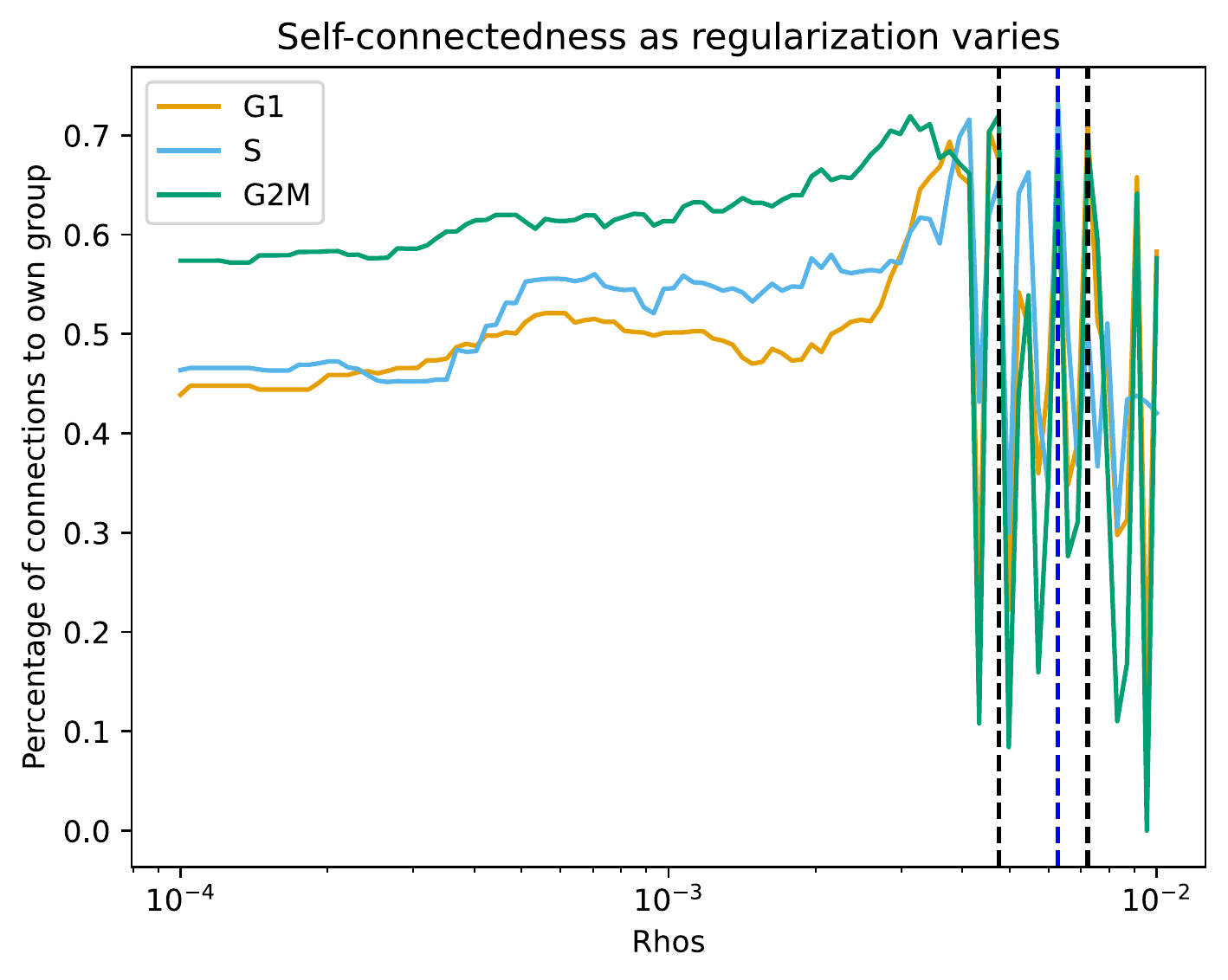}
        \caption{}
        \label{fig:mouse-self-connected-chaos}
    \end{subfigure}
    \caption{(a) Self-connectedness of each of the three cell cycle stages as regularization parameter varies.  Vertical bars represent the optimal parameters; one for each of the parameters (there is overlap).  (b) The same plot as on the left, with the range extended to show the region of instability.}
\end{figure}

\begin{figure}
    \centering
    \includegraphics[width=0.5\textwidth]{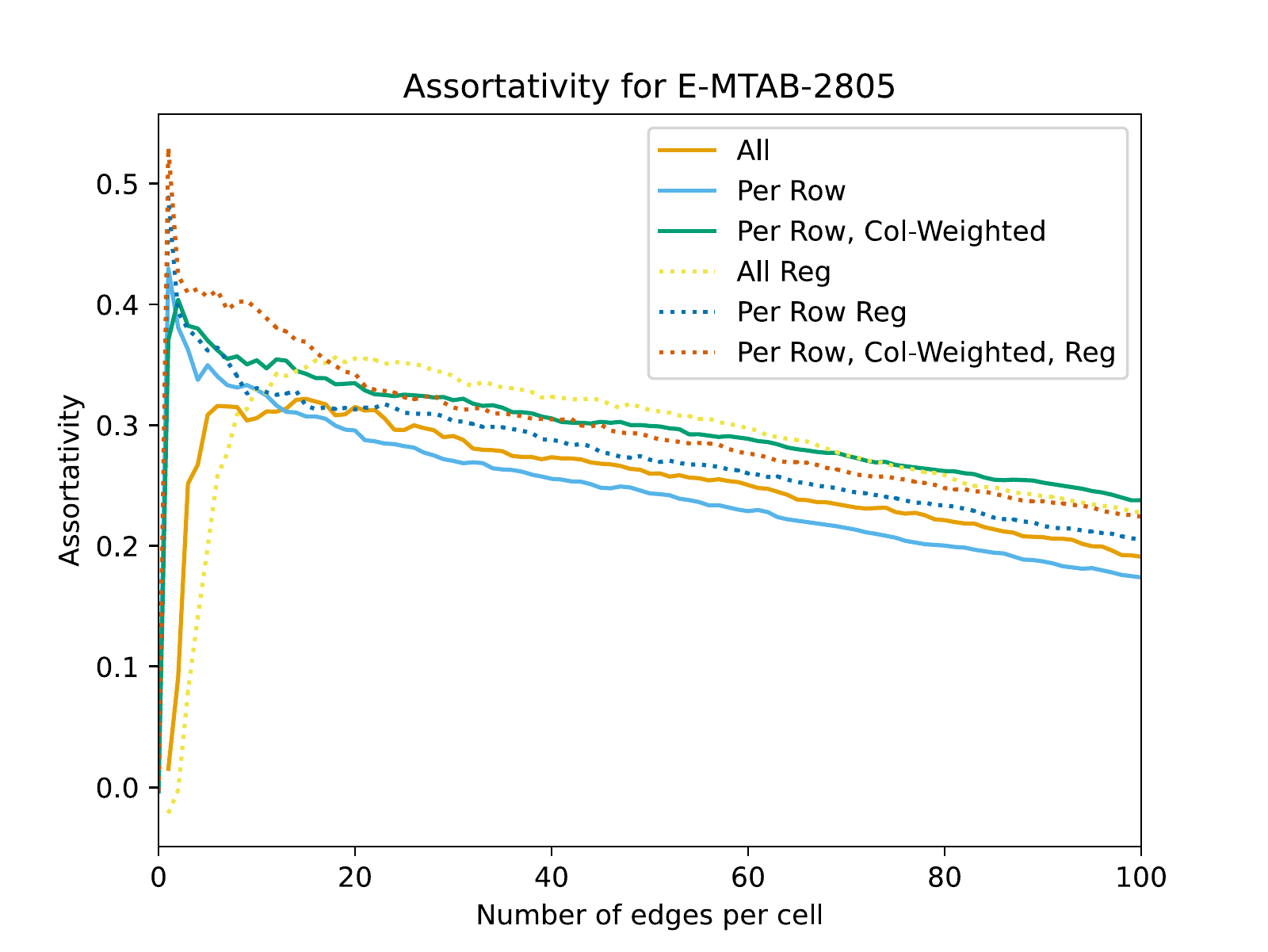}
    \caption{Assortativity as the number of kept edges per cell varies.}
    \label{fig:mouse-reg}
\end{figure}

\section{ASYMPTOTIC COMPLEXITY}
\label{sec:complexity}

Our space complexity is the optimal $O(\sum_\ell d_\ell^2)$, as was all prior work except the original BiGLasso.  For computational complexity, refer to Table \ref{tab:complexity}.  In the reported values, we treat matrix multiplication as having cubic complexity.  Note that Gram matrix computation is an $O(Kd^{K+1})$ operation, and that our algorithm has the asymptotically fastest iterations in the matrix-variate case ($O(d^2)$).  If we ignore the per-axis sizes $d$ and instead consider the size of the whole dataset $d_\forall = d^K$, we can note that the runtime of the tensor-variate algorithms is almost linear in the size of the dataset.

\begin{table}[t]
    \centering
    \begin{tabular}{c|ccc}
         Algorithm & Range of K & Overall Complexity & Per-Iteration Complexity \\
         \hline
         BiGLasso & $K=2$ & $O(Kd^4)$ & $O(Kd^4)$ \\
         EiGLasso & $K=2$ & $O(Kd^3)$ & $O(Kd^3)$ \\
         TeraLasso & Any & $O(Kd^{K+1})$ & $O(Kd^3 + d^K)$ \\
         GmGM & Any & $O(Kd^{K+1})$ & $O(d^K)$ \\
         GmGM L1 & Any & $O(Kd^{K+1})$ & $O(Kd^3 + d^K)$ \\
    \end{tabular}
    \caption{Computational complexities of selected algorithms.  `GmGM L1' refers to GmGM equipped with our restricted L1 penalty; note that GmGM L1 first converges to the unregularized solution (using the cheap iterations of GmGM) before it adds in the regularizer.  For simplicity, all dimensions $d_\ell$ are considered to be the same size $d$.}
    \label{tab:complexity}
\end{table}

\section{CENTERING THE DATA}
\label{sec:centering}

Multi-axis models such as ours typically assume a mean of 0 for the data.  When one has multiple samples, this is not a problem - one can just center the data.  However, with multi-axis methods we typically only have one sample.  Because the nonparanormal skeptic is invariant to monotone transformations, this is a viable method to circumvent this problem.  When not using that, we would take the average value of the tensor (averaged across all elements in the tensor) and subtract this value from ever element in the tensor.

\newpage

\printbibliography